\documentclass[12pt]{article}

\usepackage{placeins}
\usepackage{rotating}
\usepackage{amsmath, amssymb, amsfonts, latexsym, mathtools}
\usepackage{amsthm, mathrsfs, bm}

\usepackage{booktabs}
\usepackage{multirow}
\usepackage{makecell}
\usepackage{adjustbox}
\usepackage{tabularx}

\usepackage{graphicx}
\usepackage{subcaption}
\usepackage{float}

\usepackage[ruled, lined, linesnumbered, commentsnumbered, longend]{algorithm2e}
\SetKwProg{Pro}{procedure}{}{end\ procedure}
\usepackage[table,x11names]{xcolor}
\usepackage{url}
\usepackage{hyperref}

\usepackage{courier}
\usepackage{helvet}

\usepackage[utf8]{inputenc}
\usepackage[english]{babel}
\usepackage{setspace}
\usepackage{comment}
\usepackage{enumitem}

\usepackage{caption}

\usepackage[nohead]{geometry}
\usepackage[round]{natbib}

\usepackage{tikz}
\usetikzlibrary{shapes.geometric, arrows, positioning}

\usepackage{authblk}

\definecolor{RoyalBlue}{rgb}{0.25,0.41,0.88}
\definecolor{light-gray}{gray}{0.97}
\definecolor{qblack}{gray}{0.30}

\hypersetup{
    colorlinks=true,
    linkcolor=blue,
    urlcolor=blue,
    citecolor=blue
}

\theoremstyle{plain}   

\newtheorem{theorem}{Theorem}[section]
\newtheorem{proposition}[theorem]{Proposition}

\newtheorem{corollary}[theorem]{Corollary}

\theoremstyle{definition}  
\newtheorem{definition}[theorem]{Definition}

\theoremstyle{remark}  

\numberwithin{equation}{section}

\title{\Large Informative Label Missingness in Multiclass Classification:
Information Geometry and Excess Risk}

\author[1]{Fariborz Setoudehtazangi}

\author[2]{Geoffrey J. McLachlan%
\thanks{Corresponding author.
Email:
\href{mailto:g.mclachlan@uq.edu.au}
{g.mclachlan@uq.edu.au}}}

\affil[1]{Dipartimento di Scienze Statistiche,
Università di Padova, Padova, Italy}

\affil[2]{School of Mathematics and Physics,
The University of Queensland, Australia}
\begin{document}

\date{}
\maketitle

\begin{abstract}
Informative label missingness can change the usual efficiency ordering between
completely and partially labelled classifiers because the pattern of missing
labels may itself carry information about the classification model. We develop
a general likelihood-based theory for this phenomenon in parametric multiclass
classification. An efficient-information decomposition separates information
lost through unavailable class memberships from information contributed by the
missing-label mechanism. We then derive a quadratic expansion of plug-in
excess risk over the active pairwise faces of the multiclass Bayes boundary,
showing that classification efficiency depends on how information gains and
losses align with directions that perturb the decision boundary. This yields
a classification-weighted generalized-eigenvalue criterion under which
informative partial classification may have smaller asymptotic classification
risk without globally dominating complete classification in Fisher
information. Near missing completely at random, with the marginal missing-label
proportion fixed, redistribution of missing labels changes lost class-label
information at first order, whereas efficient information from the missingness
pattern appears only at second order. Three-class quadratic discriminant
calculations, finite-sample experiments, and a semi-synthetic multiclass
application illustrate the resulting regime-dependent behaviour.
\end{abstract}

\section{Introduction}
\label{sec:introduction}

Classification from partially labelled data arises when the feature vectors in
a training sample are observed but the class memberships of some observations
are unavailable. This problem has a long history in statistical discriminant
analysis and is now commonly viewed as a form of semi-supervised learning
\citep{mclachlan1975iterative,dempster1977maximum, Chapelle2006,van2020survey,ahfock2020apparent}.
From a likelihood perspective, the unavailable class memberships constitute
missing data, so labelled and unlabelled observations can be combined through
the likelihood of the corresponding finite-mixture model. When label
availability is unrelated to the observed data, an observation whose class
membership is unknown is generally less informative for estimating a
classification rule than the same observation with its label observed. The
resulting loss of information and its implications for discriminant efficiency
were studied in early work on partially classified samples
\citep{o1978normal,ganesalingam1978efficiency,
mclachlan1995asymptotic}; related questions concerning the statistical value of
labelled and unlabelled observations have also been considered from a learning
theory perspective \citep{castelli1996relative}.

The comparison changes when label availability is related to the
classification problem. In many applications, labels are assigned selectively.
For example, a human expert may be more likely to leave an observation
unclassified when its features place it in a region of substantial class
overlap. The observed pattern of label availability can then contain
information about parameters governing the classifier. Such situations are
naturally described within a missing-data framework
\citep{rubin1976inference,mealli2015clarifying}. If \(M\) denotes the indicator
that the class label \(Z\) is unavailable and \(\boldsymbol Y\) is the observed
feature vector, we consider mechanisms satisfying
\[
M\perp Z\mid\boldsymbol Y,
\qquad
\Pr(M=1\mid\boldsymbol Y=\boldsymbol y)
=
q(\boldsymbol y;\boldsymbol\theta,\boldsymbol\xi),
\]
where \(\boldsymbol\theta\) parameterizes the classification model and
\(\boldsymbol\xi\) contains parameters specific to the missing-label
mechanism. Although \(M\) is conditionally independent of the latent class
given the observed features, the mechanism need not be ignorable for
likelihood inference on \(\boldsymbol\theta\), because its distribution may
itself depend on \(\boldsymbol\theta\).

A counterintuitive consequence of this dependence was developed by
\citet{ahfock2020apparent}. For two homoscedastic Gaussian classes, they
modelled the probability of a missing label as a function of classification
uncertainty and showed that the information carried by the missing-label
indicators can exceed the information lost through the unavailable class
memberships. Consequently, an estimated Bayes rule based on a partially
classified sample can have a smaller asymptotic expected error rate than the
corresponding rule based on a completely classified sample. Their formulation
was motivated by the observation that unclassified observations may
concentrate near class boundaries, and related the missingness probability to
classification entropy or, in the binary Gaussian setting, to the squared
discriminant function
\citep{mclachlan2019estimation,ahfock2020apparent}.
This phenomenon and its relation to statistical semi-supervised learning have
subsequently been reviewed by \citet{ahfock2023semi}.

The phenomenon is not confined numerically to the homoscedastic binary model.
In particular, \citet{lyu2024analysis} studied Bayes-rule estimation under a
specified missing-data mechanism for Gaussian mixture models and considered
two- and three-component settings with unequal covariance matrices through
simulation and real-data illustrations. More recent work has continued to
emphasize the implications of informative missingness in semi-supervised
learning and the roles of class overlap, the proportion of missing labels, and
the strength of the association between missingness and classification
uncertainty \citep{wu2025informative}. Thus there is evidence that favourable
informative missingness can arise beyond the particular Gaussian configuration
in which the phenomenon was first analysed.

What remains less understood is the \emph{classification-risk geometry} of
this phenomenon in a general multiclass problem. Existing analytic results
explaining efficiency gains from informative label missingness have been
developed primarily through binary discriminant structures. In a multiclass
problem, however, the Bayes rule is determined by several pairwise equality
surfaces, only portions of which form active decision boundaries, and these
surfaces may meet at triple or higher-order class junctions. More importantly,
global improvement in Fisher information is stronger than what classification
itself requires. Information gained in a parameter direction that barely
moves the Bayes boundary may have little effect on classification risk,
whereas a gain concentrated in a strongly boundary-relevant direction may be
substantially more important.

This issue connects informative label missingness with the classical
asymptotic theory of estimated discriminant rules. Previous work has studied
the effect of parameter estimation on classification error and developed
large-sample and higher-order risk expansions for discriminant procedures
\citep{o1980general,taniguchi1994higher,ducinskas2002second}.
A central implication of this literature is that parameter-estimation
precision and classification performance are not interchangeable:
classification risk depends on how estimation error perturbs the decision
rule. For informative partial classification, a comparison of Fisher
information alone therefore does not reveal which information gains and losses
are relevant to classification.

The aim of this paper is to connect these two strands of theory. We first
derive an efficient-information decomposition for a general parametric
multiclass model with informative label missingness. After eliminating the
parameters specific to the missing-label mechanism, the information for the
classification model is
\begin{equation}
\boldsymbol I_{\mathrm{PC}}^{\mathrm{eff}}
=
\boldsymbol I_{\mathrm{CC}}
-
E\!\left[
q(\boldsymbol Y)
\boldsymbol I_{Z\mid\boldsymbol Y}(\boldsymbol Y)
\right]
+
\boldsymbol I_M^{\mathrm{eff}},
\label{eq:intro-information-decomposition}
\end{equation}
where the second term is the information lost through unavailable class
memberships and
\(\boldsymbol I_M^{\mathrm{eff}}\succeq\boldsymbol0\) is the efficient
information contributed by the missing-label indicators. The weighting by
\(q(\boldsymbol Y)\) is important: the information cost of missing labels
depends not only on how many labels are unavailable but also on where in the
feature space they are missing. Missing completely at random is recovered as
a special case in which the missingness indicators contribute no information
about the classification parameters.

We then relate this information structure to classification performance by
deriving a local quadratic excess-risk expansion for regular multiclass Bayes
classifiers. The leading curvature is determined by the
\emph{active pairwise Bayes faces}, namely the portions of pairwise equality
surfaces on which the corresponding two classes jointly attain the Bayes
maximum. Under regularity and transversality conditions, triple and
higher-order Bayes junctions do not contribute to the leading quadratic term.
If
\(\widehat{\boldsymbol\theta}\) is asymptotically normal with asymptotic
covariance matrix \(\boldsymbol V\), then
\begin{equation}
E\!\left\{
R(\widehat{\boldsymbol\theta})
\right\}
-
R^\ast
=
\frac{1}{2n}
\operatorname{tr}
\left(
\boldsymbol H_R\boldsymbol V
\right)
+
o(n^{-1}),
\label{eq:intro-risk-expansion}
\end{equation}
where \(\boldsymbol H_R\) is the curvature matrix induced by the active Bayes
faces.

Combining the information and risk calculations yields the main
classification-specific comparison. Let \(\boldsymbol A\) and
\(\boldsymbol J\) denote the relevant efficient information matrices under
complete classification and informative partial classification, respectively.
The leading classification advantage of the latter is governed by
\begin{equation}
\Delta_R
=
\operatorname{tr}
\left[
\boldsymbol H_R
\left(
\boldsymbol A^{-1}
-
\boldsymbol J^{-1}
\right)
\right].
\label{eq:intro-delta}
\end{equation}
Thus a smaller asymptotic classification risk does not require
\(\boldsymbol J\succeq\boldsymbol A\). A partially classified experiment may
lose information in some parameter directions and nevertheless improve
classification if its information gains are sufficiently concentrated in
directions that move the active Bayes boundary. A generalized-eigenvalue
representation makes this distinction explicit by separating relative
information gain or loss from the classification relevance of the
corresponding parameter direction. Favourable informative missingness is
therefore a directional alignment phenomenon rather than simply a question of
global Fisher-information dominance.

We also examine how this advantage emerges as a missing-label mechanism
departs from missing completely at random while the marginal proportion of
missing labels is held fixed. Locally, redistributing the missing labels
according to classification uncertainty alters the loss of conditional
class-label information at first order, whereas the efficient information
carried by the missingness indicators appears only at second order. Hence a
weakly informative mechanism need not initially improve classification and may
instead make it worse. Whether a favourable regime subsequently emerges
depends on the global information geometry and its alignment with the active
decision boundaries.

The theory is illustrated using a three-class Gaussian quadratic discriminant
model with unequal covariance matrices. This is the smallest setting that allows
curved pairwise Bayes faces and a genuine multiclass junction to arise
simultaneously while the decision geometry remains directly interpretable.
We examine how the classification-risk comparison varies with the proportion
of missing labels, the strength of uncertainty dependence, class separation,
prior imbalance, covariance heterogeneity, and the choice of entropy- or
Gini-based uncertainty measures. The population calculations reveal
face-specific gains and losses and geometry-dependent transitions between
unfavourable and favourable regimes. A finite-sample Monte Carlo experiment
then assesses the asymptotic predictions under likelihood estimation; in the
reference configuration, the theoretical classification-risk coefficients,
empirical quadratic approximations, and directly evaluated scaled excess risks
are in close agreement.

The contribution of this paper is therefore not to establish that informative
label missingness can sometimes be beneficial---that phenomenon is already
known---nor simply to demonstrate it for another class-conditional
distribution. Rather, we develop a general multiclass framework that explains
how information gained and lost through informative label missingness is
translated into classification risk through the geometry of the active Bayes
boundary. This perspective identifies when and where the missingness pattern
can improve classification and shows why favourable partial classification
does not require global information dominance.

The remainder of the paper is organized as follows.
Section~\ref{sec:model} introduces the general partially labelled model and
derives the efficient-information decomposition.
Section~\ref{sec:classifier-information} considers efficient classifier
information in the presence of nuisance parameters.
Section~\ref{sec:risk} develops the multiclass excess-risk geometry, and
Section~\ref{sec:spectral} gives the classification-weighted spectral
characterization and studies local departures from missing completely at
random.
Section~\ref{sec:qda} specializes the framework to quadratic discriminant
analysis.
Sections~\ref{sec:numerical} and~\ref{sec:finite-sample} present the
population and finite-sample investigations, respectively.
Section~\ref{sec:realdata} provides a semi-synthetic application to the
Vertebral Column data.
Section~\ref{sec:discussion} discusses the main findings, limitations, and
directions for further research, and Section~\ref{sec:conclusion} concludes.
\section{General model and information decomposition}
\label{sec:model}

Let
$
\boldsymbol{Y}\in\mathcal{Y}\subseteq\mathbb{R}^{p}
$
denote the observed feature vector and let
$
Z\in\{1,\ldots,g\}
$
denote the corresponding class label. We assume a parametric class model of the
form
\[
\Pr_{\boldsymbol{\theta}}(Z=k)=\pi_k,
\qquad
\boldsymbol{Y}\mid Z=k
\sim
f_k(\boldsymbol{y};\boldsymbol{\vartheta}_k),
\qquad
k=1,\ldots,g,
\]
where
$
\boldsymbol{\theta}
=
\left(
\boldsymbol{\pi}^{\top},
\boldsymbol{\vartheta}_1^{\top},
\ldots,
\boldsymbol{\vartheta}_g^{\top}
\right)^{\top}
$
denotes an identifiable finite-dimensional parameter vector, with a suitable
parameterization of the class probabilities
\(\boldsymbol{\pi}=(\pi_1,\ldots,\pi_g)^\top\).

The joint class-feature density is
\[
p_{\boldsymbol{\theta}}(\boldsymbol{y},Z=k)
=
\pi_k
f_k(\boldsymbol{y};\boldsymbol{\vartheta}_k),
\]
and the marginal feature density is
\[
p_{\boldsymbol{\theta}}(\boldsymbol{y})
=
\sum_{k=1}^{g}
\pi_k
f_k(\boldsymbol{y};\boldsymbol{\vartheta}_k).
\]
The posterior class probabilities are therefore
\[
\tau_k(\boldsymbol{y};\boldsymbol{\theta})
=
\Pr_{\boldsymbol{\theta}}
(Z=k\mid\boldsymbol{Y}=\boldsymbol{y})
=
\frac{
\pi_k f_k(\boldsymbol{y};\boldsymbol{\vartheta}_k)
}{
p_{\boldsymbol{\theta}}(\boldsymbol{y})
}.
\]
The corresponding Bayes classifier assigns \(\boldsymbol{y}\) to a class
maximizing \(\tau_k(\boldsymbol{y};\boldsymbol{\theta})\), or equivalently,
\(\pi_kf_k(\boldsymbol{y};\boldsymbol{\vartheta}_k)\).

Let
\[
M=
\begin{cases}
1, & \text{if the class label is unavailable},\\
0, & \text{if the class label is observed}.
\end{cases}
\]
We assume throughout that
\begin{equation}
M\perp Z\mid\boldsymbol{Y},
\label{eq:MAR-label}
\end{equation}
while allowing the probability of label absence to depend on both the observed
features and the parameters of the classification model:
\begin{equation}
\Pr_{\boldsymbol{\theta},\boldsymbol{\xi}}
(M=1\mid\boldsymbol{Y}=\boldsymbol{y})
=
q(\boldsymbol{y};\boldsymbol{\theta},\boldsymbol{\xi}),
\label{eq:q-general}
\end{equation}
where
\(\boldsymbol{\xi}\) denotes parameters specific to the missing-label mechanism
and
\[
0<
q(\boldsymbol{y};\boldsymbol{\theta},\boldsymbol{\xi})
<1
\]
on the relevant support.

The observed datum for one individual is
\[
\mathcal{O}
=
\left(
\boldsymbol{Y},
M,
(1-M)Z
\right).
\]
Under \eqref{eq:MAR-label}, its density can be written as
\begin{align}
p_{\boldsymbol{\theta},\boldsymbol{\xi}}(\mathcal{O})
={}&
p_{\boldsymbol{\theta}}(\boldsymbol{Y})
\,
\tau_Z(\boldsymbol{Y};\boldsymbol{\theta})^{1-M}
\nonumber\\
&\times
q(\boldsymbol{Y};\boldsymbol{\theta},\boldsymbol{\xi})^M
\left\{
1-
q(\boldsymbol{Y};\boldsymbol{\theta},\boldsymbol{\xi})
\right\}^{1-M}.
\label{eq:observed-density}
\end{align}
Consequently, the individual observed-data log-likelihood is
\begin{align}
\ell_{\mathrm{obs}}
(\boldsymbol{\theta},\boldsymbol{\xi})
={}&
\log
p_{\boldsymbol{\theta}}(\boldsymbol{Y})
+
(1-M)
\log
\tau_Z(\boldsymbol{Y};\boldsymbol{\theta})
\nonumber\\
&+
M
\log
q(\boldsymbol{Y};\boldsymbol{\theta},\boldsymbol{\xi})
+
(1-M)
\log
\left\{
1-
q(\boldsymbol{Y};\boldsymbol{\theta},\boldsymbol{\xi})
\right\}.
\label{eq:observed-loglik}
\end{align}

Expression \eqref{eq:observed-loglik} separates three sources of information:
the marginal distribution of the features, the observed class labels, and the
pattern of label availability. The last component is absent from an ignorable
analysis when the missing-label mechanism is treated as ancillary to the
classification model.

To quantify the contribution of an observed class label at a given feature
value, define the conditional class score
\[
\boldsymbol{S}_{Z\mid\boldsymbol{Y}}
=
\nabla_{\boldsymbol{\theta}}
\log
\tau_Z(\boldsymbol{Y};\boldsymbol{\theta}),
\]
and its conditional Fisher information
\begin{equation}
\boldsymbol{I}_{Z\mid\boldsymbol{Y}}
(\boldsymbol{\theta};\boldsymbol{y})
=
E_{\boldsymbol{\theta}}
\left[
\boldsymbol{S}_{Z\mid\boldsymbol{Y}}
\boldsymbol{S}_{Z\mid\boldsymbol{Y}}^{\top}
\mid
\boldsymbol{Y}=\boldsymbol{y}
\right].
\label{eq:conditional-label-info}
\end{equation}
This matrix measures the additional information provided by observing the class
membership after the feature vector is known.

For the missing-label mechanism, write
\[
q_{\boldsymbol{\theta}}
=
\frac{\partial q}
{\partial\boldsymbol{\theta}},
\qquad
q_{\boldsymbol{\xi}}
=
\frac{\partial q}
{\partial\boldsymbol{\xi}},
\]
and define
\begin{equation}
\boldsymbol{B}_{ab}
=
E
\left[
\frac{
q_a q_b^{\top}
}{
q(1-q)
}
\right],
\qquad
a,b\in\{\boldsymbol{\theta},\boldsymbol{\xi}\},
\label{eq:B-blocks}
\end{equation}
where the arguments of \(q\) are suppressed for notational simplicity.

The following result gives the central information decomposition.

\begin{theorem}
\label{thm:information-decomposition}
Assume that the model is locally identifiable and satisfies the standard
regularity conditions for Fisher-information calculations, including common
support in a neighborhood of the true parameter, differentiability with respect
to \((\boldsymbol{\theta},\boldsymbol{\xi})\), interchange of differentiation
and integration, finite second moments of the relevant score functions, and
nonsingularity of \(\boldsymbol{B}_{\boldsymbol{\xi}\boldsymbol{\xi}}\).
Then the efficient Fisher information for
\(\boldsymbol{\theta}\), after eliminating the missingness-specific nuisance
parameter \(\boldsymbol{\xi}\), is
\begin{equation}
\boldsymbol{I}_{\mathrm{PC}}^{\mathrm{eff}}
(\boldsymbol{\theta})
=
\boldsymbol{I}_{\mathrm{CC}}
(\boldsymbol{\theta})
-
\boldsymbol{D}
(\boldsymbol{\theta},\boldsymbol{\xi})
+
\boldsymbol{I}_{M}^{\mathrm{eff}}
(\boldsymbol{\theta},\boldsymbol{\xi}),
\label{eq:main-info-decomp}
\end{equation}
where
\begin{equation}
\boldsymbol{D}
(\boldsymbol{\theta},\boldsymbol{\xi})
=
E
\left[
q(\boldsymbol{Y};\boldsymbol{\theta},\boldsymbol{\xi})
\,
\boldsymbol{I}_{Z\mid\boldsymbol{Y}}
(\boldsymbol{\theta};\boldsymbol{Y})
\right],
\label{eq:D-def}
\end{equation}
and
\begin{equation}
\boldsymbol{I}_{M}^{\mathrm{eff}}
=
\boldsymbol{B}_{\boldsymbol{\theta}\boldsymbol{\theta}}
-
\boldsymbol{B}_{\boldsymbol{\theta}\boldsymbol{\xi}}
\boldsymbol{B}_{\boldsymbol{\xi}\boldsymbol{\xi}}^{-1}
\boldsymbol{B}_{\boldsymbol{\xi}\boldsymbol{\theta}}
\succeq
\boldsymbol{0}.
\label{eq:IM-eff}
\end{equation}
Here
\(\boldsymbol{I}_{\mathrm{CC}}(\boldsymbol{\theta})\) denotes the Fisher
information that would be available if all class labels were observed.
\end{theorem}

The proof is given in Supplementary Section~\ref{supp:proof-information}.

The decomposition in
\eqref{eq:main-info-decomp} separates two effects that are otherwise confounded
in a partially labelled sample. The matrix
\(\boldsymbol{D}\) is the information lost because some class memberships are
not observed. Importantly, it is not determined solely by the marginal
missing-label proportion. Instead, each feature value is weighted by both its
probability of losing the label and the conditional information that the label
would have provided. In contrast,
\(\boldsymbol{I}_{M}^{\mathrm{eff}}\) is the efficient information conveyed by
the missing-label indicators after the missingness-specific parameters have been
removed.

For the finite mixture model considered here,
\(\boldsymbol{I}_{Z\mid\boldsymbol{Y}}\) has a particularly useful form.
Define the class-specific score
\[
\boldsymbol{s}_k(\boldsymbol{y})
=
\nabla_{\boldsymbol{\theta}}
\log
\left\{
\pi_k
f_k(\boldsymbol{y};\boldsymbol{\vartheta}_k)
\right\},
\]
and
\[
\overline{\boldsymbol{s}}(\boldsymbol{y})
=
\sum_{k=1}^{g}
\tau_k(\boldsymbol{y};\boldsymbol{\theta})
\boldsymbol{s}_k(\boldsymbol{y}).
\]
Since
\[
\nabla_{\boldsymbol{\theta}}
\log
\tau_k(\boldsymbol{y};\boldsymbol{\theta})
=
\boldsymbol{s}_k(\boldsymbol{y})
-
\overline{\boldsymbol{s}}(\boldsymbol{y}),
\]
we obtain
\begin{align}
\boldsymbol{I}_{Z\mid\boldsymbol{Y}}
(\boldsymbol{\theta};\boldsymbol{y})
={}&
\sum_{k=1}^{g}
\tau_k(\boldsymbol{y};\boldsymbol{\theta})
\left\{
\boldsymbol{s}_k(\boldsymbol{y})
-
\overline{\boldsymbol{s}}(\boldsymbol{y})
\right\}
\nonumber\\
&\qquad\times
\left\{
\boldsymbol{s}_k(\boldsymbol{y})
-
\overline{\boldsymbol{s}}(\boldsymbol{y})
\right\}^{\top}.
\label{eq:label-info-score-cov}
\end{align}
Thus the information supplied by the class label is the posterior covariance of
the class-specific score vectors. This representation will be useful below when
we examine how uncertainty-dependent missingness redistributes label
information across the feature space.

A direct consequence of Theorem~\ref{thm:information-decomposition} is the
standard missing-completely-at-random benchmark.

\begin{corollary}
\label{cor:MCAR-information}
Suppose
\[
q(\boldsymbol{Y};\boldsymbol{\theta},\boldsymbol{\xi})
=
\gamma,
\qquad
0<\gamma<1,
\]
where \(\gamma\) is variation-independent of \(\boldsymbol{\theta}\). Then
\[
\boldsymbol{I}_{M}^{\mathrm{eff}}
=
\boldsymbol{0},
\]
and
\begin{equation}
\boldsymbol{I}_{\mathrm{PC}}^{\mathrm{eff}}
=
\boldsymbol{I}_{\mathrm{CC}}
-
\gamma
E
\left[
\boldsymbol{I}_{Z\mid\boldsymbol{Y}}
(\boldsymbol{\theta};\boldsymbol{Y})
\right]
\preceq
\boldsymbol{I}_{\mathrm{CC}}.
\label{eq:MCAR-info}
\end{equation}
\end{corollary}

Hence randomly removing class labels cannot increase Fisher information about
the classification model. Any favourable information effect must arise because
the observed pattern of missing labels is itself informative about
\(\boldsymbol{\theta}\).

Theorem~\ref{thm:information-decomposition} concerns the complete identifiable
data-model parameter \(\boldsymbol{\theta}\), after eliminating only the
missingness-specific parameter \(\boldsymbol{\xi}\). In many classification
models, however, some components of \(\boldsymbol{\theta}\) may be nuisance
parameters for a particular classifier parameter or functional of interest.
Eliminating such nuisance parameters does not generally preserve the simple
additive form in \eqref{eq:main-info-decomp}. The resulting efficient
classifier information is considered in the next section.

\section{Efficient classifier information in the presence of nuisance parameters}
\label{sec:classifier-information}

The decomposition in \eqref{eq:main-info-decomp} concerns the full identifiable
parameter \(\boldsymbol{\theta}\), after eliminating only the
missingness-specific parameter \(\boldsymbol{\xi}\). In some classification
models, however, the inferential target depends on only part of
\(\boldsymbol{\theta}\), with the remaining components acting as nuisance
parameters. Eliminating these additional nuisance parameters need not preserve
the additive information decomposition of Section~\ref{sec:model}.

Write
\[
\boldsymbol{\theta}
=
\begin{pmatrix}
\boldsymbol{\beta}\\
\boldsymbol{\lambda}
\end{pmatrix},
\qquad
\boldsymbol{\beta}\in\mathbb{R}^{r},
\quad
\boldsymbol{\lambda}\in\mathbb{R}^{s},
\]
where \(\boldsymbol{\beta}\) is the parameter of interest and
\(\boldsymbol{\lambda}\) is nuisance. Here ``nuisance'' is meant in the
inferential sense: the target depends on \(\boldsymbol{\beta}\), whereas
\(\boldsymbol{\lambda}\) is eliminated by efficient-score projection.
Parameters that determine the Bayes decision boundary are therefore not
nuisance when classification risk itself is the target.

Let
\[
\boldsymbol{A}
=
\boldsymbol{I}_{\mathrm{CC}}(\boldsymbol{\theta}),
\qquad
\boldsymbol{K}
=
\boldsymbol{I}_{M}^{\mathrm{eff}}-\boldsymbol{D},
\qquad
\boldsymbol{J}
=
\boldsymbol{I}_{\mathrm{PC}}^{\mathrm{eff}}
=
\boldsymbol{A}+\boldsymbol{K},
\]
and partition \(\boldsymbol{A}\) and \(\boldsymbol{K}\) conformably with
\((\boldsymbol{\beta}^{\top},\boldsymbol{\lambda}^{\top})^{\top}\):
\[
\boldsymbol{A}
=
\begin{pmatrix}
\boldsymbol{A}_{\beta\beta} &
\boldsymbol{A}_{\beta\lambda}\\
\boldsymbol{A}_{\lambda\beta} &
\boldsymbol{A}_{\lambda\lambda}
\end{pmatrix},
\qquad
\boldsymbol{K}
=
\begin{pmatrix}
\boldsymbol{K}_{\beta\beta} &
\boldsymbol{K}_{\beta\lambda}\\
\boldsymbol{K}_{\lambda\beta} &
\boldsymbol{K}_{\lambda\lambda}
\end{pmatrix}.
\]
Under complete classification, the efficient information for
\(\boldsymbol{\beta}\) is the Schur complement
\begin{equation}
\boldsymbol{A}_{\mathrm{eff}}(\boldsymbol{\beta})
=
\boldsymbol{A}_{\beta\beta}
-
\boldsymbol{A}_{\beta\lambda}
\boldsymbol{A}_{\lambda\lambda}^{-1}
\boldsymbol{A}_{\lambda\beta}.
\label{eq:Aeff-beta}
\end{equation}

To express the corresponding information under partial classification, define
\(\boldsymbol{R}
=
\boldsymbol{A}_{\beta\lambda}
\boldsymbol{A}_{\lambda\lambda}^{-1}\)
and
\begin{align}
\widetilde{\boldsymbol{K}}_{\beta\beta}
={}&
\boldsymbol{K}_{\beta\beta}
-
\boldsymbol{R}\boldsymbol{K}_{\lambda\beta}
-
\boldsymbol{K}_{\beta\lambda}\boldsymbol{R}^{\top}
+
\boldsymbol{R}\boldsymbol{K}_{\lambda\lambda}
\boldsymbol{R}^{\top},
\label{eq:Ktilde-betabeta}\\
\widetilde{\boldsymbol{K}}_{\beta\lambda}
={}&
\boldsymbol{K}_{\beta\lambda}
-
\boldsymbol{R}\boldsymbol{K}_{\lambda\lambda}.
\label{eq:Ktilde-betalambda}
\end{align}
These quantities describe the perturbation induced by partial classification
after orthogonalization with respect to the complete-data nuisance score.

\begin{proposition}
\label{prop:nuisance-coupling}
Assume that
\(\boldsymbol{A}_{\lambda\lambda}\) and
\(\boldsymbol{J}_{\lambda\lambda}
=
\boldsymbol{A}_{\lambda\lambda}
+
\boldsymbol{K}_{\lambda\lambda}\)
are nonsingular. Then the efficient Fisher information for
\(\boldsymbol{\beta}\) under informative partial classification is
\begin{equation}
\boldsymbol{J}_{\mathrm{eff}}(\boldsymbol{\beta})
=
\boldsymbol{A}_{\mathrm{eff}}(\boldsymbol{\beta})
+
\widetilde{\boldsymbol{K}}_{\beta\beta}
-
\widetilde{\boldsymbol{K}}_{\beta\lambda}
\boldsymbol{J}_{\lambda\lambda}^{-1}
\widetilde{\boldsymbol{K}}_{\lambda\beta}.
\label{eq:Jeff-beta}
\end{equation}
Moreover,
\begin{equation}
\boldsymbol{C}_{K}
=
\widetilde{\boldsymbol{K}}_{\beta\lambda}
\boldsymbol{J}_{\lambda\lambda}^{-1}
\widetilde{\boldsymbol{K}}_{\lambda\beta}
\succeq
\boldsymbol{0}.
\label{eq:nuisance-penalty}
\end{equation}
\end{proposition}

The proof is given in Supplementary Section~\ref{supp:nuisance-information}.
Proposition~\ref{prop:nuisance-coupling} shows that eliminating additional
nuisance parameters introduces a nonnegative coupling penalty. In particular,
\begin{equation}
\boldsymbol{J}_{\mathrm{eff}}(\boldsymbol{\beta})
-
\boldsymbol{A}_{\mathrm{eff}}(\boldsymbol{\beta})
=
\widetilde{\boldsymbol{K}}_{\beta\beta}
-
\boldsymbol{C}_{K}.
\label{eq:classifier-info-difference}
\end{equation}
Thus \(\widetilde{\boldsymbol{K}}_{\beta\beta}\) represents the net information
perturbation in the complete-data efficient direction for
\(\boldsymbol{\beta}\), whereas \(\boldsymbol{C}_{K}\) measures the loss
arising from the coupling of that direction with the nuisance score under
partial classification. A positive perturbation in the
classifier-relevant block is therefore not sufficient for an information gain;
it must also dominate this coupling penalty.

If
\(\widetilde{\boldsymbol{K}}_{\beta\lambda}=\boldsymbol{0}\), then
\(\boldsymbol{C}_{K}=\boldsymbol{0}\) and
\begin{equation}
\boldsymbol{J}_{\mathrm{eff}}(\boldsymbol{\beta})
=
\boldsymbol{A}_{\mathrm{eff}}(\boldsymbol{\beta})
+
\widetilde{\boldsymbol{K}}_{\beta\beta}.
\label{eq:Jeff-orthogonal}
\end{equation}
In this case, informative partial classification does not recouple the
complete-data efficient score for \(\boldsymbol{\beta}\) with the nuisance
score, and the additive form is recovered after nuisance elimination.

This distinction is relevant when inference concerns a lower-dimensional
classifier parameter. In the multiclass risk analysis that follows, however,
all parameters determining the Bayes decision boundary are retained in the
classifier-relevant vector. Thus, for quadratic discriminant analysis, the
class probabilities, means, and covariance matrices are not eliminated as
nuisance parameters.
\section{Excess-risk geometry for multiclass Bayes classification}
\label{sec:risk}

The information decompositions in Sections~\ref{sec:model} and
\ref{sec:classifier-information} describe how informative label missingness
affects the precision with which the parameters of a classification model can
be estimated. To determine whether these changes improve classification,
however, the information matrix must be related to the geometry of the Bayes
decision boundary. We develop this connection here.

For \(k=1,\ldots,g\), define the prior-weighted class density
\(r_k(\boldsymbol{y};\boldsymbol{\theta})
=\pi_k f_k(\boldsymbol{y};\boldsymbol{\vartheta}_k)\), and write
\(r_k^0(\boldsymbol{y})
=r_k(\boldsymbol{y};\boldsymbol{\theta}_0)\) at the true parameter
\(\boldsymbol{\theta}_0\). The Bayes classifier is
\begin{equation}
C_0(\boldsymbol{y})
=
\arg\max_{1\leq k\leq g}
r_k^0(\boldsymbol{y}),
\label{eq:bayes-classifier}
\end{equation}
with an arbitrary fixed rule for breaking ties on sets of probability zero.
For \(\boldsymbol{\theta}\) in a neighborhood of
\(\boldsymbol{\theta}_0\), let
\[
C_{\boldsymbol{\theta}}(\boldsymbol{y})
=
\arg\max_k
r_k(\boldsymbol{y};\boldsymbol{\theta})
\]
denote the corresponding plug-in classifier, so that
\(C_{\boldsymbol{\theta}_0}=C_0\). For a local perturbation
\(\boldsymbol{h}\), write
\[
C_{\boldsymbol{h}}(\boldsymbol{y})
:=
C_{\boldsymbol{\theta}_0+\boldsymbol{h}}(\boldsymbol{y}).
\]

For \(k\neq l\), define the pairwise weighted-density contrast
\begin{equation}
g_{kl}(\boldsymbol{y};\boldsymbol{\theta})
=
r_k(\boldsymbol{y};\boldsymbol{\theta})
-
r_l(\boldsymbol{y};\boldsymbol{\theta}).
\label{eq:pairwise-contrast}
\end{equation}
Not every equality surface \(\{g_{kl}=0\}\) contributes to the Bayes decision
boundary. Only the portion on which classes \(k\) and \(l\) jointly attain the
largest prior-weighted density is relevant.

\begin{definition}
\label{def:active-face}
The active Bayes face separating classes \(k\) and \(l\) is
\begin{equation}
\mathcal{F}_{kl}
=
\left\{
\boldsymbol{y}:
r_k^0(\boldsymbol{y})
=
r_l^0(\boldsymbol{y})
>
\max_{m\notin\{k,l\}}
r_m^0(\boldsymbol{y})
\right\}.
\label{eq:active-face}
\end{equation}
\end{definition}

The distinction between a pairwise equality surface and an active face is
essential in the multiclass setting. If
\(r_k^0(\boldsymbol{y})=r_l^0(\boldsymbol{y})\) while a third class has a
strictly larger weighted density, perturbing the \(k\)-versus-\(l\)
comparison does not change the Bayes decision locally and therefore does not
contribute to the leading excess risk.

For \(\boldsymbol{s}\in\mathcal{F}_{kl}\), let
\(\boldsymbol{b}_{kl}(\boldsymbol{s})
=\nabla_{\boldsymbol{\theta}}
g_{kl}(\boldsymbol{s};\boldsymbol{\theta}_0)\).
This vector describes the first-order change in the \(k\)-versus-\(l\)
contrast induced by a local parameter perturbation at the boundary point
\(\boldsymbol{s}\).

We impose regularity conditions excluding singular decision geometry. For
each nonempty active face, assume
\begin{equation}
\nabla_{\boldsymbol{y}}
g_{kl}(\boldsymbol{s};\boldsymbol{\theta}_0)
\neq
\boldsymbol{0},
\qquad
\boldsymbol{s}\in\mathcal{F}_{kl}.
\label{eq:regular-face}
\end{equation}
The regular part of \(\mathcal{F}_{kl}\) is then a
\((p-1)\)-dimensional smooth hypersurface. At points where \(r\geq3\) classes
tie at the Bayes maximum, we additionally assume transversality: \(r-1\)
independent pairwise contrast gradients span the normal space of the
corresponding tie stratum. Ties of positive Lebesgue measure are excluded. When the feature support is
unbounded, we additionally assume the tail condition stated in Supplementary
Section~\ref{supp:noncompact}. When one or more active faces are noncompact, we also assume the
regular-exhaustion and boundary-integrability conditions stated there.

Let
\[
R(\boldsymbol{\theta})
=
\Pr_{\boldsymbol{\theta}_0}
\left\{
C_{\boldsymbol{\theta}}(\boldsymbol{Y})\neq Z
\right\}
\]
denote the classification risk of the rule determined by
\(\boldsymbol{\theta}\), evaluated under the true distribution
\(\boldsymbol{\theta}_0\), and let
\(R^\ast=R(\boldsymbol{\theta}_0)\) denote the Bayes risk. The local geometry
of the excess risk is characterized by the following result.

\begin{theorem}
\label{thm:multiclass-risk}
Suppose that the active Bayes boundary lies in the interior of the common
feature support and that the class-weighted densities are twice continuously
differentiable in \((\boldsymbol{y},\boldsymbol{\theta})\) in a neighborhood
of the active Bayes boundary. Assume that the active pairwise
faces satisfy \eqref{eq:regular-face}, that higher-order active tie sets
satisfy the stated transversality condition, and that ties of positive
Lebesgue measure are absent. When the feature support is unbounded, we additionally assume the tail condition stated in Supplementary
Section~\ref{supp:noncompact}. When one or more active faces are noncompact, we also assume the
regular-exhaustion and boundary-integrability conditions stated there. Then,
as \(\boldsymbol{h}\to\boldsymbol{0}\),
\begin{equation}
R(\boldsymbol{\theta}_0+\boldsymbol{h})
-
R^\ast
=
\frac{1}{2}
\boldsymbol{h}^{\top}
\boldsymbol{H}_{R}
\boldsymbol{h}
+
o\!\left(\|\boldsymbol{h}\|^2\right),
\label{eq:risk-quadratic}
\end{equation}
where
\begin{equation}
\boldsymbol{H}_{R}
=
\sum_{1\leq k<l\leq g}
\boldsymbol{H}_{kl},
\qquad
\boldsymbol{H}_{kl}
=
\int_{\mathcal{F}_{kl}}
\frac{
\boldsymbol{b}_{kl}(\boldsymbol{s})
\boldsymbol{b}_{kl}(\boldsymbol{s})^{\top}
}{
\left\|
\nabla_{\boldsymbol{y}}
g_{kl}(\boldsymbol{s};\boldsymbol{\theta}_0)
\right\|
}
\,dS(\boldsymbol{s}).
\label{eq:Hkl}
\end{equation}
In particular, \(\boldsymbol{H}_{R}\succeq\boldsymbol{0}\). Generic triple
and higher-order Bayes junctions do not contribute to the quadratic term.
\end{theorem}
A proof is given in Supplementary Sections~\ref{supp:risk-proof}--\ref{supp:noncompact}. The argument localizes
classification disagreement to a thin neighborhood of the Bayes boundary and
uses the pairwise contrast \(g_{kl}\) as a coordinate normal to each active
face. The loss incurred by crossing the \(k\)-versus-\(l\) boundary is then
\(|g_{kl}|\), and integration across the displaced boundary yields the
quadratic contribution in \eqref{eq:Hkl}. Under transversality, higher-order
junctions are lower-dimensional and contribute only to higher-order terms.

The matrix \(\boldsymbol{H}_{R}\) therefore defines a local
classification-relevance metric on the parameter space. For any direction
\(\boldsymbol{v}\),
\begin{equation}
\boldsymbol{v}^{\top}
\boldsymbol{H}_{R}
\boldsymbol{v}
=
\sum_{k<l}
\int_{\mathcal{F}_{kl}}
\frac{
\left\{
\boldsymbol{b}_{kl}(\boldsymbol{s})^{\top}
\boldsymbol{v}
\right\}^{2}
}{
\left\|
\nabla_{\boldsymbol{y}}
g_{kl}(\boldsymbol{s};\boldsymbol{\theta}_0)
\right\|
}
\,dS(\boldsymbol{s}).
\label{eq:HR-direction}
\end{equation}
Thus a direction \(\boldsymbol{v}\) is locally irrelevant to classification
if and only if
\(\boldsymbol{b}_{kl}(\boldsymbol{s})^{\top}\boldsymbol{v}=0\)
for \(dS\)-almost every \(\boldsymbol{s}\) on every active face
\(\mathcal{F}_{kl}\); equivalently, it produces no first-order displacement
of the active Bayes boundary.

An equivalent representation in terms of pairwise log-contrasts, which is
convenient for the QDA calculations in Section~\ref{sec:qda}, is given in
Supplementary Section~\ref{supp:noncompact}.

Theorem~\ref{thm:multiclass-risk} immediately yields an asymptotic
classification-risk expansion for regular parameter estimators.

\begin{corollary}
\label{cor:asymptotic-risk}
Suppose that
\[
\sqrt{n}
\left(
\widehat{\boldsymbol{\theta}}_n-\boldsymbol{\theta}_0
\right)
\overset{d}{\longrightarrow}
N(\boldsymbol{0},\boldsymbol{V}),
\]
and that, for some \(\delta>0\),
\[
\sup_n
E\left\|
\sqrt n
\left(
\widehat{\boldsymbol{\theta}}_n-\boldsymbol{\theta}_0
\right)
\right\|^{2+\delta}
<\infty.
\]
Then
\begin{equation}
n
\left\{
R(\widehat{\boldsymbol{\theta}}_n)-R^\ast
\right\}
\overset{d}{\longrightarrow}
\frac{1}{2}
\boldsymbol{Z}^{\top}
\boldsymbol{H}_{R}
\boldsymbol{Z},
\qquad
\boldsymbol{Z}\sim N(\boldsymbol{0},\boldsymbol{V}),
\label{eq:risk-limit}
\end{equation}
and
\begin{equation}
E
\left\{
R(\widehat{\boldsymbol{\theta}}_n)
\right\}
-
R^\ast
=
\frac{1}{2n}
\operatorname{tr}
\left(
\boldsymbol{H}_{R}\boldsymbol{V}
\right)
+
o(n^{-1}).
\label{eq:expected-risk}
\end{equation}
\end{corollary}

The proof and the corresponding weighted-\(\chi^2\) representation of the
limiting quadratic form are given in Supplementary Section~\ref{supp:noncompact}. For an
efficient likelihood estimator,
\(\boldsymbol{V}=\boldsymbol{I}_{\mathrm{eff}}^{-1}\), and hence
\begin{equation}
E
\left\{
R(\widehat{\boldsymbol{\theta}}_n)
\right\}
-
R^\ast
=
\frac{1}{2n}
\operatorname{tr}
\left(
\boldsymbol{H}_{R}
\boldsymbol{I}_{\mathrm{eff}}^{-1}
\right)
+
o(n^{-1}).
\label{eq:risk-information}
\end{equation}
This expression provides the required link between the information
calculations of Sections~\ref{sec:model}--\ref{sec:classifier-information}
and classification performance: information gains matter only insofar as they
occur in directions to which \(\boldsymbol{H}_{R}\) assigns classification
relevance.

Let \(\boldsymbol{A}\) and \(\boldsymbol{J}\) denote the relevant efficient
information matrices under complete and informative partial classification,
respectively. Their leading expected excess-risk difference is determined by
\begin{equation}
\Delta_R
=
\operatorname{tr}
\left[
\boldsymbol{H}_{R}
\left(
\boldsymbol{A}^{-1}
-
\boldsymbol{J}^{-1}
\right)
\right],
\label{eq:DeltaR}
\end{equation}
since
\begin{equation}
E
\left\{
R(\widehat{\boldsymbol{\theta}}_{\mathrm{CC}})
\right\}
-
E
\left\{
R(\widehat{\boldsymbol{\theta}}_{\mathrm{PC}})
\right\}
=
\frac{\Delta_R}{2n}
+
o(n^{-1}).
\label{eq:risk-difference}
\end{equation}
Thus \(\Delta_R>0\) characterizes favourable informative missingness at the
leading \(n^{-1}\) order in expected classification risk. The next section
examines this criterion through the relative information geometry of the two
experiments.

\section{Classification-weighted information and favourable missingness}
\label{sec:spectral}

The previous section shows that, for a regular estimator with asymptotic
covariance matrix \(\boldsymbol{V}/n\), the leading excess classification risk
is governed by
\[
\operatorname{tr}
\left(
\boldsymbol{H}_{R}\boldsymbol{V}
\right).
\]
Thus the comparison between complete and partial classification depends not
only on the amount of information available, but also on the directions in
which that information is gained or lost relative to the geometry of the
active Bayes boundary.

Let \(\boldsymbol{A}\) and \(\boldsymbol{J}\) denote the relevant efficient
information matrices under complete and informative partial classification,
respectively, and assume throughout that
\(\boldsymbol{A}\succ\boldsymbol{0}\) and
\(\boldsymbol{J}\succ\boldsymbol{0}\). The leading difference in expected
excess classification risk is
\begin{equation}
\Delta_R
=
\operatorname{tr}
\left[
\boldsymbol{H}_{R}
\left(
\boldsymbol{A}^{-1}
-
\boldsymbol{J}^{-1}
\right)
\right].
\label{eq:DeltaR-main}
\end{equation}
By \eqref{eq:risk-difference}, \(\Delta_R>0\) implies that informative partial
classification has the smaller expected excess classification risk at the
leading \(n^{-1}\) order.

The corresponding asymptotic relative efficiency is
\begin{equation}
\operatorname{ARE}_{R}
=
\frac{
\operatorname{tr}
\left(
\boldsymbol{H}_{R}\boldsymbol{A}^{-1}
\right)
}{
\operatorname{tr}
\left(
\boldsymbol{H}_{R}\boldsymbol{J}^{-1}
\right)
},
\label{eq:classification-ARE}
\end{equation}
provided the denominator is positive. Hence
\(\operatorname{ARE}_{R}>1\) if and only if \(\Delta_R>0\). We use
\(\Delta_R\) as the primary criterion because it admits an additive directional
decomposition.

Let \(\boldsymbol{A}^{1/2}\) denote the symmetric positive-definite square root
of \(\boldsymbol{A}\), with inverse \(\boldsymbol{A}^{-1/2}\), and define
\begin{equation}
\boldsymbol{C}
=
\boldsymbol{A}^{-1/2}
\boldsymbol{J}
\boldsymbol{A}^{-1/2},
\qquad
\boldsymbol{W}
=
\boldsymbol{A}^{-1/2}
\boldsymbol{H}_{R}
\boldsymbol{A}^{-1/2}.
\label{eq:CW-matrices}
\end{equation}
Write
\[
\boldsymbol{C}
=
\boldsymbol{Q}
\boldsymbol{\Lambda}
\boldsymbol{Q}^{\top},
\qquad
\boldsymbol{\Lambda}
=
\operatorname{diag}
(\lambda_1,\ldots,\lambda_r),
\]
where \(\boldsymbol{Q}^{\top}\boldsymbol{Q}=\boldsymbol{I}\) and
\(\lambda_j>0\). For the \(j\)th eigenvector
\(\boldsymbol{q}_j\), define
\begin{equation}
w_j
=
\boldsymbol{q}_j^{\top}
\boldsymbol{W}
\boldsymbol{q}_j
\geq 0.
\label{eq:w-j}
\end{equation}

\begin{theorem}
\label{thm:spectral}
Under the conditions above,
\begin{equation}
\Delta_R
=
\sum_{j=1}^{r}
w_j
\frac{\lambda_j-1}{\lambda_j}.
\label{eq:spectral-Delta}
\end{equation}
Consequently, if
\(\mathcal{P}=\{j:\lambda_j>1\}\) and
\(\mathcal{N}=\{j:\lambda_j<1\}\), then
\(\Delta_R>0\) if and only if
\begin{equation}
\sum_{j\in\mathcal{P}}
w_j
\frac{\lambda_j-1}{\lambda_j}
>
\sum_{j\in\mathcal{N}}
w_j
\frac{1-\lambda_j}{\lambda_j}.
\label{eq:spectral-condition}
\end{equation}
\end{theorem}

The proof is given in Supplementary Section~\ref{supp:spectral-proof}.

The generalized eigenvalues in Theorem~\ref{thm:spectral} have a direct
statistical interpretation. If
\(\boldsymbol{v}_j=\boldsymbol{A}^{-1/2}\boldsymbol{q}_j\), then
\begin{equation}
\boldsymbol{J}\boldsymbol{v}_j
=
\lambda_j
\boldsymbol{A}\boldsymbol{v}_j,
\label{eq:generalized-eigen}
\end{equation}
and, under the normalization
\(\boldsymbol{v}_j^{\top}\boldsymbol{A}\boldsymbol{v}_j=1\),
\[
\lambda_j
=
\boldsymbol{v}_j^{\top}
\boldsymbol{J}
\boldsymbol{v}_j.
\]
Thus \(\lambda_j>1\) indicates greater information under partial
classification than under complete classification in direction
\(\boldsymbol{v}_j\), whereas \(\lambda_j<1\) indicates less information
in that direction.

Moreover,
\begin{equation}
w_j
=
\boldsymbol{v}_j^{\top}
\boldsymbol{H}_{R}
\boldsymbol{v}_j.
\label{eq:w-risk-direction}
\end{equation}
Hence \(w_j\) measures the classification relevance of the same direction.
By \eqref{eq:HR-direction}, it is large when perturbation in
\(\boldsymbol{v}_j\) moves one or more active Bayes faces substantially.

Theorem~\ref{thm:spectral} therefore characterizes favourable informative
missingness as an alignment phenomenon. Global Fisher-information dominance is
not required: \(\boldsymbol{J}\not\succeq\boldsymbol{A}\) is compatible with
\(\Delta_R>0\) whenever information gains occur primarily in directions
receiving large classification weights.

For interpretation, define
\begin{equation}
G_R
=
\sum_{\lambda_j>1}
w_j
\frac{\lambda_j-1}{\lambda_j},
\qquad
L_R
=
\sum_{\lambda_j<1}
w_j
\frac{1-\lambda_j}{\lambda_j}.
\label{eq:GR-LR}
\end{equation}
Then
\begin{equation}
\Delta_R
=
G_R-L_R.
\label{eq:Delta-gain-loss}
\end{equation}
Thus informative partial classification is favourable precisely when its
classification-weighted information gains exceed the corresponding losses.

Two limiting cases follow immediately. If
\(\boldsymbol{J}\succeq\boldsymbol{A}\), then all
\(\lambda_j\geq1\) and hence \(\Delta_R\geq0\). Conversely, if
\(\boldsymbol{J}\preceq\boldsymbol{A}\), then
\(\Delta_R\leq0\). Neither matrix ordering is required in the general case.

\subsection{Local departures from missing completely at random}
\label{subsec:local-MCAR}

We now examine how favourable informative missingness can emerge as the
missing-label mechanism departs continuously from missing completely at random.
Let \(U_{\boldsymbol{\theta}}(\boldsymbol{Y})\) be a smooth scalar measure of
classification uncertainty and consider
\begin{equation}
q_t(\boldsymbol{y};\boldsymbol{\theta})
=
\operatorname{expit}
\left\{
\alpha(t)
+
t\,U_{\boldsymbol{\theta}}(\boldsymbol{y})
\right\},
\qquad
t\geq0.
\label{eq:q-t}
\end{equation}
For comparisons across \(t\), the intercept \(\alpha(t)\) is calibrated at the
true parameter \(\boldsymbol{\theta}_0\) so that
\begin{equation}
E_{\boldsymbol{\theta}_0}
\left[
q_t(\boldsymbol{Y};\boldsymbol{\theta}_0)
\right]
=
\gamma,
\qquad
0<\gamma<1.
\label{eq:fixed-gamma}
\end{equation}
Thus \(t\) changes the locations at which labels are unavailable while keeping
their marginal proportion fixed. At \(t=0\), \(q_0(\boldsymbol{y})=\gamma\),
and the mechanism reduces to MCAR.

Equation~\eqref{eq:fixed-gamma} defines a sequence of population experiments;
it is not imposed as a constraint on the likelihood for arbitrary
\(\boldsymbol{\theta}\). In estimation, the missingness intercept and slope
remain ordinary nuisance parameters.

Write
\(U(\boldsymbol{Y})=U_{\boldsymbol{\theta}_0}(\boldsymbol{Y})\) and
\[
\boldsymbol{G}(\boldsymbol{Y})
=
\nabla_{\boldsymbol{\theta}}
U_{\boldsymbol{\theta}}(\boldsymbol{Y})
\big|_{\boldsymbol{\theta}=\boldsymbol{\theta}_0}.
\]
The following proposition describes the local behavior of the information
decomposition around MCAR.

\begin{proposition}
\label{prop:local-MCAR}
Assume that
\[
\operatorname{Var}\{U(\boldsymbol{Y})\}>0,\qquad
E\{U(\boldsymbol{Y})^2\}<\infty,\qquad
E\{\|\boldsymbol{G}(\boldsymbol{Y})\|^2\}<\infty,
\]
and that the required differentiations with respect to \(t\) and
\(\boldsymbol{\theta}\) may be interchanged with expectation in a neighborhood
of \((t,\boldsymbol{\theta})=(0,\boldsymbol{\theta}_0)\), with sufficient local
smoothness for the Taylor expansions below. Then
\begin{equation}
\alpha(0)
=
\operatorname{logit}(\gamma),
\qquad
\alpha'(0)
=
-
E\{U(\boldsymbol{Y})\},
\label{eq:alpha-local}
\end{equation}
and
\begin{equation}
q_t(\boldsymbol{Y})
=
\gamma
+
t\gamma(1-\gamma)
\left[
U(\boldsymbol{Y})
-
E\{U(\boldsymbol{Y})\}
\right]
+
O(t^2).
\label{eq:q-local}
\end{equation}

The label-information loss
\[
\boldsymbol{D}(t)
=
E
\left[
q_t(\boldsymbol{Y})
\boldsymbol{I}_{Z\mid\boldsymbol{Y}}(\boldsymbol{Y})
\right]
\]
satisfies
\begin{equation}
\boldsymbol{D}'(0)
=
\gamma(1-\gamma)
E
\left[
\left[
U(\boldsymbol{Y})-E\{U(\boldsymbol{Y})\}
\right]
\boldsymbol{I}_{Z\mid\boldsymbol{Y}}(\boldsymbol{Y})
\right].
\label{eq:Dprime}
\end{equation}
By contrast, the efficient information contributed by the missing-label
mechanism satisfies
\begin{equation}
\boldsymbol{I}_{M}^{\mathrm{eff}}(t)
=
t^2
\gamma(1-\gamma)
\boldsymbol{\mathcal V}_{U}
+
o(t^2),
\label{eq:IM-local}
\end{equation}
where
\begin{equation}
\boldsymbol{\mathcal V}_{U}
=
E
\left[
\boldsymbol{G}
\boldsymbol{G}^{\top}
\right]
-
E
\left[
\boldsymbol{G}
\boldsymbol{X}^{\top}
\right]
E
\left[
\boldsymbol{X}
\boldsymbol{X}^{\top}
\right]^{-1}
E
\left[
\boldsymbol{X}
\boldsymbol{G}^{\top}
\right]
\succeq
\boldsymbol{0},
\label{eq:VU}
\end{equation}
with
\[
\boldsymbol{X}
=
\begin{pmatrix}
1\\
U(\boldsymbol{Y})
\end{pmatrix}.
\]

Finally, if
\[
\boldsymbol{J}_0
=
\boldsymbol{A}
-
\gamma
E
\left[
\boldsymbol{I}_{Z\mid\boldsymbol{Y}}
(\boldsymbol{Y})
\right],
\]
then the derivative of the classification advantage at MCAR is
\begin{equation}
\Delta_R'(0)
=
-
\operatorname{tr}
\left[
\boldsymbol{H}_{R}
\boldsymbol{J}_0^{-1}
\boldsymbol{D}'(0)
\boldsymbol{J}_0^{-1}
\right].
\label{eq:Delta-prime}
\end{equation}
\end{proposition}

The proof is given in Supplementary Section~\ref{supp:local-MCAR-proof}.

Proposition~\ref{prop:local-MCAR} reveals an asymmetry in how informative
missingness first enters the experiment. Redistributing missing labels changes
the conditional class-label information loss at order \(t\), whereas the
efficient information carried by the missingness indicators appears only at
order \(t^2\). Weak dependence between label absence and classification
uncertainty therefore need not improve classification.

In particular, if
\(\boldsymbol{D}'(0)\succeq\boldsymbol{0}\)
and
\[
\operatorname{tr}
\left[
\boldsymbol{H}_{R}
\boldsymbol{J}_0^{-1}
\boldsymbol{D}'(0)
\boldsymbol{J}_0^{-1}
\right]
>0,
\]
then
\(\Delta_R'(0)<0\). A small departure from MCAR then initially worsens
classification. This can occur, for example, when observations with
above-average uncertainty are also those for which the true class label carries
greater classification-relevant information.

The result is local: it does not imply that \(\Delta_R(t)\) is monotone, nor
that it must eventually become positive. Whether a favourable regime emerges
for larger \(t\) depends on the global information geometry of the missingness
mechanism. This behavior is examined numerically in
Section~\ref{sec:numerical}.

\section{Three-class quadratic discriminant analysis}
\label{sec:qda}

We illustrate the preceding theory using Gaussian quadratic discriminant
analysis (QDA). The purpose is not to introduce a new classification model, but
to examine the information--risk framework in a setting where the Bayes
geometry is genuinely multiclass and nonlinear. Unequal covariance matrices
produce curved pairwise decision boundaries, all distributional parameters may
affect the classifier, and different active faces may meet at multiclass
junctions.

Let \(Z\in\{1,2,3\}\), with \(\Pr(Z=k)=\pi_k\), and
\begin{equation}
\boldsymbol{Y}\mid Z=k
\sim
N_p
\left(
\boldsymbol{\mu}_k,
\boldsymbol{\Sigma}_k
\right),
\qquad
k=1,2,3,
\label{eq:qda-model}
\end{equation}
where \(\pi_k>0\), \(\sum_{k=1}^{3}\pi_k=1\), and
\(\boldsymbol{\Sigma}_k\succ\boldsymbol{0}\). The prior-weighted class
densities are
\begin{equation}
r_k(\boldsymbol{y})
=
\pi_k
\phi_p
\left(
\boldsymbol{y};
\boldsymbol{\mu}_k,
\boldsymbol{\Sigma}_k
\right),
\label{eq:qda-rk}
\end{equation}
and the Bayes rule assigns \(\boldsymbol{y}\) to the class maximizing
\(r_k(\boldsymbol{y})\).

Although the theory in Sections~\ref{sec:model}--\ref{sec:spectral} is stated
for arbitrary \(g\) and \(p\), the numerical analysis uses \(g=3\) and \(p=2\).
This is the smallest configuration that allows curved active pairwise faces and
a genuine three-class junction to occur simultaneously while remaining directly visualizable.

\subsection{Bayes geometry and risk curvature}
\label{subsec:qda-geometry}

For \(k\neq l\), define the log weighted-density contrast
\begin{equation}
d_{kl}(\boldsymbol{y})
=
\log
\frac{
r_k(\boldsymbol{y})
}{
r_l(\boldsymbol{y})
}.
\label{eq:qda-dkl}
\end{equation}
Under \eqref{eq:qda-model},
\begin{align}
d_{kl}(\boldsymbol{y})
={}&
\log\frac{\pi_k}{\pi_l}
-
\frac{1}{2}
\log
\frac{
|\boldsymbol{\Sigma}_k|
}{
|\boldsymbol{\Sigma}_l|
}
\nonumber\\
&-
\frac{1}{2}
(\boldsymbol{y}-\boldsymbol{\mu}_k)^{\top}
\boldsymbol{\Sigma}_k^{-1}
(\boldsymbol{y}-\boldsymbol{\mu}_k)
\nonumber\\
&+
\frac{1}{2}
(\boldsymbol{y}-\boldsymbol{\mu}_l)^{\top}
\boldsymbol{\Sigma}_l^{-1}
(\boldsymbol{y}-\boldsymbol{\mu}_l).
\label{eq:qda-discriminant}
\end{align}
Thus \(\{\boldsymbol{y}:d_{kl}(\boldsymbol{y})=0\}\) is a quadratic
hypersurface. For three classes, the corresponding active face is
\begin{equation}
\mathcal{F}_{kl}
=
\left\{
\boldsymbol{y}:
d_{kl}(\boldsymbol{y})=0,\;
r_k(\boldsymbol{y})=r_l(\boldsymbol{y})>r_m(\boldsymbol{y})
\right\},
\label{eq:qda-active-face}
\end{equation}
where \(m\notin\{k,l\}\).

Differentiation with respect to the feature vector gives
\begin{equation}
\nabla_{\boldsymbol{y}}
d_{kl}(\boldsymbol{y})
=
-
\boldsymbol{\Sigma}_k^{-1}
(\boldsymbol{y}-\boldsymbol{\mu}_k)
+
\boldsymbol{\Sigma}_l^{-1}
(\boldsymbol{y}-\boldsymbol{\mu}_l),
\label{eq:qda-spatial-gradient}
\end{equation}
so a regular point of \(\mathcal{F}_{kl}\) satisfies
\[
\boldsymbol{\Sigma}_k^{-1}
(\boldsymbol{y}-\boldsymbol{\mu}_k)
\neq
\boldsymbol{\Sigma}_l^{-1}
(\boldsymbol{y}-\boldsymbol{\mu}_l).
\]

At a three-class junction \(\boldsymbol{y}^{\dagger}\),
\begin{equation}
r_1(\boldsymbol{y}^{\dagger})
=
r_2(\boldsymbol{y}^{\dagger})
=
r_3(\boldsymbol{y}^{\dagger}),
\label{eq:qda-triple-tie}
\end{equation}
or equivalently
\(d_{12}(\boldsymbol{y}^{\dagger})
=d_{13}(\boldsymbol{y}^{\dagger})=0\).
For \(p=2\), transversality requires
\begin{equation}
\det
\begin{pmatrix}
\nabla_{\boldsymbol{y}}
d_{12}(\boldsymbol{y}^{\dagger})^{\top}
\\
\nabla_{\boldsymbol{y}}
d_{13}(\boldsymbol{y}^{\dagger})^{\top}
\end{pmatrix}
\neq0.
\label{eq:qda-transversality}
\end{equation}
Under this condition the junction is isolated and, by
Theorem~\ref{thm:multiclass-risk}, contributes only a higher-order term to the
local excess risk.

For the class probabilities, we use class 3 as the baseline and write
\begin{equation}
\alpha_1
=
\log\frac{\pi_1}{\pi_3},
\qquad
\alpha_2
=
\log\frac{\pi_2}{\pi_3}.
\label{eq:qda-logit-priors}
\end{equation}
A convenient theoretical parameter vector is
\[
\boldsymbol{\theta}
=
\left(
\boldsymbol{\alpha}^{\top},
\boldsymbol{\mu}_1^{\top},
\boldsymbol{\mu}_2^{\top},
\boldsymbol{\mu}_3^{\top},
\operatorname{vech}(\boldsymbol{\Sigma}_1)^{\top},
\operatorname{vech}(\boldsymbol{\Sigma}_2)^{\top},
\operatorname{vech}(\boldsymbol{\Sigma}_3)^{\top}
\right)^{\top}.
\]
For numerical optimization, the covariance matrices are represented instead
through unconstrained log-Cholesky coordinates, which guarantee positive
definiteness. Information and risk quantities are transformed consistently
between the two parameterizations; the scalar criteria in
Section~\ref{sec:spectral} are invariant under smooth nonsingular
reparameterization.

The derivatives of the pairwise discriminant with respect to the two class
means are
\begin{equation}
\nabla_{\boldsymbol{\mu}_k}
d_{kl}(\boldsymbol{y})
=
\boldsymbol{\Sigma}_k^{-1}
(\boldsymbol{y}-\boldsymbol{\mu}_k),
\qquad
\nabla_{\boldsymbol{\mu}_l}
d_{kl}(\boldsymbol{y})
=
-
\boldsymbol{\Sigma}_l^{-1}
(\boldsymbol{y}-\boldsymbol{\mu}_l).
\label{eq:qda-mean-gradients}
\end{equation}
The corresponding covariance derivatives and their transformation to
\(\operatorname{vech}\) coordinates are given in Supplementary
Section~\ref{supp:qda-calculations}, while the log-Cholesky parameterization
used for numerical optimization is described in Supplementary
Section~\ref{supp:finite-computation}.

Using the equivalent log-contrast representation given in Supplementary
Section~\ref{supp:noncompact}, the contribution of the active \(k\)-versus-\(l\) face to the
classification-risk curvature is
\begin{equation}
\boldsymbol{H}_{kl}
=
\int_{\mathcal{F}_{kl}}
\frac{
c_{kl}(\boldsymbol{s})
}{
\left\|
\nabla_{\boldsymbol{y}}
d_{kl}(\boldsymbol{s})
\right\|
}
\boldsymbol{a}_{kl}(\boldsymbol{s})
\boldsymbol{a}_{kl}(\boldsymbol{s})^{\top}
\,dS(\boldsymbol{s}),
\label{eq:qda-Hkl}
\end{equation}
where
\(c_{kl}(\boldsymbol{s})
=r_k(\boldsymbol{s})=r_l(\boldsymbol{s})\) and
\(\boldsymbol{a}_{kl}(\boldsymbol{s})
=\nabla_{\boldsymbol{\theta}}d_{kl}(\boldsymbol{s})\).
For three classes,
\begin{equation}
\boldsymbol{H}_{R}
=
\boldsymbol{H}_{12}
+
\boldsymbol{H}_{13}
+
\boldsymbol{H}_{23}.
\label{eq:qda-HR}
\end{equation}
This face-wise decomposition later allows classification gains and losses to be
attributed to individual pairwise decision boundaries.

\subsection{Information and uncertainty-dependent missingness}
\label{subsec:qda-information-missingness}

For QDA, the complete-classification Fisher information has a convenient block
structure. For
\(\boldsymbol{\alpha}=(\alpha_1,\alpha_2)^{\top}\),
\begin{equation}
\boldsymbol{I}_{\alpha}
=
\begin{pmatrix}
\pi_1(1-\pi_1) & -\pi_1\pi_2\\
-\pi_1\pi_2 & \pi_2(1-\pi_2)
\end{pmatrix},
\label{eq:qda-Ialpha}
\end{equation}
while the information for the mean and covariance parameters of class \(k\) is
\begin{equation}
\boldsymbol{I}_{\mu_k}
=
\pi_k\boldsymbol{\Sigma}_k^{-1},
\qquad
\boldsymbol{I}_{\Sigma_k}
=
\frac{\pi_k}{2}
\boldsymbol{D}_p^{\top}
\left(
\boldsymbol{\Sigma}_k^{-1}
\otimes
\boldsymbol{\Sigma}_k^{-1}
\right)
\boldsymbol{D}_p,
\label{eq:qda-complete-blocks}
\end{equation}
where \(\boldsymbol{D}_p\) is the duplication matrix. The prior score is
orthogonal to the within-class distributional scores, and the mean and
covariance scores are orthogonal in expectation within each class. Hence the
complete-classification information is block diagonal in these coordinates.
Further details are given in Supplementary Section~\ref{supp:qda-calculations}.

The general theory does not require a particular missing-label mechanism. For
the numerical analysis, we consider mechanisms driven by posterior
classification uncertainty, with
\[
\tau_k(\boldsymbol{y})
=
\frac{
r_k(\boldsymbol{y})
}{
\sum_{l=1}^{3}r_l(\boldsymbol{y})
}.
\]
Our primary uncertainty measure is normalized Shannon entropy,
\begin{equation}
U_H(\boldsymbol{y})
=
\frac{
-\sum_{k=1}^{3}
\tau_k(\boldsymbol{y})
\log\tau_k(\boldsymbol{y})
}{
\log 3
},
\qquad
0\leq U_H(\boldsymbol{y})\leq1,
\label{eq:normalized-entropy}
\end{equation}
with missing-label probability
\begin{equation}
q_H(\boldsymbol{y})
=
\operatorname{expit}
\left\{
\xi_0+\xi_1 U_H(\boldsymbol{y})
\right\}.
\label{eq:entropy-missingness}
\end{equation}
Positive \(\xi_1\) makes labels more likely to be unavailable in regions of
greater posterior uncertainty.

To assess sensitivity to the uncertainty functional, we also consider
normalized Gini uncertainty,
\begin{equation}
U_G(\boldsymbol{y})
=
\frac{3}{2}
\left\{
1-\sum_{k=1}^{3}
\tau_k(\boldsymbol{y})^2
\right\},
\qquad
0\leq U_G(\boldsymbol{y})\leq1,
\label{eq:normalized-gini}
\end{equation}
with
\begin{equation}
q_G(\boldsymbol{y})
=
\operatorname{expit}
\left\{
\zeta_0+\zeta_1 U_G(\boldsymbol{y})
\right\}.
\label{eq:gini-missingness}
\end{equation}
Because both uncertainty measures are normalized to \([0,1]\), their slope
parameters have a more comparable interpretation than under their unscaled
forms.

For a specified marginal missing-label proportion \(\gamma\), the intercept in
either mechanism is calibrated so that
\begin{equation}
E\{q(\boldsymbol{Y})\}
=
\gamma.
\label{eq:qda-fixed-gamma}
\end{equation}
Changing the slope therefore redistributes the unavailable labels across
feature space while preserving their expected proportion, corresponding
directly to the fixed-\(\gamma\) comparison in
Proposition~\ref{prop:local-MCAR}.

For a generic uncertainty functional
\(U_{\boldsymbol{\theta}}(\boldsymbol{y})\),
\begin{equation}
\nabla_{\boldsymbol{\theta}}
\tau_k(\boldsymbol{y})
=
\tau_k(\boldsymbol{y})
\left\{
\boldsymbol{s}_k(\boldsymbol{y})
-
\overline{\boldsymbol{s}}(\boldsymbol{y})
\right\},
\label{eq:qda-posterior-gradient}
\end{equation}
where \(\boldsymbol{s}_k\) and
\(\overline{\boldsymbol{s}}\) are defined in Section~\ref{sec:model}.
These derivatives enter the missingness-information blocks in
Theorem~\ref{thm:information-decomposition}; explicit entropy and Gini
derivatives are given in Supplementary Section~\ref{supp:qda-calculations}.

\subsection{Population criteria for the numerical analysis}
\label{subsec:qda-population-quantities}

For a specified QDA configuration and missing-label mechanism, the population
comparison requires the complete-classification information
\(\boldsymbol{A}=\boldsymbol{I}_{\mathrm{CC}}\), the partially classified
information
\begin{equation}
\boldsymbol{J}
=
\boldsymbol{A}
-
\boldsymbol{D}
+
\boldsymbol{I}_{M}^{\mathrm{eff}},
\qquad
\boldsymbol{D}
=
E
\left[
q(\boldsymbol{Y})
\boldsymbol{I}_{Z\mid\boldsymbol{Y}}(\boldsymbol{Y})
\right],
\label{eq:qda-J}
\end{equation}
and the classification-risk curvature
\(\boldsymbol{H}_{R}\) from
\eqref{eq:qda-Hkl}--\eqref{eq:qda-HR}. Here
\[
\boldsymbol{I}_{M}^{\mathrm{eff}}
=
\boldsymbol{B}_{\boldsymbol{\theta}\boldsymbol{\theta}}
-
\boldsymbol{B}_{\boldsymbol{\theta}\boldsymbol{\xi}}
\boldsymbol{B}_{\boldsymbol{\xi}\boldsymbol{\xi}}^{-1}
\boldsymbol{B}_{\boldsymbol{\xi}\boldsymbol{\theta}},
\]
as in Theorem~\ref{thm:information-decomposition}.

The corresponding leading excess-risk coefficients are
\begin{equation}
\mathcal{E}_{\mathrm{CC}}
=
\operatorname{tr}
\left(
\boldsymbol{H}_{R}\boldsymbol{A}^{-1}
\right),
\qquad
\mathcal{E}_{\mathrm{PC}}
=
\operatorname{tr}
\left(
\boldsymbol{H}_{R}\boldsymbol{J}^{-1}
\right),
\label{eq:qda-risk-coefficients}
\end{equation}
with
\begin{equation}
\Delta_R
=
\mathcal{E}_{\mathrm{CC}}
-
\mathcal{E}_{\mathrm{PC}},
\qquad
\operatorname{ARE}_{R}
=
\frac{
\mathcal{E}_{\mathrm{CC}}
}{
\mathcal{E}_{\mathrm{PC}}
}.
\label{eq:qda-Delta-ARE}
\end{equation}

For \(p=2\), each \(\boldsymbol{H}_{kl}\) is a one-dimensional integral along
the active portion of the corresponding quadratic boundary. These integrals
are evaluated numerically after identifying the active contour segments.
Population expectations entering \(\boldsymbol{D}\) and
\(\boldsymbol{I}_{M}^{\mathrm{eff}}\) are evaluated independently using
deterministic quadrature or high-accuracy Monte Carlo integration. Numerical
implementation and validation details are provided in Supplementary
Section~\ref{supp:finite-computation}.
\section{Numerical investigation}
\label{sec:numerical}

We use the three-class QDA model to examine the implications of the
information--risk theory developed above. The numerical analysis has three
main objectives: to verify the regular multiclass geometry required by
Theorem~\ref{thm:multiclass-risk} in a nontrivial unequal-covariance setting;
to examine the transition between unfavourable and favourable informative
missingness as the amount and location of missing labels vary; and to assess
the sensitivity of this transition to class separation, prior imbalance,
covariance heterogeneity, and the choice of uncertainty functional.

All quantities in this section are population quantities. In particular,
\(\boldsymbol A\), \(\boldsymbol D\),
\(\boldsymbol I_M^{\mathrm{eff}}\), \(\boldsymbol J\), and
\(\boldsymbol H_R\) are evaluated at the generating parameter rather than
estimated from finite samples.

\subsection{Reference configuration and baseline comparison}
\label{subsec:reference-configuration}

The reference model has class probabilities
\[
(\pi_1,\pi_2,\pi_3)=(0.35,0.35,0.30),
\]
means
\begin{equation}
\boldsymbol{\mu}_1=
\begin{pmatrix}
-1.5\\
0
\end{pmatrix},
\qquad
\boldsymbol{\mu}_2=
\begin{pmatrix}
1.5\\
0
\end{pmatrix},
\qquad
\boldsymbol{\mu}_3=
\begin{pmatrix}
0\\
2
\end{pmatrix},
\label{eq:reference-means}
\end{equation}
and covariance matrices
\begin{equation}
\boldsymbol{\Sigma}_1=
\begin{pmatrix}
1 & 0.30\\
0.30 & 0.70
\end{pmatrix},
\qquad
\boldsymbol{\Sigma}_2=
\begin{pmatrix}
0.80 & -0.20\\
-0.20 & 1.20
\end{pmatrix},
\qquad
\boldsymbol{\Sigma}_3=
\begin{pmatrix}
1.30 & 0.40\\
0.40 & 0.80
\end{pmatrix}.
\label{eq:reference-covariances}
\end{equation}
All three covariance matrices are positive definite, and each class has a
nonempty Bayes region.

The corresponding Bayes geometry contains a genuine three-class junction.
Solving \(d_{12}(\boldsymbol y)=d_{13}(\boldsymbol y)=0\) gives
\begin{equation}
\boldsymbol y^\dagger
\approx
(0.149,\;0.824)^\top.
\label{eq:reference-junction}
\end{equation}
At this point,
\(r_1(\boldsymbol y^\dagger)
=r_2(\boldsymbol y^\dagger)
=r_3(\boldsymbol y^\dagger)\).
The associated spatial gradients are
\[
\nabla_{\boldsymbol y}d_{12}(\boldsymbol y^\dagger)
\approx(-3.070,-0.117)^\top,
\qquad
\nabla_{\boldsymbol y}d_{13}(\boldsymbol y^\dagger)
\approx(-0.817,-2.345)^\top.
\]
Their absolute determinant is approximately \(7.10\), confirming that the
junction is transversal rather than tangential and hence satisfies
\eqref{eq:qda-transversality}. The reference configuration therefore provides
a genuinely multiclass setting with unequal covariance matrices, curved active
Bayes boundaries, and a regular three-class junction.

The classification-risk curvature matrix was obtained by numerical integration
over the three active Bayes faces. Under complete classification, the leading
excess-risk coefficient is
\begin{equation}
\mathcal E_{\mathrm{CC}}
=
\operatorname{tr}
\left(
\boldsymbol H_R\boldsymbol A^{-1}
\right)
=
1.5581,
\label{eq:reference-ECC}
\end{equation}
so that
\(E\{R(\widehat{\boldsymbol\theta}_{\mathrm{CC}})\}-R^\ast
\approx0.7791/n\)
to first order.

We next consider entropy-dependent missingness with marginal missing-label
proportion \(\gamma=0.30\) and slope \(t_H=4\log3\). Calibrating the intercept
so that \(E\{q_H(\boldsymbol Y)\}=0.30\) gives \(\xi_0=-2.5583\). The
corresponding population coefficients are
\[
\mathcal E_{\mathrm{CC}}=1.5581,
\qquad
\mathcal E_{\mathrm{MCAR}}=2.0541,
\qquad
\mathcal E_{\mathrm{PC}}=1.3645,
\]
which yield
\[
\operatorname{ARE}_{R,\mathrm{PC:CC}}=1.1419,
\qquad
\operatorname{ARE}_{R,\mathrm{MCAR:CC}}=0.7586,
\]
and
\[
\Delta_R
=
\mathcal E_{\mathrm{CC}}
-
\mathcal E_{\mathrm{PC}}
=
0.1936.
\]
Thus the same \(30\%\) marginal rate of unavailable labels is detrimental
under MCAR but favourable when missingness depends sufficiently strongly on
classification uncertainty and the mechanism is incorporated into the
likelihood. Numerical integration and reparameterization checks are reported
in Supplementary Section~\ref{supp:finite-computation}.

\subsection{Phase transition under uncertainty-dependent missingness}
\label{subsec:phase-diagram}

To examine how this favourable regime emerges, we vary the strength of
dependence on normalized entropy while holding the marginal missing-label
proportion fixed. For each pair \((t_H,\gamma)\), the intercept in
\[
q_H(\boldsymbol y)
=
\operatorname{expit}
\left\{
\xi_0+t_HU_H(\boldsymbol y)
\right\}
\]
is recalibrated so that \(E\{q_H(\boldsymbol Y)\}=\gamma\).

At \(t_H=0\), the mechanism is MCAR and, consistently with
Corollary~\ref{cor:MCAR-information},
\(\operatorname{ARE}_R<1\) for every positive missing-label rate considered.
As \(t_H\) increases, missing labels become increasingly concentrated in
regions of high posterior uncertainty and the relative efficiency eventually
crosses one. We define the critical slope by
\begin{equation}
t_H^\star(\gamma)
=
\inf
\left\{
t_H:
\Delta_R(\gamma,t_H)>0
\right\}.
\label{eq:critical-slope}
\end{equation}

The resulting values are reported in Table~\ref{tab:critical-entropy}.

\begin{table}[t]
\centering
\caption{Critical normalized entropy slope \(t_H^\star(\gamma)\) for the
reference QDA configuration. For \(t_H<t_H^\star\), complete classification
has the smaller leading excess risk; for \(t_H>t_H^\star\), informative
partial classification is favourable.}
\label{tab:critical-entropy}
\begin{tabular}{cc}
\toprule
Missing-label proportion \(\gamma\) & Critical slope \(t_H^\star\)\\
\midrule
0.10 & 2.90\\
0.20 & 3.17\\
0.30 & 3.52\\
0.40 & 4.03\\
0.50 & 4.87\\
0.60 & 6.63\\
\bottomrule
\end{tabular}
\end{table}

The critical slope increases monotonically over the range considered and rises
more rapidly at larger missing-label proportions. Stronger dependence between
missingness and classification uncertainty is therefore required to offset the
greater loss of class-label information as \(\gamma\) increases.

Figure~\ref{fig:phase-diagram} displays the corresponding phase boundary. The
numerically determined critical slopes satisfy
\(\Delta_R=0\), equivalently \(\operatorname{ARE}_R=1\), and separate the
unfavourable and favourable regimes over the range considered.

\begin{figure}[H]
\centering
\includegraphics[width=0.72\textwidth]{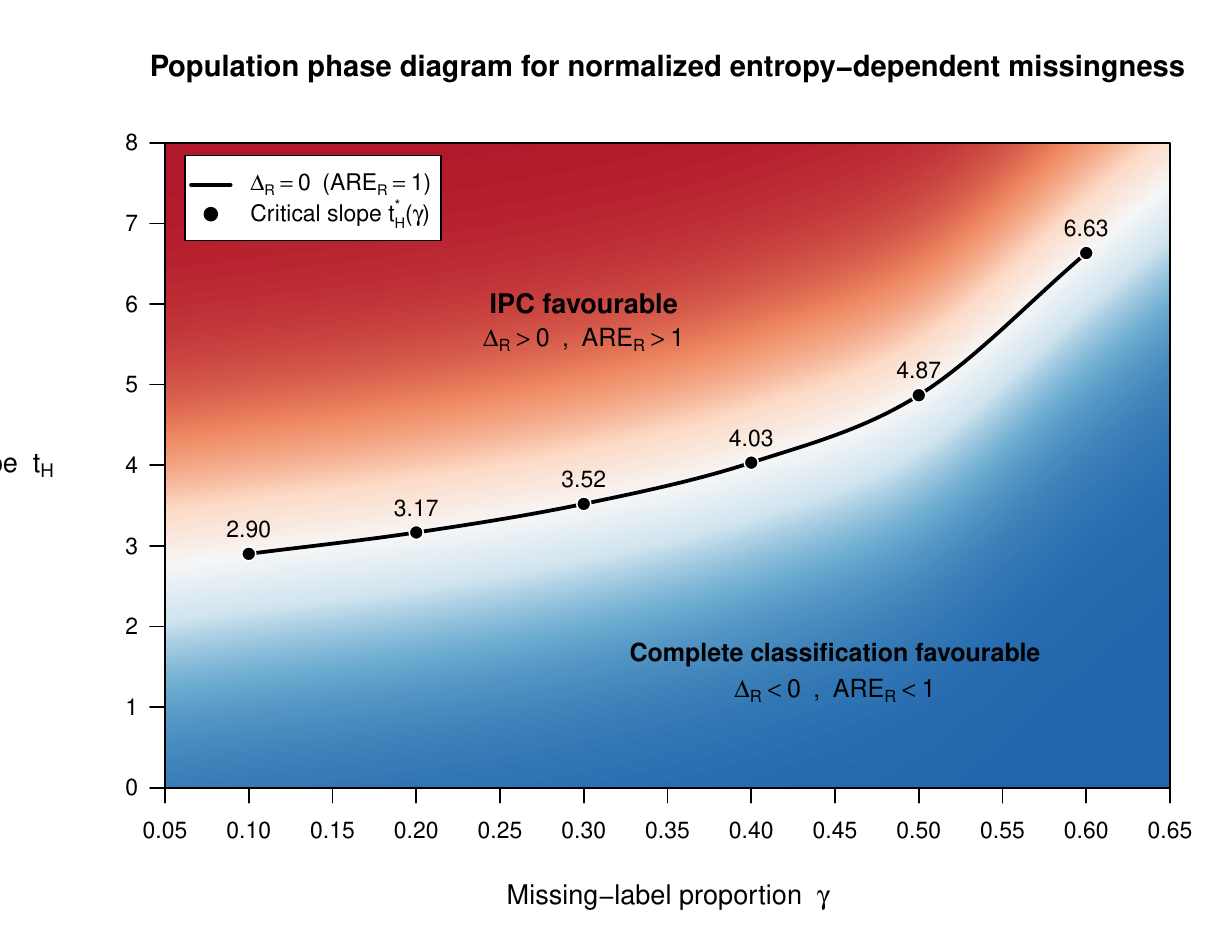}
\caption{\footnotesize Population phase boundary for the reference QDA model
under normalized entropy-dependent missingness. Points denote the numerically
determined critical slopes satisfying \(\Delta_R=0\), equivalently
\(\operatorname{ARE}_R=1\); the connecting curve displays the evolution of the
boundary over the range considered.}
\label{fig:phase-diagram}
\end{figure}

This phase structure is consistent with
Proposition~\ref{prop:local-MCAR}. Near MCAR, changing the uncertainty slope
first redistributes the loss of class-label information, whereas the efficient
information supplied by the missingness indicators enters only at second
order. The dependence must therefore become sufficiently strong before the
latter contribution can offset the former.

\subsection{Sensitivity to model geometry and uncertainty specification}
\label{subsec:sensitivity}

We next examine how the favourable regime changes when the geometry of the
classification problem is altered. Three perturbations are considered:
class separation, prior imbalance, and covariance heterogeneity. Throughout
these comparisons, unless otherwise stated, the missing-label proportion is
\(\gamma=0.30\) and the normalized entropy slope is \(t_H=4\log3\).

\emph{Class overlap.}
We first scale the reference means according to
\begin{equation}
\boldsymbol{\mu}_k(s)
=
s\boldsymbol{\mu}_k^{(0)},
\qquad s>0,
\label{eq:separation-path}
\end{equation}
while holding the covariance matrices and class probabilities fixed. Values
\(s<1\) increase overlap, whereas \(s>1\) increase separation.

\begin{table}[t]
\centering
\caption{Population classification efficiency along the class-separation path.}
\label{tab:separation-ARE}
\begin{tabular}{cccc}
\toprule
\(s\) &
\(\mathcal E_{\mathrm{CC}}\) &
\(\mathcal E_{\mathrm{PC}}\) &
\(\operatorname{ARE}_R\)\\
\midrule
0.40 & 4.121 & 4.313 & 0.955\\
0.50 & 3.083 & 3.039 & 1.014\\
0.60 & 2.523 & 2.379 & 1.061\\
0.75 & 2.042 & 1.835 & 1.113\\
0.90 & 1.722 & 1.511 & 1.140\\
1.00 & 1.558 & 1.365 & 1.142\\
1.15 & 1.352 & 1.196 & 1.130\\
1.30 & 1.162 & 1.046 & 1.111\\
1.50 & 0.926 & 0.847 & 1.093\\
1.70 & 0.698 & 0.652 & 1.071\\
2.00 & 0.406 & 0.386 & 1.051\\
\bottomrule
\end{tabular}
\end{table}

Under severe overlap (\(s=0.40\)), the missingness pattern does not compensate
for the information carried by the unavailable labels. The comparison becomes
favourable between \(s=0.40\) and \(s=0.50\), reaches its largest relative
advantage near the reference configuration, and then gradually weakens as the
classes become more separated. This nonmonotone behaviour reflects the
competing information sources. Under strong overlap, observing the true class
membership is highly informative, so removing labels is costly. At moderate
separation, posterior uncertainty remains concentrated around active decision
boundaries, allowing the missingness pattern to provide information in
classification-relevant directions. Under strong separation, both experiments
have increasingly small excess-risk coefficients, and the relative advantage
diminishes.

\emph{Prior imbalance.}
We next vary the class-3 prior according to
\begin{equation}
\pi_1(\omega)
=
\pi_2(\omega)
=
\frac{1-\omega}{2},
\qquad
\pi_3(\omega)=\omega,
\label{eq:prior-path}
\end{equation}
while retaining the reference means, covariance matrices, missing-label rate,
and uncertainty slope.

\begin{table}[t]
\centering
\caption{Population classification efficiency as the prior probability of
class 3 varies.}
\label{tab:prior-ARE}
\begin{tabular}{cccc}
\toprule
\(\pi_3\) &
\(\mathcal E_{\mathrm{CC}}\) &
\(\mathcal E_{\mathrm{PC}}\) &
\(\operatorname{ARE}_R\)\\
\midrule
0.03 & 1.634 & 1.862 & 0.877\\
0.05 & 1.633 & 1.742 & 0.937\\
0.08 & 1.620 & 1.602 & 1.011\\
0.10 & 1.615 & 1.546 & 1.044\\
0.20 & 1.578 & 1.407 & 1.121\\
0.30 & 1.558 & 1.365 & 1.142\\
0.40 & 1.542 & 1.345 & 1.147\\
0.50 & 1.529 & 1.338 & 1.143\\
0.60 & 1.513 & 1.338 & 1.131\\
0.70 & 1.487 & 1.343 & 1.107\\
0.80 & 1.435 & 1.354 & 1.060\\
0.90 & 1.322 & 1.367 & 0.968\\
\bottomrule
\end{tabular}
\end{table}

Informative partial classification is favourable over a broad range of class
probabilities but ceases to be favourable when class 3 becomes sufficiently
rare or dominant. The largest relative advantage occurs for moderately balanced
priors. The face-specific calculations clarify the reversal under strong
imbalance. At \(\pi_3=0.05\), the relative efficiencies associated with faces
\(12\), \(13\), and \(23\) are approximately \(1.040\), \(0.968\), and
\(0.856\), respectively. Thus the gain associated with the \(12\) boundary is
insufficient to compensate for losses on the two boundaries involving the rare
third class. At the reference prior \(\pi_3=0.30\), the corresponding
efficiencies are \(1.134\), \(1.139\), and \(1.151\), so all three active faces
contribute favourably. Additional face-specific results are reported in Supplementary
Section~\ref{supp:population-robustness}.

\noindent\textit{Covariance heterogeneity.}
To examine the role of nonlinear boundary geometry, define
\[
\overline{\boldsymbol{\Sigma}}
=
\sum_{k=1}^{3}
\pi_k\boldsymbol{\Sigma}_k^{(0)}
\]
and
\[
\boldsymbol{\Sigma}_k(\rho)
=
(1-\rho)\overline{\boldsymbol{\Sigma}}
+
\rho\boldsymbol{\Sigma}_k^{(0)},
\qquad
0\leq \rho\leq 1,
\]
where
\(\boldsymbol{\Sigma}_k^{(0)}\) denotes the reference covariance matrix.
Thus \(\rho=0\) gives a common-covariance model with linear Bayes boundaries,
whereas \(\rho=1\) recovers the reference QDA model. Detailed results for this
path are reported in Supplementary Table~\ref{supp:tab:covariance-heterogeneity}. The classification advantage is
stable along the path: the relative efficiency decreases only modestly from
\(1.160\) under common covariance matrices to \(1.142\) in the reference QDA
model. Thus, within this controlled family, the favourable regime persists as
the Bayes geometry changes continuously from linear to curved. The
ignored-mechanism coefficient remains substantially larger than both the
complete-classification and correctly modelled partial-classification
coefficients throughout the path.

Taken together, the three geometric perturbations show that the classification
advantage is sensitive to the underlying Bayes geometry but not uniformly so.
Within the ranges examined here, class overlap and prior imbalance can change
the sign of \(\Delta_R\), whereas covariance heterogeneity primarily changes
its magnitude. Additional face-specific and geometry-dependent phase-boundary
calculations are reported in Supplementary Section~\ref{supp:population-robustness}.

Finally, replacing normalized Shannon entropy with normalized Gini uncertainty
produces the same qualitative phase structure. For both uncertainty measures,
the critical slope increases with the marginal missing-label proportion. In
the reference configuration, normalized Gini uncertainty reaches the favourable
region at a smaller slope than normalized Shannon entropy for each value of
\(\gamma\) examined. This comparison is a robustness check rather than an
ordering of the two uncertainty measures; it shows that the phase-transition
phenomenon is not specific to entropy-based missingness. Detailed values are
reported in Supplementary Table~\ref{supp:tab:entropy-gini-critical}.

\subsection{Directional information anatomy}
\label{subsec:spectral-anatomy}

We conclude the population investigation by using the spectral representation
in Theorem~\ref{thm:spectral} to examine why the reference informative
mechanism is favourable. Recall that
\[
\Delta_R
=
\sum_j
w_j
\frac{\lambda_j-1}{\lambda_j}.
\]
The generalized eigenvalues are not uniformly larger than one. Hence the
partially classified experiment does not dominate complete classification in
the Loewner order. Instead, several directions with \(\lambda_j>1\) also carry
substantial classification weights \(w_j\), and their combined contribution
outweighs the losses in directions with \(\lambda_j<1\).

The four largest positive directional contributions are approximately
\[
0.0877,\qquad
0.0571,\qquad
0.0536,\qquad
0.0352,
\]
whereas the two largest negative contributions are approximately
\[
-0.0313
\qquad\text{and}\qquad
-0.0303.
\]
Summing over all generalized directions gives
\[
\sum_j
w_j
\frac{\lambda_j-1}{\lambda_j}
=
0.193601,
\]
which agrees to numerical precision with the direct calculation
\[
\Delta_R
=
\mathcal E_{\mathrm{CC}}
-
\mathcal E_{\mathrm{PC}}
=
0.193601.
\]

This decomposition provides the numerical counterpart of
Theorem~\ref{thm:spectral}. The favourable result does not arise from a
uniform increase in Fisher information: some parameter directions lose
information. The improvement occurs because the gains are concentrated
sufficiently strongly in directions that receive substantial
classification-risk weight.

\section{Finite-sample validation}
\label{sec:finite-sample}

We next examine whether the population information--risk comparison is
reflected in finite samples. The purpose is not to compare alternative
classification procedures broadly, but to assess three implications of the
theory: whether the empirical covariance of the estimators approaches its
information-based limit, whether the excess classification risk exhibits the
predicted \(n^{-1}\) behaviour, and whether the population criterion correctly
predicts the relative classification performance of complete and partially
classified samples. In this section, IPC denotes informative partial
classification.

\subsection{Simulation design and estimation}
\label{subsec:finite-design}

Data are generated from the reference three-class QDA configuration
of Section~\ref{subsec:reference-configuration}. For each
observation,
\[
Z_i\sim\operatorname{Multinomial}(1;\pi_1,\pi_2,\pi_3),
\qquad
\boldsymbol Y_i\mid Z_i=k
\sim
N_2(\boldsymbol\mu_k,\boldsymbol\Sigma_k).
\]
We consider \(n\in\{250,500,1000\}\), with \(B=500\) Monte Carlo replications
at each sample size.

For every generated complete sample, three observation experiments are
constructed. Under complete classification (CC), all class labels are
observed. Under MCAR, each label is independently unavailable with probability
\(\gamma=0.30\). The third experiment uses informative partial classification
with
\begin{equation}
q(\boldsymbol y)
=
\operatorname{expit}
\left\{
\xi_0+t_HU_H(\boldsymbol y)
\right\},
\label{eq:finite-IPC}
\end{equation}
where \(\gamma=0.30\), \(t_H=4\log3\), and \(U_H\) is normalized posterior
entropy. The intercept is calibrated under the generating model so that
\(E\{q(\boldsymbol Y)\}=0.30\), giving \(\xi_0=-2.5583\). Thus MCAR and IPC
have the same marginal proportion of unavailable labels but differ in their
location: under IPC, missing labels are concentrated more strongly in regions
of high posterior uncertainty.

The population coefficients for this configuration are
\[
\mathcal K_{\mathrm{CC}}=1.5581,
\qquad
\mathcal K_{\mathrm{MCAR}}=2.0541,
\qquad
\mathcal K_{\mathrm{IPC}}=1.3645,
\]
corresponding to
\(\operatorname{ARE}_{R,\mathrm{IPC:CC}}=1.1419\) and
\(\operatorname{ARE}_{R,\mathrm{MCAR:CC}}=0.7586\).
The theory therefore predicts a classification advantage for IPC and an
efficiency loss under MCAR.

The same generated complete sample is used to construct all three experiments
within each replication, reducing Monte Carlo variation in their comparison.
Under IPC, the missingness parameters
\(\boldsymbol\xi=(\xi_0,t_H)^\top\) are estimated jointly with
\(\boldsymbol\theta\); the population calibration
\(E\{q(\boldsymbol Y)\}=\gamma\) is used only to define the generating
experiment and is not imposed during estimation. Covariance matrices are
parameterized through lower log-Cholesky factors to ensure positive
definiteness. Numerical optimization uses multiple starting values and
standard convergence and admissibility checks. Further computational details
are provided in Supplementary Section~\ref{supp:finite-computation}.

\subsection{Evaluation criteria}
\label{subsec:finite-criteria}

For an estimate \(\widehat{\boldsymbol\theta}\), let
\(\widehat C(\boldsymbol y)
=\arg\max_k r_k(\boldsymbol y;\widehat{\boldsymbol\theta})\)
denote the corresponding plug-in Bayes classifier. We evaluate its excess
population risk using
\begin{equation}
\mathcal X(\widehat{\boldsymbol\theta})
=
E_{\boldsymbol\theta_0}
\left[
\max_{1\leq k\leq3}\tau_{0k}(\boldsymbol Y)
-
\tau_{0,\widehat C(\boldsymbol Y)}(\boldsymbol Y)
\right],
\label{eq:finite-excess-posterior}
\end{equation}
where \(\tau_{0k}\) is the posterior class probability under the generating
parameter. This criterion evaluates the fitted decision rule under the true
population distribution rather than through empirical error on a finite test
sample.

The expectation in \eqref{eq:finite-excess-posterior} is evaluated using a
common stratified Monte Carlo population sample of size \(600{,}000\), with
\(200{,}000\) observations generated conditionally from each class and
class-specific averages weighted by the true mixing proportions. The same
population sample is used for every fitted classifier.

For method \(m\), define the scaled excess-risk estimate
\begin{equation}
\widehat{\mathcal K}_{m,n}
=
2n
\frac{1}{B_m}
\sum_{b\in\mathcal C_m}
\mathcal X
\left(
\widehat{\boldsymbol\theta}_m^{(b)}
\right),
\label{eq:scaled-excess-risk}
\end{equation}
where \(\mathcal C_m\) is the set of usable replications and
\(B_m=|\mathcal C_m|\). The theory shows that
\begin{equation}
\widehat{\mathcal K}_{m,n}
\longrightarrow
\mathcal K_m
=
\operatorname{tr}
\left(
\boldsymbol H_R\boldsymbol V_m
\right),
\label{eq:scaled-risk-limit}
\end{equation}
where
\(\boldsymbol V_{\mathrm{CC}}=\boldsymbol A^{-1}\),
\(\boldsymbol V_{\mathrm{IPC}}=\boldsymbol J^{-1}\), and
\(\boldsymbol V_{\mathrm{MCAR}}\) is obtained from the corresponding
partially classified information matrix.

As an independent covariance check, let
\[
\widehat{\boldsymbol V}_{m,n}
=
n\,\widehat{\operatorname{Cov}}
(\widehat{\boldsymbol\theta}_m).
\]
We assess the classification-relevant discrepancy through
\begin{equation}
D_{R,m}(n)
=
\frac{
\left|
\operatorname{tr}
\left(
\boldsymbol H_R\widehat{\boldsymbol V}_{m,n}
\right)
-
\operatorname{tr}
\left(
\boldsymbol H_R\boldsymbol V_m
\right)
\right|
}{
\operatorname{tr}
\left(
\boldsymbol H_R\boldsymbol V_m
\right)
}.
\label{eq:risk-weighted-discrepancy}
\end{equation}

\subsection{Finite-sample results}
\label{subsec:finite-results}

Table~\ref{tab:finite-risk} reports the scaled excess classification risks.
The theoretical ordering is
\begin{equation}
\mathcal K_{\mathrm{IPC}}
<
\mathcal K_{\mathrm{CC}}
<
\mathcal K_{\mathrm{MCAR}},
\label{eq:predicted-risk-order}
\end{equation}
and this ordering is reproduced at every sample size considered.

\begin{table}[t]
\centering
\caption{Finite-sample scaled excess classification risk. Values are Monte Carlo averages over usable fits of
\(2n\{R(\widehat{\boldsymbol{\theta}})-R^\ast\}\), with Monte
Carlo standard errors in parentheses. The final column gives the corresponding
asymptotic coefficient
\(\mathcal K_m=\operatorname{tr}(\boldsymbol H_R\boldsymbol V_m)\).}
\label{tab:finite-risk}
\begin{tabular}{llccc}
\toprule
\(n\) & Method & Usable fits &
\(2n\{R(\widehat{\boldsymbol\theta})-R^\ast\}\) &
\(\mathcal K_m\)\\
\midrule
250
& CC   & 500 & 1.5756 (0.0466) & 1.5581\\
& MCAR & 500 & 2.2367 (0.0697) & 2.0541\\
& IPC  & 463 & 1.4971 (0.0497) & 1.3645\\[2pt]

500
& CC   & 500 & 1.5008 (0.0416) & 1.5581\\
& MCAR & 500 & 2.0217 (0.0583) & 2.0541\\
& IPC  & 485 & 1.4521 (0.0416) & 1.3645\\[2pt]

1000
& CC   & 500 & 1.5247 (0.0424) & 1.5581\\
& MCAR & 499 & 2.0294 (0.0559) & 2.0541\\
& IPC  & 491 & 1.3682 (0.0398) & 1.3645\\
\bottomrule
\end{tabular}
\end{table}

The finite-sample results approach the population predictions as the sample
size increases. At \(n=1000\), the scaled risks are \(1.5247\), \(2.0294\),
and \(1.3682\) for CC, MCAR, and IPC, respectively, compared with theoretical
limits \(1.5581\), \(2.0541\), and \(1.3645\). In particular, the IPC result
differs from its first-order limit by only about \(0.004\).

The corresponding finite-sample IPC-to-CC relative efficiencies are
approximately \(1.052\), \(1.034\), and \(1.114\) for
\(n=250,500,\) and \(1000\), respectively, compared with the population value
\(1.142\). Thus the favourable population comparison is already clearly
visible at the largest sample size.

\begin{figure}[H]
\centering
\includegraphics[width=0.78\textwidth]{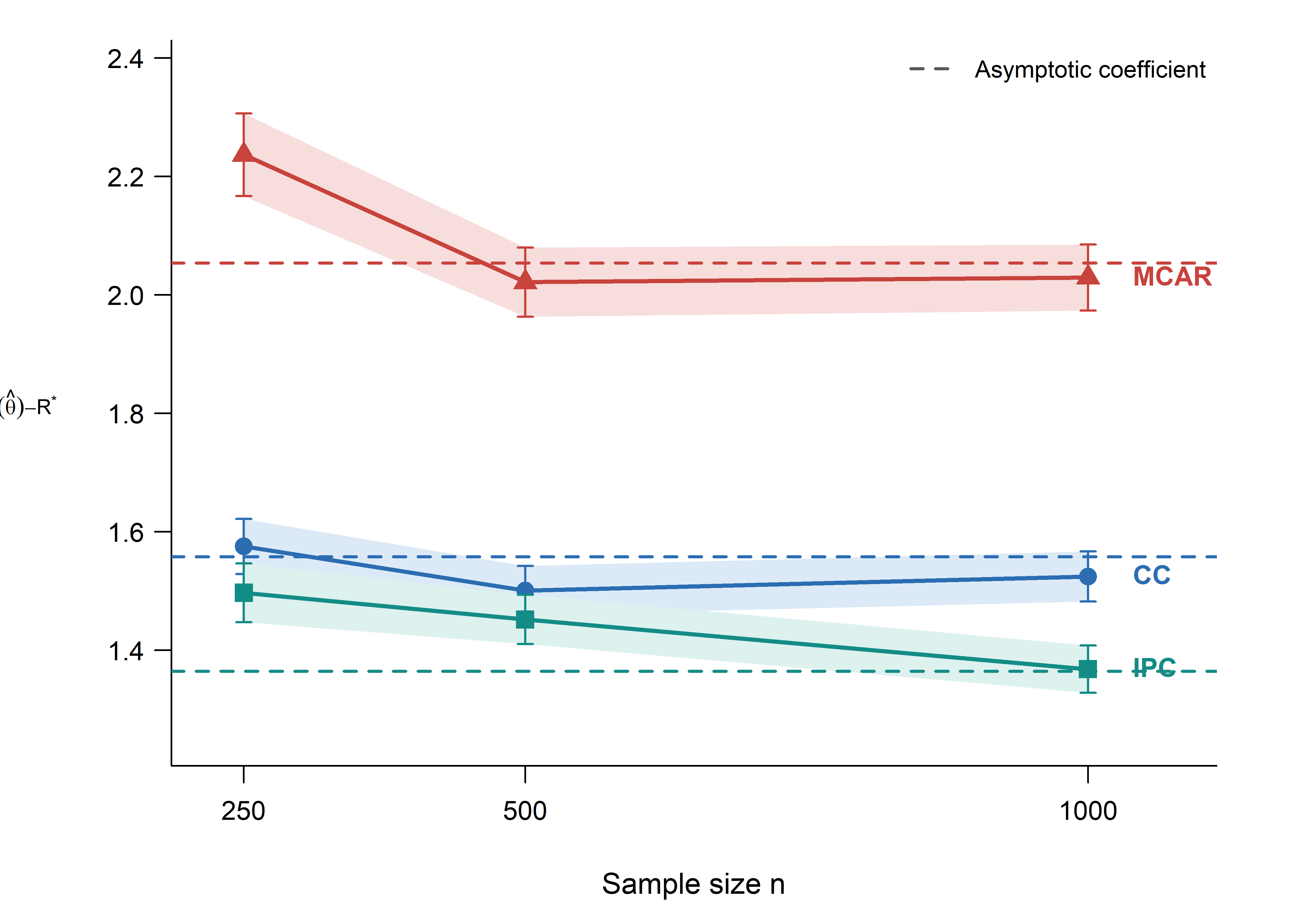}
\caption{\footnotesize Finite-sample scaled excess classification risk for
complete classification (CC), missing completely at random (MCAR), and
informative partial classification (IPC). Points show Monte Carlo averages of
\(2n\{R(\widehat{\boldsymbol\theta})-R^\ast\}\), with error bars representing
one Monte Carlo standard error across usable fits. Dashed horizontal lines show
the corresponding asymptotic coefficients
\(\operatorname{tr}(\boldsymbol H_R\boldsymbol V_m)\).}
\label{fig:finite-risk}
\end{figure}

Numerical convergence also improves with sample size. All CC fits were usable,
while the IPC convergence rates were \(92.6\%\), \(97.0\%\), and \(98.2\%\)
for \(n=250,500,\) and \(1000\), respectively.

The covariance calculation provides an independent validation of the
information approximation; detailed covariance diagnostics are reported in
Supplementary Table~\ref{supp:tab:covariance}. At \(n=1000\), the classification-weighted covariance
discrepancies are approximately \(1.4\%\), \(0.5\%\), and \(0.7\%\) for CC,
MCAR, and IPC, respectively. A further check based directly on the quadratic
risk approximation gives \(1.5341\), \(2.0407\), and \(1.3735\),
respectively. These values are close both to the directly evaluated scaled
risks \(1.5247\), \(2.0294\), and \(1.3682\), and to their theoretical
limits. The agreement among direct population risk, the local quadratic
approximation, and the information-based covariance calculation provides a
numerical validation of the mechanism underlying \eqref{eq:scaled-risk-limit}.

The paired comparisons require a different interpretation. At \(n=1000\),
IPC has smaller excess risk than CC in approximately \(55.6\%\) of paired
usable replications, while the average paired difference favours IPC. This is
not inconsistent with the theory, which concerns expected excess
classification risk rather than the probability that IPC outperforms CC in
each individual sample.

Overall, the finite-sample experiment reproduces the central population
prediction. Removing \(30\%\) of labels under MCAR increases classification
risk, whereas under the informative mechanism the missingness indicators
provide additional information about classifier-relevant parameters. In the
configuration considered here, this contribution is sufficiently aligned with
the active Bayes boundary that IPC has smaller expected excess classification
risk than complete classification, both asymptotically and in the finite
samples examined.

\section{Application to the Vertebral Column data}
\label{sec:realdata}

To examine how informative label missingness and explicit modelling of
its mechanism behave in a real multiclass classification geometry, we
considered the three-class Vertebral Column data set from the UCI
Machine Learning Repository \citep{vertebralUCI}. The data contain 310
observations classified as disk hernia (DH), spondylolisthesis (SL), or
normal (NO), with six continuous biomechanical measurements. Because
the original data are fully labelled, the observed feature vectors and
class labels were retained, while label availability was generated
according to the uncertainty-dependent mechanism considered in the
preceding sections. The analysis is therefore semi-synthetic: the
classification problem is observed, whereas the missing-label process
is imposed.

The QDA implementation in Section~\ref{sec:qda} is two-dimensional.
Variable selection was therefore carried out within each outer training
sample to avoid selecting the most favourable pair using observations
on which predictive performance was subsequently evaluated.
Specifically, for every outer cross-validation fold, all 15 pairs formed
from the six biomechanical variables were compared by an inner
five-fold stratified cross-validation, and the pair with the largest
complete-classification accuracy was retained. The purpose of this step
was to define a common two-dimensional classification representation
before generating the semi-synthetic missing-label indicators. Feature
standardization, variable selection, generation of the missing-label
indicators, and model fitting were all performed using the outer
training sample only. Across the 50 outer folds of the main experiment,
sacral slope together with degree of spondylolisthesis was selected in
72\% of the folds, pelvic radius together with degree of
spondylolisthesis in 18\%, and lumbar lordosis angle together with degree
of spondylolisthesis in the remaining 10\%.

For a feature vector $y$, let
\[
U_H(y;\theta)
=
-\frac{1}{\log 3}
\sum_{k=1}^{3}
\tau_k(y;\theta)\log\{\tau_k(y;\theta)\}
\]
denote normalized classification entropy. Within each outer training
sample, a complete-classification QDA fit with estimate
$\widehat{\theta}_{\mathrm{CC}}$ was first obtained and used only to
construct the semi-synthetic missing-label mechanism. Writing
\[
U_H^{\mathrm{ref}}(y)
=
U_H(y;\widehat{\theta}_{\mathrm{CC}}),
\]
labels were made unavailable according to
\begin{equation}
\operatorname{logit}\{q_{\mathrm{gen}}(y)\}
=
\xi_0+\xi_1 U_H^{\mathrm{ref}}(y),
\label{eq:real-missing}
\end{equation}
where $\xi_0$ was calibrated within the outer training sample so that
the average missing-label probability equalled a prespecified value
$\gamma$. The main setting used $\gamma=0.30$ and
$\xi_1=4\log 3$, matching the reference mechanism used in the
population and finite-sample investigations. The
complete-classification estimate was used only to generate the
missing-label indicators. In the IPC fit, the classification and
missingness parameters were subsequently estimated jointly.

Four procedures were compared. CC used all training labels. MCAR
removed labels independently of the features at the same nominal
missing-label proportion and treated the missingness mechanism as
ignorable. IG used the informatively incomplete training sample but
ignored the missingness mechanism. IPC used exactly the same
informatively incomplete training sample as IG while jointly modelling
the classification model and the uncertainty-dependent missing-label
mechanism. Predictive performance was evaluated on the untouched outer
test folds. The main analysis used ten repetitions of five-fold
stratified cross-validation. All four procedures converged in all 50
outer fits.

Table~\ref{tab:vertebral-main} summarizes the resulting predictive
performance. Complete classification had the smallest mean
misclassification rate, $0.2119$, followed by MCAR, IPC, and IG. IPC
therefore did not improve the $0$--$1$ error over CC in the reference
setting. The comparison between IG and IPC is more directly informative
about the value of modelling the missing-label mechanism because the
two procedures use the same informatively incomplete training samples.
Relative to IG, IPC reduced the mean log loss from $0.5185$ to $0.5015$
and the mean Brier score from $0.2899$ to $0.2822$, whereas the
reduction in mean misclassification error was much smaller, from
$0.2248$ to $0.2235$.

At the repeat level, the mean paired differences
$R_{\mathrm{IG}}-R_{\mathrm{IPC}}$ were $0.0013$ for
misclassification error, $0.0171$ for log loss, and $0.0077$ for Brier
score. Thus, under the reference mechanism, explicit modelling of the
informative missing-label process had a clearer effect on probabilistic
prediction than on the resulting $0$--$1$ decision rule.

\begin{table}[t]
\centering
\caption{Repeated nested cross-validation results for the Vertebral
Column data under the reference missingness mechanism
$\gamma=0.30$ and $\xi_1=4\log 3$. Values are averages over ten
repeated five-fold outer cross-validations.}
\label{tab:vertebral-main}
\begin{tabular}{lccc}
\toprule
Method & Error rate & Log loss & Brier score \\
\midrule
CC   & 0.2119 & 0.4803 & 0.2749 \\
MCAR & 0.2152 & 0.5064 & 0.2813 \\
IG   & 0.2248 & 0.5185 & 0.2899 \\
IPC  & 0.2235 & 0.5015 & 0.2822 \\
\bottomrule
\end{tabular}
\end{table}

The generated missing-label pattern was clearly uncertainty-selective.
Under the reference mechanism, the empirical missing-label proportion
was $0.300$ on average. The mean normalized entropy among observations
whose labels were unavailable was $0.641$, compared with $0.249$ among
labelled observations, corresponding to an average fold-level
difference of approximately $0.392$. Thus the mechanism systematically
concentrated missing labels in regions of substantially greater
classification uncertainty. The average fitted slope of the
missingness model was $4.86$, compared with the generating value
$4\log 3\simeq4.39$.

To examine the role of informativeness more directly, we repeated the
analysis over the prespecified grid
\[
\gamma\in\{0.20,0.30,0.40\},
\qquad
\xi_1\in\{2\log 3,4\log 3,6\log 3\}.
\]
Five repeated outer cross-validations were used for each configuration,
with the same nested variable-selection procedure. Table~
\ref{tab:vertebral-grid} reports the CC, IG, and IPC error rates,
together with the paired repeat-level difference
\[
\Delta_{\mathrm{IG,IPC}}
=
R_{\mathrm{IG}}-R_{\mathrm{IPC}},
\]
so that positive values favour explicit modelling of the informative
missing-label mechanism.

\begin{table}[t]
\centering
\caption{\footnotesize Sensitivity to the missing-label proportion $\gamma$ and
informativeness strength $\xi_1$. Positive values of
$\Delta_{\mathrm{IG,IPC}}$ favour IPC.}
\label{tab:vertebral-grid}

\begin{tabular}{cccccc}
\toprule
$\gamma$ & $\xi_1/\log 3$ & CC & IG & IPC
& $\Delta_{\mathrm{IG,IPC}}$ \\
\midrule
0.20 & 2 & 0.2110 & 0.2129 & 0.2071 & 0.0058 \\
0.30 & 2 & 0.2110 & 0.2226 & 0.2142 & 0.0084 \\
0.40 & 2 & 0.2110 & 0.2290 & 0.2284 & 0.0006 \\
\addlinespace
0.20 & 4 & 0.2110 & 0.2200 & 0.2077 & 0.0123 \\
0.30 & 4 & 0.2110 & 0.2303 & 0.2258 & 0.0045 \\
0.40 & 4 & 0.2110 & 0.2361 & 0.2194 & 0.0168 \\
\addlinespace
0.20 & 6 & 0.2110 & 0.2213 & 0.2019 & 0.0194 \\
0.30 & 6 & 0.2110 & 0.2406 & 0.2323 & 0.0185$^{\dagger}$ \\
0.40 & 6 & 0.2110 & 0.2445 & 0.2252 & 0.0194 \\
\bottomrule
\end{tabular}

\vspace{2pt}

\parbox{0.92\textwidth}{%
\footnotesize $^{\dagger}$ For $\gamma=0.30$ and $\xi_1=6\log 3$, one IPC fold did
not satisfy the convergence criterion. The displayed IG and IPC values are
method-specific averages over complete repetitions, whereas
$\Delta_{\mathrm{IG,IPC}}$ is the paired mean difference over the four
repetitions complete for both methods. Consequently, for this setting,
$\Delta_{\mathrm{IG,IPC}}$ need not equal the difference between the two
displayed marginal means.
}
\end{table}

Across the sensitivity grid, the benefit of modelling the informative
mechanism became more consistent as the dependence of missingness on
classification uncertainty strengthened. At $\xi_1=2\log 3$, the
paired IPC improvement over IG was small for all three missing-label
proportions. At $\xi_1=4\log 3$, IPC again improved on IG throughout the
grid, although the magnitude varied with $\gamma$. Under the strongest
mechanism, $\xi_1=6\log 3$, the paired reductions in error were
approximately $0.0194$, $0.0185$, and $0.0194$ for
$\gamma=0.20$, $0.30$, and $0.40$, respectively. In the corresponding
repeat-level comparisons, IPC had lower error than IG in all paired
repetitions for $\gamma=0.20$ and $\gamma=0.30$, and in four of the
five repetitions for $\gamma=0.40$.

Numerical convergence was stable across the sensitivity analysis. All
methods converged in every fit for eight of the nine configurations.
The only exception was IPC at $\gamma=0.30$ and
$\xi_1=6\log 3$, for which 24 of the 25 outer fits were usable.

This progression is displayed in Figure~\ref{fig:vertebral-grid}.
The vertical bars summarize repeat-level variation and are intended as
descriptive stability intervals rather than independent-sample
confidence intervals.

\begin{figure}[H]
\centering
\includegraphics[width=0.72\textwidth]{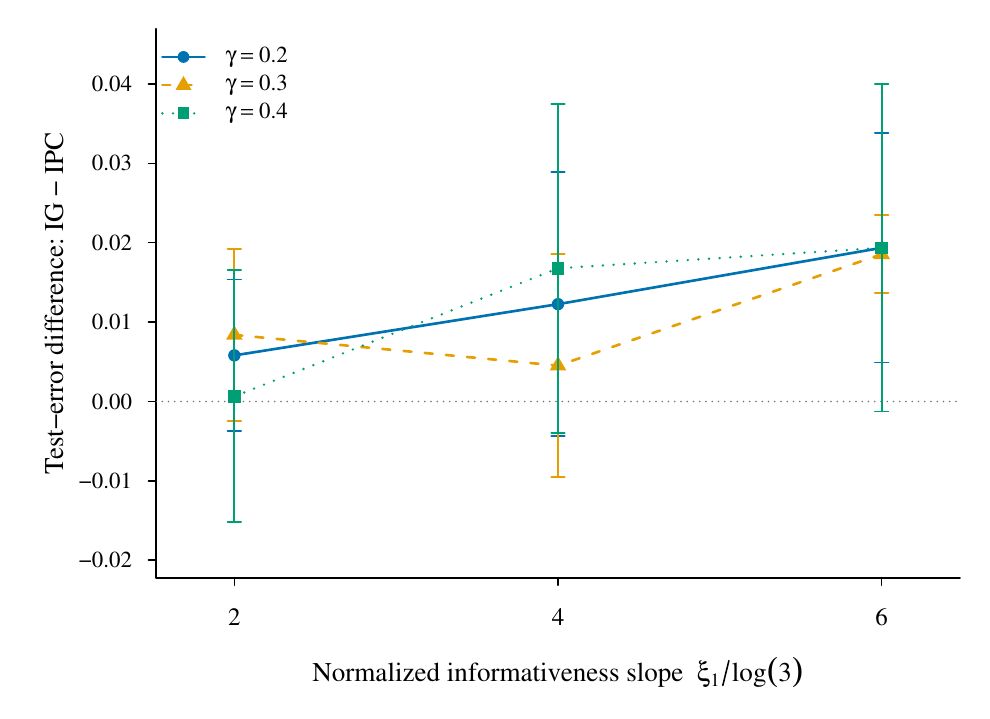}
\caption{\footnotesize Paired test-error difference between ignoring (IG) and modelling
(IPC) the informative missing-label mechanism in the Vertebral Column
application, plotted against the normalized informativeness slope
$\xi_1/\log 3$. Positive values favour IPC. Points are mean paired
differences across complete cross-validation repetitions; vertical bars
are descriptive 95\% $t$-intervals across the repeat-level paired
differences and are not interpreted as independent-sample confidence
intervals.}
\label{fig:vertebral-grid}
\end{figure}

The comparison with complete classification remained
regime-dependent. IPC did not uniformly outperform CC across the
configurations considered. For example, at
$\gamma=0.20$ and $\xi_1=6\log 3$, the mean IPC error was $0.2019$,
compared with $0.2110$ for CC, whereas CC had the smaller mean error at
$\gamma=0.30$ and $\gamma=0.40$ under the same uncertainty slope.
These finite-sample comparisons are not interpreted as evidence of
general dominance of partial over complete classification.

Rather, the semi-synthetic application is consistent with the
regime-dependent behavior emphasized by the theoretical development.
Weak uncertainty dependence provides only limited benefit from
modelling the missing-label process, whereas stronger dependence can
make the observed pattern of label availability increasingly useful.
At the same time, informative missingness alone does not guarantee that
partial classification will outperform complete classification. The population and Monte Carlo investigations in
Sections~\ref{sec:numerical} and~\ref{sec:finite-sample}
provide the direct evaluation of the classification-weighted information
criterion and its asymptotic risk implications; the present application
provides a complementary finite-sample illustration on an observed multiclass
classification geometry.
\section{Discussion}
\label{sec:discussion}

This paper has examined informative label missingness from the perspective of
classification risk rather than parameter estimation alone. The central issue
is not simply whether a partially classified experiment contains more or less
Fisher information than complete classification, but whether the information
that is lost and gained lies in parameter directions that affect the Bayes
decision boundary.

The efficient-information decomposition separates two effects of informative
label missingness: the loss of conditional class-label information when labels
are unavailable and the efficient information carried by the missing-label
indicators themselves. This distinction also makes clear that the cost of
missing labels depends on where they occur in feature space, rather than only
on their marginal frequency. Under missing completely at random, the
missingness indicators carry no information about the classification
parameters, and only the loss of class-label information remains.

The excess-risk analysis provides the corresponding decision-theoretic
interpretation. To second order, classification risk is governed by
perturbations of the active pairwise Bayes faces. Pairwise equality surfaces
that are not active decision boundaries do not contribute, and, under the
regularity and transversality conditions considered here, generic triple and
higher-order junctions do not contribute to the leading quadratic term. The
resulting matrix \(\boldsymbol H_R\) therefore defines a
classification-specific weighting of parameter uncertainty: estimation error
matters according to the extent to which it perturbs the active Bayes
geometry.

Combining these two results shows why a comparison based solely on the Loewner
ordering of Fisher information matrices can be too strong for classification.
The partially classified experiment may lose information in some directions
and gain it in others while nevertheless achieving smaller leading excess
classification risk. What matters is whether the gains occur in directions
that receive sufficiently large classification weights. The generalized
eigenvalue representation makes this alignment between information and
decision-boundary geometry explicit. Thus favourable informative missingness
is a directional phenomenon rather than a consequence of global information
dominance.

The local analysis around missing completely at random further qualifies this
conclusion. With the marginal missing-label proportion held fixed, a small
departure from MCAR redistributes the loss of class-label information at first
order, whereas the efficient information supplied by the missingness pattern
appears only at second order. Informativeness is therefore not automatically
beneficial: weak dependence between missingness and classification uncertainty
may initially worsen classification, and a favourable regime, when it exists,
may arise only after that dependence becomes sufficiently strong.

The three-class QDA calculations illustrate these mechanisms in a setting with
curved decision boundaries and a genuine multiclass junction. The numerical
results also show that favourable missingness is not universal. Its magnitude
is reduced under severe class overlap, can reverse under sufficiently strong
class imbalance, and is comparatively stable along the covariance-heterogeneity
path examined here. Entropy- and Gini-based missingness mechanisms yield the
same qualitative phase structure, while the face-specific and spectral
decompositions show that an overall classification advantage can coexist with
losses on particular decision faces or information directions.

The finite-sample experiment provides a separate check of the asymptotic predictions.
Under the reference configuration, the ordering implied by the population risk coefficients
is reproduced over the sample sizes considered. Moreover, the directly evaluated excess
risks, their quadratic approximations, and the risk-weighted empirical covariance calculations
approach the corresponding theoretical quantities as the sample size increases. This
agreement supports the information--geometry explanation of the observed classification
advantage, rather than only the sign of a particular numerical comparison.

The real-data application provides a complementary finite-sample illustration.
The observed feature vectors and class labels are retained from the Vertebral Column data,
while label availability is generated according to the uncertainty-dependent mechanism.
Informative label removal was concentrated among observations with substantially greater
classification uncertainty, and explicitly modelling the missingness mechanism generally
improved prediction relative to treating the same informative pattern as ignorable.
The improvement in misclassification error was modest under weak or moderate uncertainty
dependence but became appreciable under the strongest mechanism considered. At the same
time, informative partial classification did not uniformly outperform complete
classification. This behavior is consistent with the theoretical distinction developed
above: informativeness alone is not sufficient for favourable partial classification,
because the information carried by the missingness indicators must also compensate for
the lost class-label information in classifier-relevant directions.

There is no paradox in the possibility that informative partial classification outperforms
complete classification under the experiments compared here. The
two experiments do not differ only in the number of observed labels. Under
informative partial classification, the missing-label indicator is itself an
observed random variable whose distribution may depend on the classification
parameters. It can therefore contribute information that is absent from the
complete-classification experiment, in which the corresponding missingness
process is not observed.

The scope of the results is nevertheless limited by the assumptions underlying
the analysis. The theory is parametric and relies on regular likelihood
asymptotics and sufficiently smooth, transversal Bayes-boundary geometry.
Nonregular boundaries, singular models, and high-dimensional regimes in which
the parameter dimension grows with the sample size require different
arguments. The conditional-independence assumption
\(M\perp Z\mid\boldsymbol Y\) also excludes mechanisms in which label
availability depends directly on the latent class after conditioning on the
features; such settings raise additional identification issues. Furthermore,
any information gain from the missingness indicators depends on an adequate
model for \(q(\boldsymbol y;\boldsymbol\theta,\boldsymbol\xi)\).
Misspecification may remove that gain and can introduce bias, making
sensitivity analysis and robust or semiparametric formulations important
extensions.

The present analysis treats the label-missingness mechanism as given. An
important extension is to reverse the problem and choose a labelling or
abstention mechanism to minimize classification risk subject to a labelling
budget. The classification-weighted information criterion developed here
provides a natural starting point for such a design problem. Further extensions
include high-dimensional classifiers, covariate-dependent class probabilities,
alternative classification losses, semiparametric missingness mechanisms, and
settings involving multiple forms of incomplete information.

\section{Conclusion}
\label{sec:conclusion}

The main conclusion is that the statistical value of a partially labelled
sample cannot be determined by the proportion of observed labels or by an
unweighted comparison of Fisher information matrices alone. When label
missingness is informative, classification performance depends on how the
resulting information gains and losses align with the parameter directions
that perturb the active Bayes decision boundary.

The theory developed here formalizes this principle by linking efficient
information under informative partial classification to the local geometry of
multiclass excess risk. The numerical investigations and the application to
the Vertebral Column data further illustrate that the resulting classification
advantage is regime-dependent rather than universal. It explains how partial
classification can be unfavourable near MCAR yet become favourable as the
missingness mechanism becomes sufficiently informative, and why such an
improvement can occur without global information dominance over complete
classification. The
relevant comparison is therefore not simply how much information is available,
but where that information lies relative to the geometry of the classification
problem.





\clearpage

\setcounter{section}{0}
\setcounter{subsection}{0}
\setcounter{equation}{0}
\setcounter{table}{0}
\setcounter{figure}{0}

\renewcommand{\thesection}{S\arabic{section}}
\renewcommand{\thesubsection}{\thesection.\arabic{subsection}}
\renewcommand{\theequation}{\thesection.\arabic{equation}}
\renewcommand{\thetable}{S\arabic{table}}
\renewcommand{\thefigure}{S\arabic{figure}}
\section{Proof of the information decomposition}
\label{supp:proof-information}

This section proves Theorem~\ref{thm:information-decomposition}
and Corollary~\ref{cor:MCAR-information}.

\medskip
\noindent\textit{Proof of Theorem~\ref{thm:information-decomposition}.}
For one observation, write
\(\mathcal{O}=(\boldsymbol{Y},M,(1-M)Z)\). Under
\(M\perp Z\mid\boldsymbol{Y}\), the observed-data log-likelihood is
\begin{align}
\ell_{\mathrm{obs}}
(\boldsymbol{\theta},\boldsymbol{\xi})
={}&
\log p_{\boldsymbol{\theta}}(\boldsymbol{Y})
+
(1-M)
\log
\tau_Z(\boldsymbol{Y};\boldsymbol{\theta})
\nonumber\\
&+
M
\log
q(\boldsymbol{Y};\boldsymbol{\theta},\boldsymbol{\xi})
+
(1-M)
\log
\left\{
1-q(\boldsymbol{Y};\boldsymbol{\theta},\boldsymbol{\xi})
\right\}.
\label{supp:eq:obs-loglik}
\end{align}
Let
\[
\boldsymbol{S}_Y
=
\nabla_{\boldsymbol{\theta}}
\log p_{\boldsymbol{\theta}}(\boldsymbol{Y}),
\qquad
\boldsymbol{S}_{Z\mid\boldsymbol{Y}}
=
\nabla_{\boldsymbol{\theta}}
\log\tau_Z(\boldsymbol{Y};\boldsymbol{\theta}),
\]
and, for notational simplicity, write
\(q=q(\boldsymbol{Y};\boldsymbol{\theta},\boldsymbol{\xi})\).
The score contributions arising from the Bernoulli missing-label mechanism are
\begin{equation}
\boldsymbol{S}_{M,\boldsymbol{\theta}}
=
\frac{M-q}{q(1-q)}
q_{\boldsymbol{\theta}},
\qquad
\boldsymbol{S}_{M,\boldsymbol{\xi}}
=
\frac{M-q}{q(1-q)}
q_{\boldsymbol{\xi}},
\label{supp:eq:missing-scores}
\end{equation}
where
\(q_{\boldsymbol{\theta}}=\partial q/\partial\boldsymbol{\theta}\) and
\(q_{\boldsymbol{\xi}}=\partial q/\partial\boldsymbol{\xi}\). Hence
\begin{equation}
\boldsymbol{S}_{\boldsymbol{\theta}}
=
\boldsymbol{S}_Y
+
(1-M)\boldsymbol{S}_{Z\mid\boldsymbol{Y}}
+
\boldsymbol{S}_{M,\boldsymbol{\theta}},
\qquad
\boldsymbol{S}_{\boldsymbol{\xi}}
=
\boldsymbol{S}_{M,\boldsymbol{\xi}}.
\label{supp:eq:joint-scores}
\end{equation}

The conditional class score has mean zero:
\begin{align}
E
\left[
\boldsymbol{S}_{Z\mid\boldsymbol{Y}}
\mid
\boldsymbol{Y}=\boldsymbol{y}
\right]
&=
\sum_{k=1}^{g}
\tau_k(\boldsymbol{y};\boldsymbol{\theta})
\nabla_{\boldsymbol{\theta}}
\log
\tau_k(\boldsymbol{y};\boldsymbol{\theta})
\nonumber\\
&=
\sum_{k=1}^{g}
\nabla_{\boldsymbol{\theta}}
\tau_k(\boldsymbol{y};\boldsymbol{\theta})
=
\boldsymbol{0}.
\label{supp:eq:conditional-score-zero}
\end{align}
Also, since
\(M\mid\boldsymbol{Y}\sim\operatorname{Bernoulli}(q)\),
\[
E(M-q\mid\boldsymbol{Y})=0,
\]
and therefore, by \eqref{supp:eq:missing-scores},
\begin{equation}
E
\left[
\boldsymbol{S}_{M,\boldsymbol{\theta}}
\mid\boldsymbol{Y}
\right]
=
\boldsymbol{0},
\qquad
E
\left[
\boldsymbol{S}_{M,\boldsymbol{\xi}}
\mid\boldsymbol{Y}
\right]
=
\boldsymbol{0}.
\label{supp:eq:missing-score-zero}
\end{equation}

These conditional-mean identities imply the orthogonality needed below.
First,
\[
E
\left[
\boldsymbol{S}_Y
\left\{
(1-M)\boldsymbol{S}_{Z\mid\boldsymbol{Y}}
\right\}^{\top}
\right]
=
\boldsymbol{0},
\]
because, conditional on \(\boldsymbol{Y}\),
\(M\) and \(Z\) are independent and
\(E(\boldsymbol{S}_{Z\mid\boldsymbol{Y}}\mid\boldsymbol{Y})
=\boldsymbol{0}\).
Similarly,
\[
E
\left[
\boldsymbol{S}_Y
\boldsymbol{S}_{M,\boldsymbol{\theta}}^{\top}
\right]
=
E
\left[
\boldsymbol{S}_Y
\boldsymbol{S}_{M,\boldsymbol{\xi}}^{\top}
\right]
=
\boldsymbol{0}.
\]
Finally, conditioning on \(\boldsymbol{Y}\) and using
\(M\perp Z\mid\boldsymbol{Y}\),
\begin{align*}
&E
\left[
(1-M)
\boldsymbol{S}_{Z\mid\boldsymbol{Y}}
\boldsymbol{S}_{M,\boldsymbol{\theta}}^{\top}
\mid
\boldsymbol{Y}
\right]
\\
&\qquad=
E
\left[
(1-M)
\boldsymbol{S}_{M,\boldsymbol{\theta}}^{\top}
\mid
\boldsymbol{Y}
\right]
E
\left[
\boldsymbol{S}_{Z\mid\boldsymbol{Y}}
\mid
\boldsymbol{Y}
\right]
=
\boldsymbol{0},
\end{align*}
with the same conclusion for
\(\boldsymbol{S}_{M,\boldsymbol{\xi}}\).
Thus the marginal feature score, the observed-label score, and the
missingness score contribute orthogonally to Fisher information.

If all labels were observed, the score for
\(\boldsymbol{\theta}\) would be
\(\boldsymbol{S}_{\mathrm{CC}}
=\boldsymbol{S}_Y+\boldsymbol{S}_{Z\mid\boldsymbol{Y}}\).
Using \eqref{supp:eq:conditional-score-zero},
\begin{equation}
\boldsymbol{I}_{\mathrm{CC}}
=
\boldsymbol{I}_Y
+
E
\left[
\boldsymbol{I}_{Z\mid\boldsymbol{Y}}
(\boldsymbol{\theta};\boldsymbol{Y})
\right],
\qquad
\boldsymbol{I}_Y
=
E
\left[
\boldsymbol{S}_Y
\boldsymbol{S}_Y^{\top}
\right].
\label{supp:eq:ICC-decomp}
\end{equation}

Under partial classification, the observed-label score is
\((1-M)\boldsymbol{S}_{Z\mid\boldsymbol{Y}}\). Its conditional
Fisher-information contribution is therefore
\[
E
\left[
(1-M)^2
\boldsymbol{S}_{Z\mid\boldsymbol{Y}}
\boldsymbol{S}_{Z\mid\boldsymbol{Y}}^{\top}
\mid
\boldsymbol{Y}
\right].
\]
Since \((1-M)^2=1-M\), conditional independence gives
\begin{align}
&E
\left[
(1-M)
\boldsymbol{S}_{Z\mid\boldsymbol{Y}}
\boldsymbol{S}_{Z\mid\boldsymbol{Y}}^{\top}
\mid
\boldsymbol{Y}
\right]
\nonumber\\
&\qquad=
E(1-M\mid\boldsymbol{Y})
E
\left[
\boldsymbol{S}_{Z\mid\boldsymbol{Y}}
\boldsymbol{S}_{Z\mid\boldsymbol{Y}}^{\top}
\mid
\boldsymbol{Y}
\right]
\nonumber\\
&\qquad=
(1-q)
\boldsymbol{I}_{Z\mid\boldsymbol{Y}}
(\boldsymbol{\theta};\boldsymbol{Y}).
\label{supp:eq:retained-label-info}
\end{align}
Hence the information lost relative to complete classification is
\begin{equation}
\boldsymbol{D}
=
E
\left[
q(\boldsymbol{Y};\boldsymbol{\theta},\boldsymbol{\xi})
\boldsymbol{I}_{Z\mid\boldsymbol{Y}}
(\boldsymbol{\theta};\boldsymbol{Y})
\right].
\label{supp:eq:D}
\end{equation}

It remains to determine the information supplied by the missing-label
indicators. From \eqref{supp:eq:missing-scores} and
\[
E
\left[
(M-q)^2
\mid
\boldsymbol{Y}
\right]
=
q(1-q),
\]
their Fisher-information blocks are
\begin{equation}
\boldsymbol{B}_{ab}
=
E
\left[
\frac{
q_a q_b^{\top}
}{
q(1-q)
}
\right],
\qquad
a,b\in
\{\boldsymbol{\theta},\boldsymbol{\xi}\}.
\label{supp:eq:B-blocks}
\end{equation}
Combining these blocks with the preceding orthogonality relations, the joint
observed-data Fisher information for
\((\boldsymbol{\theta}^{\top},
\boldsymbol{\xi}^{\top})^{\top}\) is
\begin{equation}
\boldsymbol{I}_{\mathrm{obs}}
=
\begin{pmatrix}
\boldsymbol{I}_{\mathrm{CC}}
-
\boldsymbol{D}
+
\boldsymbol{B}_{\boldsymbol{\theta}\boldsymbol{\theta}}
&
\boldsymbol{B}_{\boldsymbol{\theta}\boldsymbol{\xi}}
\\[2mm]
\boldsymbol{B}_{\boldsymbol{\xi}\boldsymbol{\theta}}
&
\boldsymbol{B}_{\boldsymbol{\xi}\boldsymbol{\xi}}
\end{pmatrix}.
\label{supp:eq:joint-information}
\end{equation}

The efficient Fisher information for
\(\boldsymbol{\theta}\), after eliminating the missingness-specific nuisance
parameter \(\boldsymbol{\xi}\), is the Schur complement of
\(\boldsymbol{B}_{\boldsymbol{\xi}\boldsymbol{\xi}}\):
\begin{align}
\boldsymbol{I}_{\mathrm{PC}}^{\mathrm{eff}}
={}&
\boldsymbol{I}_{\mathrm{CC}}
-
\boldsymbol{D}
+
\boldsymbol{B}_{\boldsymbol{\theta}\boldsymbol{\theta}}
\nonumber\\
&-
\boldsymbol{B}_{\boldsymbol{\theta}\boldsymbol{\xi}}
\boldsymbol{B}_{\boldsymbol{\xi}\boldsymbol{\xi}}^{-1}
\boldsymbol{B}_{\boldsymbol{\xi}\boldsymbol{\theta}}.
\label{supp:eq:Schur-information}
\end{align}
Therefore, with
\begin{equation}
\boldsymbol{I}_{M}^{\mathrm{eff}}
=
\boldsymbol{B}_{\boldsymbol{\theta}\boldsymbol{\theta}}
-
\boldsymbol{B}_{\boldsymbol{\theta}\boldsymbol{\xi}}
\boldsymbol{B}_{\boldsymbol{\xi}\boldsymbol{\xi}}^{-1}
\boldsymbol{B}_{\boldsymbol{\xi}\boldsymbol{\theta}},
\label{supp:eq:IM-eff}
\end{equation}
we obtain
\begin{equation}
\boldsymbol{I}_{\mathrm{PC}}^{\mathrm{eff}}
=
\boldsymbol{I}_{\mathrm{CC}}
-
\boldsymbol{D}
+
\boldsymbol{I}_{M}^{\mathrm{eff}}.
\label{supp:eq:information-decomposition}
\end{equation}

It remains only to verify that
\(\boldsymbol{I}_{M}^{\mathrm{eff}}\) is positive semidefinite. The Bernoulli
information matrix
\[
\boldsymbol{B}
=
\begin{pmatrix}
\boldsymbol{B}_{\boldsymbol{\theta}\boldsymbol{\theta}}
&
\boldsymbol{B}_{\boldsymbol{\theta}\boldsymbol{\xi}}
\\
\boldsymbol{B}_{\boldsymbol{\xi}\boldsymbol{\theta}}
&
\boldsymbol{B}_{\boldsymbol{\xi}\boldsymbol{\xi}}
\end{pmatrix}
\]
is the covariance matrix of the joint missingness score
\[
\left(
\boldsymbol{S}_{M,\boldsymbol{\theta}}^{\top},
\boldsymbol{S}_{M,\boldsymbol{\xi}}^{\top}
\right)^{\top},
\]
and is therefore positive semidefinite. By assumption,
\(\boldsymbol{B}_{\boldsymbol{\xi}\boldsymbol{\xi}}\) is nonsingular; being a
nonsingular principal submatrix of a positive semidefinite matrix, it is
positive definite. Its Schur complement is therefore positive semidefinite,
which gives
\[
\boldsymbol{I}_{M}^{\mathrm{eff}}
\succeq
\boldsymbol{0}.
\]
This proves Theorem~\ref{thm:information-decomposition}.
\hfill\(\square\)

\medskip
\noindent\textit{roof of Corollary~\ref{cor:MCAR-information}.}
Under missing completely at random,
\[
q(\boldsymbol{Y};\boldsymbol{\theta},\boldsymbol{\xi})
=
\gamma,
\qquad
0<\gamma<1,
\]
where \(\gamma\) is variation-independent of
\(\boldsymbol{\theta}\). Hence
\(q_{\boldsymbol{\theta}}=\boldsymbol{0}\), and consequently
\[
\boldsymbol{B}_{\boldsymbol{\theta}\boldsymbol{\theta}}
=
\boldsymbol{0},
\qquad
\boldsymbol{B}_{\boldsymbol{\theta}\boldsymbol{\xi}}
=
\boldsymbol{0}.
\]
It follows that
\(\boldsymbol{I}_{M}^{\mathrm{eff}}=\boldsymbol{0}\).
Moreover, the label-information loss reduces to
\[
\boldsymbol{D}
=
\gamma
E
\left[
\boldsymbol{I}_{Z\mid\boldsymbol{Y}}
(\boldsymbol{\theta};\boldsymbol{Y})
\right].
\]
Substituting these expressions into the information decomposition established
above gives
\begin{equation}
\boldsymbol{I}_{\mathrm{PC}}^{\mathrm{eff}}
=
\boldsymbol{I}_{\mathrm{CC}}
-
\gamma
E
\left[
\boldsymbol{I}_{Z\mid\boldsymbol{Y}}
(\boldsymbol{\theta};\boldsymbol{Y})
\right].
\label{supp:eq:MCAR-information}
\end{equation}
Since
\(\boldsymbol{I}_{Z\mid\boldsymbol{Y}}
(\boldsymbol{\theta};\boldsymbol{Y})\)
is positive semidefinite for every \(\boldsymbol{Y}\), its expectation is
positive semidefinite, and therefore
\[
\boldsymbol{I}_{\mathrm{PC}}^{\mathrm{eff}}
\preceq
\boldsymbol{I}_{\mathrm{CC}}.
\]
This proves Corollary~\ref{cor:MCAR-information}.
\hfill\(\square\)

\section{Proof of the nuisance-parameter information result}
\label{supp:nuisance-information}

This section proves Proposition~\ref{prop:nuisance-coupling}.

\medskip
\noindent\textit{Proof of Proposition~\ref{prop:nuisance-coupling}.}
Partition the complete-classification information matrix according to
\(\boldsymbol{\theta}
=
(\boldsymbol{\beta}^{\top},
\boldsymbol{\lambda}^{\top})^{\top}\) as
\[
\boldsymbol{A}
=
\begin{pmatrix}
\boldsymbol{A}_{\beta\beta}
&
\boldsymbol{A}_{\beta\lambda}
\\
\boldsymbol{A}_{\lambda\beta}
&
\boldsymbol{A}_{\lambda\lambda}
\end{pmatrix},
\]
and suppose that
\(\boldsymbol{A}_{\lambda\lambda}\) is nonsingular. Since
\(\boldsymbol{A}\) is a Fisher-information matrix and
\(\boldsymbol{A}_{\lambda\lambda}\) is a nonsingular principal block,
\(\boldsymbol{A}_{\lambda\lambda}\) is positive definite. Define
\[
\boldsymbol{R}
=
\boldsymbol{A}_{\beta\lambda}
\boldsymbol{A}_{\lambda\lambda}^{-1}.
\]
If
\(\boldsymbol{S}_{\beta}\) and
\(\boldsymbol{S}_{\lambda}\) are the corresponding complete-classification
score components, then the efficient score for
\(\boldsymbol{\beta}\) relative to
\(\boldsymbol{\lambda}\) is
\[
\boldsymbol{S}_{\beta}^{\,\mathrm{eff}}
=
\boldsymbol{S}_{\beta}
-
\boldsymbol{R}\boldsymbol{S}_{\lambda}.
\]
Indeed,
\[
E
\left[
\boldsymbol{S}_{\beta}^{\,\mathrm{eff}}
\boldsymbol{S}_{\lambda}^{\top}
\right]
=
\boldsymbol{A}_{\beta\lambda}
-
\boldsymbol{R}
\boldsymbol{A}_{\lambda\lambda}
=
\boldsymbol{0},
\]
and its covariance is the usual Schur complement
\begin{equation}
\boldsymbol{A}_{\mathrm{eff}}(\boldsymbol{\beta})
=
\boldsymbol{A}_{\beta\beta}
-
\boldsymbol{A}_{\beta\lambda}
\boldsymbol{A}_{\lambda\lambda}^{-1}
\boldsymbol{A}_{\lambda\beta}.
\label{supp:eq:Aeff-beta}
\end{equation}

After eliminating the missingness-specific parameter
\(\boldsymbol{\xi}\), Theorem~\ref{thm:information-decomposition} gives
\[
\boldsymbol{J}
=
\boldsymbol{A}
+
\boldsymbol{K},
\qquad
\boldsymbol{K}
=
\boldsymbol{I}_{M}^{\mathrm{eff}}
-
\boldsymbol{D}.
\]
Write
\[
\boldsymbol{K}
=
\begin{pmatrix}
\boldsymbol{K}_{\beta\beta}
&
\boldsymbol{K}_{\beta\lambda}
\\
\boldsymbol{K}_{\lambda\beta}
&
\boldsymbol{K}_{\lambda\lambda}
\end{pmatrix}
\]
and introduce the nonsingular block-triangular matrix
\[
\boldsymbol{T}
=
\begin{pmatrix}
\boldsymbol{I}_{r}
&
-\boldsymbol{R}
\\
\boldsymbol{0}
&
\boldsymbol{I}_{s}
\end{pmatrix}.
\]
Multiplication of the score vector by
\(\boldsymbol{T}\) replaces
\(\boldsymbol{S}_{\beta}\) by
\(\boldsymbol{S}_{\beta}
-\boldsymbol{R}\boldsymbol{S}_{\lambda}\)
while leaving
\(\boldsymbol{S}_{\lambda}\) unchanged. The corresponding information
matrix is therefore transformed by congruence.

For the complete-classification information,
\begin{equation}
\boldsymbol{T}
\boldsymbol{A}
\boldsymbol{T}^{\top}
=
\begin{pmatrix}
\boldsymbol{A}_{\mathrm{eff}}(\boldsymbol{\beta})
&
\boldsymbol{0}
\\
\boldsymbol{0}
&
\boldsymbol{A}_{\lambda\lambda}
\end{pmatrix}.
\label{supp:eq:TAT}
\end{equation}
The vanishing off-diagonal block follows from
\(\boldsymbol{A}_{\beta\lambda}
-\boldsymbol{R}\boldsymbol{A}_{\lambda\lambda}
=\boldsymbol{0}\), while the upper-left block is
\eqref{supp:eq:Aeff-beta}.

Applying the same transformation to the perturbation
\(\boldsymbol{K}\) gives
\begin{equation}
\boldsymbol{T}
\boldsymbol{K}
\boldsymbol{T}^{\top}
=
\begin{pmatrix}
\widetilde{\boldsymbol{K}}_{\beta\beta}
&
\widetilde{\boldsymbol{K}}_{\beta\lambda}
\\
\widetilde{\boldsymbol{K}}_{\lambda\beta}
&
\boldsymbol{K}_{\lambda\lambda}
\end{pmatrix},
\label{supp:eq:TKT}
\end{equation}
where
\begin{align}
\widetilde{\boldsymbol{K}}_{\beta\beta}
={}&
\boldsymbol{K}_{\beta\beta}
-
\boldsymbol{R}\boldsymbol{K}_{\lambda\beta}
-
\boldsymbol{K}_{\beta\lambda}\boldsymbol{R}^{\top}
+
\boldsymbol{R}
\boldsymbol{K}_{\lambda\lambda}
\boldsymbol{R}^{\top},
\label{supp:eq:Ktilde-bb}\\
\widetilde{\boldsymbol{K}}_{\beta\lambda}
={}&
\boldsymbol{K}_{\beta\lambda}
-
\boldsymbol{R}
\boldsymbol{K}_{\lambda\lambda},
\label{supp:eq:Ktilde-bl}
\end{align}
and
\(\widetilde{\boldsymbol{K}}_{\lambda\beta}
=
\widetilde{\boldsymbol{K}}_{\beta\lambda}^{\top}\)
because \(\boldsymbol{K}\) is symmetric.

Combining \eqref{supp:eq:TAT} and
\eqref{supp:eq:TKT} yields
\begin{equation}
\boldsymbol{J}^{\star}
:=
\boldsymbol{T}
\boldsymbol{J}
\boldsymbol{T}^{\top}
=
\begin{pmatrix}
\boldsymbol{A}_{\mathrm{eff}}(\boldsymbol{\beta})
+
\widetilde{\boldsymbol{K}}_{\beta\beta}
&
\widetilde{\boldsymbol{K}}_{\beta\lambda}
\\
\widetilde{\boldsymbol{K}}_{\lambda\beta}
&
\boldsymbol{J}_{\lambda\lambda}
\end{pmatrix},
\label{supp:eq:Jstar}
\end{equation}
where
\[
\boldsymbol{J}_{\lambda\lambda}
=
\boldsymbol{A}_{\lambda\lambda}
+
\boldsymbol{K}_{\lambda\lambda}.
\]

Because \(\boldsymbol{T}\) is nonsingular and its transformation replaces
the score for \(\boldsymbol{\beta}\) by that score minus a linear combination
of the nuisance score while leaving the nuisance-score space unchanged, the
efficient information for \(\boldsymbol{\beta}\) is invariant under this
transformation. Equivalently, it is obtained as the Schur complement of the
lower-right block of \(\boldsymbol{J}^{\star}\). Therefore
\begin{equation}
\boldsymbol{J}_{\mathrm{eff}}(\boldsymbol{\beta})
=
\boldsymbol{A}_{\mathrm{eff}}(\boldsymbol{\beta})
+
\widetilde{\boldsymbol{K}}_{\beta\beta}
-
\widetilde{\boldsymbol{K}}_{\beta\lambda}
\boldsymbol{J}_{\lambda\lambda}^{-1}
\widetilde{\boldsymbol{K}}_{\lambda\beta}.
\label{supp:eq:Jeff-proof}
\end{equation}
This is the first assertion of Proposition~\ref{prop:nuisance-coupling}.

Now define
\begin{equation}
\boldsymbol{C}_{K}
=
\widetilde{\boldsymbol{K}}_{\beta\lambda}
\boldsymbol{J}_{\lambda\lambda}^{-1}
\widetilde{\boldsymbol{K}}_{\lambda\beta}.
\label{supp:eq:CK}
\end{equation}
The matrix
\(\boldsymbol{J}_{\lambda\lambda}\) is a principal block of the efficient
Fisher-information matrix \(\boldsymbol{J}\). Under the assumed
nonsingularity, it is therefore positive definite, and so is
\(\boldsymbol{J}_{\lambda\lambda}^{-1}\). Consequently, for every
\(\boldsymbol{x}\in\mathbb{R}^{r}\),
\begin{align}
\boldsymbol{x}^{\top}
\boldsymbol{C}_{K}
\boldsymbol{x}
&=
\left(
\widetilde{\boldsymbol{K}}_{\lambda\beta}
\boldsymbol{x}
\right)^{\top}
\boldsymbol{J}_{\lambda\lambda}^{-1}
\left(
\widetilde{\boldsymbol{K}}_{\lambda\beta}
\boldsymbol{x}
\right)
\nonumber\\
&\geq
0.
\label{supp:eq:CK-psd}
\end{align}
Thus
\[
\boldsymbol{C}_{K}
\succeq
\boldsymbol{0},
\]
which proves the second assertion.

For completeness, the sequential nuisance elimination used above is equivalent
to eliminating
\(\boldsymbol{\lambda}\) and
\(\boldsymbol{\xi}\) jointly from the full observed-data information matrix.
If that matrix is partitioned according to
\((\boldsymbol{\beta}^{\top},
\boldsymbol{\lambda}^{\top},
\boldsymbol{\xi}^{\top})^{\top}\),
the quotient identity for Schur complements gives, under the required
invertibility conditions,
\begin{equation}
\operatorname{Schur}_{(\lambda,\xi)}
\left(
\boldsymbol{I}_{\mathrm{obs}}
\right)
=
\operatorname{Schur}_{\lambda}
\left\{
\operatorname{Schur}_{\xi}
\left(
\boldsymbol{I}_{\mathrm{obs}}
\right)
\right\}.
\label{supp:eq:sequential-Schur}
\end{equation}
Thus eliminating the missingness-specific nuisance parameter first and then
the data-model nuisance parameter gives exactly the same efficient information
for \(\boldsymbol{\beta}\) as their joint elimination.

This completes the proof of Proposition~\ref{prop:nuisance-coupling}.
\hfill\(\square\)

\section{Proof of the multiclass excess-risk expansion}
\label{supp:risk-proof}

This section establishes the local boundary calculation underlying
Theorem~\ref{thm:multiclass-risk}.

\medskip
\noindent\textit{Proof of Theorem~\ref{thm:multiclass-risk}.}
For an arbitrary deterministic classifier \(C\), evaluated under the true
parameter \(\boldsymbol{\theta}_0\),
\[
R(C)
=
1-
\int
r_{C(\boldsymbol{y})}^{0}(\boldsymbol{y})
\,d\boldsymbol{y},
\]
where
\(r_k^0(\boldsymbol{y})
=
r_k(\boldsymbol{y};\boldsymbol{\theta}_0)\).
Since the Bayes classifier \(C_0\) maximizes the prior-weighted class density
pointwise,
\[
R^\ast
=
1-
\int
\max_{1\leq k\leq g}
r_k^0(\boldsymbol{y})
\,d\boldsymbol{y}.
\]
Consequently,
\begin{equation}
R(C)-R^\ast
=
\int
\left\{
\max_k r_k^0(\boldsymbol{y})
-
r_{C(\boldsymbol{y})}^0(\boldsymbol{y})
\right\}
\,d\boldsymbol{y}.
\label{supp:eq:exact-risk}
\end{equation}
For the classifier
\(C_{\boldsymbol{h}}\) determined by
\(\boldsymbol{\theta}_0+\boldsymbol{h}\), this becomes
\begin{equation}
R(\boldsymbol{\theta}_0+\boldsymbol{h})
-
R^\ast
=
\int
\left\{
r_{C_0(\boldsymbol{y})}^0(\boldsymbol{y})
-
r_{C_{\boldsymbol{h}}(\boldsymbol{y})}^0(\boldsymbol{y})
\right\}
\,d\boldsymbol{y},
\label{supp:eq:plugin-risk}
\end{equation}
up to the immaterial choice of classifier on sets of Bayes ties of probability
zero.

We first localize the region in which
\(C_{\boldsymbol{h}}\) can differ from \(C_0\). Let
\[
m(\boldsymbol{y})
=
r_{(1)}^0(\boldsymbol{y})
-
r_{(2)}^0(\boldsymbol{y}),
\]
where \(r_{(1)}^0\) and \(r_{(2)}^0\) denote the largest and second-largest
values among
\(\{r_1^0(\boldsymbol{y}),\ldots,r_g^0(\boldsymbol{y})\}\).
On any compact set separated from the Bayes boundary,
continuity implies
\(m(\boldsymbol{y})\geq c\) for some \(c>0\).
Joint differentiability in
\((\boldsymbol{y},\boldsymbol{\theta})\) gives, uniformly on such a set,
\[
\max_k
\left|
r_k(\boldsymbol{y};\boldsymbol{\theta}_0+\boldsymbol{h})
-
r_k^0(\boldsymbol{y})
\right|
=
O(\|\boldsymbol{h}\|).
\]
It follows that, for sufficiently small \(\boldsymbol{h}\), a change in the
winning class is possible only where
\[
m(\boldsymbol{y})
=
O(\|\boldsymbol{h}\|).
\]
Thus the disagreement between the true and perturbed classifiers is confined
to a shrinking neighborhood of the Bayes boundary.

Consider now an interior point
\(\boldsymbol{s}\) of an active face
\(\mathcal{F}_{kl}\). By definition,
\[
r_k^0(\boldsymbol{s})
=
r_l^0(\boldsymbol{s})
>
r_m^0(\boldsymbol{s}),
\qquad
m\notin\{k,l\}.
\]
The strict inequality implies that, in a sufficiently small neighborhood of
\(\boldsymbol{s}\), classes other than \(k\) and \(l\) remain separated from
the two leading classes. By continuity in the parameter, this remains true for
all sufficiently small \(\boldsymbol{h}\). Hence, locally, both
\(C_0\) and \(C_{\boldsymbol{h}}\) are determined solely by the sign of
\[
g_{kl}(\boldsymbol{y};\boldsymbol{\theta})
=
r_k(\boldsymbol{y};\boldsymbol{\theta})
-
r_l(\boldsymbol{y};\boldsymbol{\theta}).
\]

Write
\(g_{kl}^0(\boldsymbol{y})
=
g_{kl}(\boldsymbol{y};\boldsymbol{\theta}_0)\).
At a regular point of the face,
\(\nabla_{\boldsymbol{y}}g_{kl}^0(\boldsymbol{s})\neq\boldsymbol{0}\).
On a sufficiently small tubular neighborhood of a compact regular portion of
\(\mathcal{F}_{kl}\), define
\[
\boldsymbol{V}_{kl}(\boldsymbol{y})
=
\frac{
\nabla_{\boldsymbol{y}}g_{kl}^0(\boldsymbol{y})
}{
\|
\nabla_{\boldsymbol{y}}g_{kl}^0(\boldsymbol{y})
\|^2
}.
\]
Let
\(\boldsymbol{\Phi}_{kl}(\boldsymbol{s},u)\) be the associated local flow,
initialized at
\(\boldsymbol{\Phi}_{kl}(\boldsymbol{s},0)=\boldsymbol{s}\).
Along this flow,
\begin{align*}
\frac{\partial}{\partial u}
g_{kl}^0
\left\{
\boldsymbol{\Phi}_{kl}(\boldsymbol{s},u)
\right\}
&=
\nabla_{\boldsymbol{y}}g_{kl}^0
\left\{
\boldsymbol{\Phi}_{kl}(\boldsymbol{s},u)
\right\}^{\top}
\boldsymbol{V}_{kl}
\left\{
\boldsymbol{\Phi}_{kl}(\boldsymbol{s},u)
\right\}
\\
&=1.
\end{align*}
Since
\(g_{kl}^0(\boldsymbol{s})=0\), it follows that
\begin{equation}
g_{kl}^0
\left\{
\boldsymbol{\Phi}_{kl}(\boldsymbol{s},u)
\right\}
=
u.
\label{supp:eq:g-equals-u}
\end{equation}
Thus \(u\) is the pairwise contrast itself and provides a normal coordinate
to the true boundary.

The corresponding volume element satisfies, uniformly on compact regular
portions of the face,
\begin{equation}
d\boldsymbol{y}
=
\left\{
\frac{
1
}{
\|
\nabla_{\boldsymbol{y}}
g_{kl}^0(\boldsymbol{s})
\|
}
+
O(|u|)
\right\}
dS(\boldsymbol{s})\,du.
\label{supp:eq:coarea-local}
\end{equation}
This follows from the coarea formula, or equivalently from the Jacobian of the
local tubular coordinate map.

Define
\[
\boldsymbol{b}_{kl}(\boldsymbol{s})
=
\nabla_{\boldsymbol{\theta}}
g_{kl}(\boldsymbol{s};\boldsymbol{\theta}_0).
\]
A Taylor expansion in the parameter, together with smoothness of the coordinate
map, gives
\begin{align}
g_{kl}
\left\{
\boldsymbol{\Phi}_{kl}(\boldsymbol{s},u);
\boldsymbol{\theta}_0+\boldsymbol{h}
\right\}
={}&
u
+
\boldsymbol{b}_{kl}(\boldsymbol{s})^{\top}
\boldsymbol{h}
\nonumber\\
&+
O
\left(
|u|\,\|\boldsymbol{h}\|
+
\|\boldsymbol{h}\|^2
\right),
\label{supp:eq:perturbed-g}
\end{align}
uniformly on compact regular portions of
\(\mathcal{F}_{kl}\).
Let \(u_{\boldsymbol{h}}(\boldsymbol{s})\) denote the normal coordinate of the
perturbed \(k\)-versus-\(l\) boundary. The implicit-function theorem applied
to \eqref{supp:eq:perturbed-g} yields
\begin{equation}
u_{\boldsymbol{h}}(\boldsymbol{s})
=
-
\boldsymbol{b}_{kl}(\boldsymbol{s})^{\top}
\boldsymbol{h}
+
O(\|\boldsymbol{h}\|^2),
\label{supp:eq:boundary-displacement}
\end{equation}
again uniformly on such compact portions.

Within this local binary neighborhood, disagreement between the true and
perturbed rules occurs precisely in the strip between the two boundaries.
The excess loss at
\(\boldsymbol{y}\) is
\[
\left|
r_k^0(\boldsymbol{y})
-
r_l^0(\boldsymbol{y})
\right|
=
|g_{kl}^0(\boldsymbol{y})|
=
|u|,
\]
where the last equality follows from
\eqref{supp:eq:g-equals-u}. Therefore, for a fixed boundary point
\(\boldsymbol{s}\), the excess-risk contribution per unit surface measure
from the displaced strip is
\begin{align}
\Delta R_{kl,\boldsymbol{s}}(\boldsymbol{h})
&=
\int_{\min\{0,u_{\boldsymbol{h}}(\boldsymbol{s})\}}^{
      \max\{0,u_{\boldsymbol{h}}(\boldsymbol{s})\}}
|u|
\left\{
\frac{1}{
\|\nabla_{\boldsymbol{y}}g_{kl}^0(\boldsymbol{s})\|
}
+
O(|u|)
\right\}
du
\nonumber\\
&=
\frac{1}{2}
\frac{
u_{\boldsymbol{h}}(\boldsymbol{s})^2
}{
\|
\nabla_{\boldsymbol{y}}
g_{kl}^0(\boldsymbol{s})
\|
}
+
O(\|\boldsymbol{h}\|^3).
\label{supp:eq:local-strip-risk}
\end{align}
Using \eqref{supp:eq:boundary-displacement},
\begin{equation}
\Delta R_{kl,\boldsymbol{s}}(\boldsymbol{h})
=
\frac{1}{2}
\frac{
\left\{
\boldsymbol{b}_{kl}(\boldsymbol{s})^{\top}
\boldsymbol{h}
\right\}^{2}
}{
\|
\nabla_{\boldsymbol{y}}
g_{kl}^0(\boldsymbol{s})
\|
}
+
O(\|\boldsymbol{h}\|^3).
\label{supp:eq:local-risk}
\end{equation}
Integrating over any compact regular portion
\(\mathcal{F}_{kl}^{(0)}\subset\mathcal{F}_{kl}\) that is bounded away
from higher-order Bayes tie sets gives
\begin{equation}
\Delta R_{kl}^{(0)}(\boldsymbol{h})
=
\frac{1}{2}
\boldsymbol{h}^{\top}
\boldsymbol{H}_{kl}^{(0)}
\boldsymbol{h}
+
o(\|\boldsymbol{h}\|^2),
\label{supp:eq:face-risk}
\end{equation}
where
\begin{equation}
\boldsymbol{H}_{kl}^{(0)}
=
\int_{\mathcal{F}_{kl}^{(0)}}
\frac{
\boldsymbol{b}_{kl}(\boldsymbol{s})
\boldsymbol{b}_{kl}(\boldsymbol{s})^{\top}
}{
\|
\nabla_{\boldsymbol{y}}
g_{kl}^0(\boldsymbol{s})
\|
}
\,dS(\boldsymbol{s}).
\label{supp:eq:Hkl-local-proof}
\end{equation}

It remains to control neighborhoods of higher-order Bayes tie sets and to
justify extending the preceding facewise calculation to the whole active
boundary. Let
\(\varepsilon=\|\boldsymbol{h}\|\).
Consider a point at which \(r\geq3\) classes tie at the Bayes maximum.
By the transversality assumption, in a neighborhood of such a point one may
choose \(r-1\) independent pairwise contrasts as local normal coordinates,
say
\[
\boldsymbol{u}
=
(u_1,\ldots,u_{r-1})^{\top},
\qquad
u_j
=
g_{k_1k_{j+1}}^0(\boldsymbol{y}).
\]
The constant-rank theorem then provides local coordinates
\((\boldsymbol{s},\boldsymbol{u})\), where
\(\boldsymbol{s}\) parameterizes the tie stratum and the associated Jacobian
is bounded above and below on compact coordinate neighborhoods.

A perturbation
\(\boldsymbol{\theta}_0+\boldsymbol{h}\) changes each of these pairwise
contrasts by \(O(\varepsilon)\), uniformly on such a compact neighborhood.
Consequently, a change in the genuinely \(r\)-way ordering can occur only
when
\[
\|\boldsymbol{u}\|=O(\varepsilon).
\]
Since there are \(r-1\) independent normal coordinates, the volume of this
region is \(O(\varepsilon^{r-1})\) per unit measure of the tie stratum.
Moreover, because the competing prior-weighted densities agree on the tie
stratum and are continuously differentiable, their differences throughout
this region are \(O(\varepsilon)\). Hence the excess-risk contribution of
the neighborhood of an \(r\)-way transversal tie is
\[
O(\varepsilon^{r-1})\,O(\varepsilon)
=
O(\varepsilon^r)
=
o(\varepsilon^2),
\qquad r\geq3.
\]

The same argument controls the portions of the adjacent active pairwise
faces removed when forming regular tubular neighborhoods. For a triple tie,
for example, the excluded portion of an adjacent face has surface measure
\(O(\varepsilon)\), while its displaced-strip contribution per unit surface
measure is \(O(\varepsilon^2)\), giving \(O(\varepsilon^3)\). More generally,
the corresponding contribution near an \(r\)-way transversal tie is of
order \(O(\varepsilon^r)\).

Therefore neighborhoods of higher-order tie strata contribute only
\(o(\varepsilon^2)\). Away from these neighborhoods, the active pairwise
faces are regular and the preceding tubular-coordinate calculation applies
uniformly on compact portions. Moreover, by transversality and smoothness,
the integrand
\[
\frac{
\boldsymbol{b}_{kl}(\boldsymbol{s})
\boldsymbol{b}_{kl}(\boldsymbol{s})^{\top}
}{
\|
\nabla_{\boldsymbol{y}}g_{kl}^0(\boldsymbol{s})
\|
}
\]
is locally bounded near each higher-order tie stratum. The portion of an
active pairwise face lying within an \(O(\varepsilon)\) neighborhood of a
triple-tie stratum has surface measure \(O(\varepsilon)\), and the
corresponding difference between the truncated and full quadratic
coefficients is therefore \(O(\varepsilon)\). After multiplication by
\(\|\boldsymbol{h}\|^2=\varepsilon^2\), this contributes
\(O(\varepsilon^3)=o(\varepsilon^2)\). Higher-order tie strata give still
smaller orders. Consequently, the truncated face integrals may be replaced
by the integrals over the full active faces at quadratic order. Thus, for
each active pair \(k<l\),
\begin{equation}
\Delta R_{kl}(\boldsymbol{h})
=
\frac{1}{2}
\boldsymbol{h}^{\top}
\boldsymbol{H}_{kl}
\boldsymbol{h}
+
o(\|\boldsymbol{h}\|^2),
\end{equation}
where
\begin{equation}
\boldsymbol{H}_{kl}
=
\int_{\mathcal{F}_{kl}}
\frac{
\boldsymbol{b}_{kl}(\boldsymbol{s})
\boldsymbol{b}_{kl}(\boldsymbol{s})^{\top}
}{
\|
\nabla_{\boldsymbol{y}}
g_{kl}^0(\boldsymbol{s})
\|
}
\,dS(\boldsymbol{s}).
\label{supp:eq:Hkl-proof}
\end{equation}
Thus only the regular active pairwise faces contribute to the quadratic term.

For compact active boundaries on bounded feature support, summing the
preceding facewise expansions over all active pairs $k<l$ gives
\[
R(\theta_0+h)-R^\ast
=
\frac12 h^\top H_R h+o(\|h\|^2),
\qquad
H_R=\sum_{k<l}H_{kl}.
\]
This establishes the quadratic expansion directly in the bounded-support
case. Supplementary Section~\ref{supp:noncompact} extends the argument to
unbounded feature supports under the stated tail condition and, when one or more active faces
are noncompact, the additional regular-exhaustion and
boundary-integrability conditions.

Finally, for any vector \(\boldsymbol{v}\),
\begin{align*}
\boldsymbol{v}^{\top}
\boldsymbol{H}_{R}
\boldsymbol{v}
&=
\sum_{k<l}
\int_{\mathcal{F}_{kl}}
\frac{
\left\{
\boldsymbol{b}_{kl}(\boldsymbol{s})^{\top}
\boldsymbol{v}
\right\}^{2}
}{
\|
\nabla_{\boldsymbol{y}}
g_{kl}^0(\boldsymbol{s})
\|
}
\,dS(\boldsymbol{s})
\\
&\geq0.
\end{align*}
Hence
\(\boldsymbol{H}_{R}\succeq\boldsymbol{0}\), completing the proof of
Theorem~\ref{thm:multiclass-risk}.
\hfill\(\square\)


\section{Noncompact boundaries and tail conditions}
\label{supp:noncompact}

This section states the tail condition used when the feature support is
unbounded and the additional regular-exhaustion and boundary-integrability
conditions required when one or more active Bayes faces are noncompact. It
then completes the corresponding extension of
Theorem~\ref{thm:multiclass-risk}.

Let
\[
\mathcal{K}_L
=
\left\{
\boldsymbol{y}\in\mathcal{Y}:
\|\boldsymbol{y}\|\leq L
\right\},
\qquad
\mathcal{F}_{kl}^{(L)}
=
\mathcal{F}_{kl}\cap\mathcal{K}_L .
\]
For each active pair \(k<l\), define the truncated curvature matrix
\begin{equation}
\boldsymbol{H}_{kl}^{(L)}
=
\int_{\mathcal{F}_{kl}^{(L)}}
\frac{
\boldsymbol{b}_{kl}(\boldsymbol{s})
\boldsymbol{b}_{kl}(\boldsymbol{s})^{\top}
}{
\left\|
\nabla_{\boldsymbol{y}}
g_{kl}(\boldsymbol{s};\boldsymbol{\theta}_0)
\right\|
}
\,dS(\boldsymbol{s}),
\label{supp:eq:Hkl-truncated}
\end{equation}
where
\[
\boldsymbol{b}_{kl}(\boldsymbol{s})
=
\nabla_{\boldsymbol{\theta}}
g_{kl}(\boldsymbol{s};\boldsymbol{\theta}_0).
\]

For noncompact active faces, we assume the following regular-exhaustion
condition in addition to boundary integrability. Tail control is imposed
separately whenever the feature support is unbounded. There exists
an increasing sequence \(L_m\to\infty\) such that, for every \(m\), each
regular active face intersects \(\partial\mathcal{K}_{L_m}\) transversely
whenever the intersection is nonempty.

We further assume the boundary-integrability condition
\begin{equation}
\int_{\mathcal{F}_{kl}}
\frac{
\left\|
\boldsymbol{b}_{kl}(\boldsymbol{s})
\right\|^2
}{
\left\|
\nabla_{\boldsymbol{y}}
g_{kl}(\boldsymbol{s};\boldsymbol{\theta}_0)
\right\|
}
\,dS(\boldsymbol{s})
<
\infty
\qquad
\text{for every nonempty }\mathcal{F}_{kl}.
\label{supp:eq:boundary-integrability}
\end{equation}
This condition guarantees that the matrix integral defining
\(\boldsymbol{H}_{kl}\) is finite and that
\[
\boldsymbol{H}_{kl}^{(L)}
\longrightarrow
\boldsymbol{H}_{kl}
\qquad
\text{as }L\to\infty.
\]

A separate condition is required to control classification disagreement in the
tails. Let
\[
\Delta_{\boldsymbol{h}}(\boldsymbol{y})
=
r_{C_0(\boldsymbol{y})}^{0}(\boldsymbol{y})
-
r_{C_{\boldsymbol{h}}(\boldsymbol{y})}^{0}(\boldsymbol{y})
\geq0
\]
denote the pointwise excess loss. We assume
\begin{equation}
\lim_{L\to\infty}
\;
\limsup_{\boldsymbol{h}\to\boldsymbol{0}}
\frac{
1
}{
\|\boldsymbol{h}\|^2
}
\int_{\mathcal{K}_L^{\,c}}
\Delta_{\boldsymbol{h}}(\boldsymbol{y})
\,d\boldsymbol{y}
=
0.
\label{supp:eq:tail-condition}
\end{equation}
Thus, after scaling by the quadratic order
\(\|\boldsymbol{h}\|^2\), the contribution to excess risk from sufficiently
far into the tails is uniformly negligible for small perturbations.

To complete the proof of Theorem~\ref{thm:multiclass-risk}, fix \(L<\infty\) such that the
active faces intersect \(\partial\mathcal{K}_L\) transversely whenever
the intersections are nonempty. For the noncompact extension we use an
increasing sequence of such regular truncation radii tending to infinity.
On \(\mathcal{K}_L\), the regular portions of the active faces are compact
after excluding arbitrarily small neighborhoods of transversal higher-order
junctions and of their intersections with
\(\partial\mathcal{K}_L\). The latter intersections are codimension two in
the ambient feature space. Their \(O(\|\boldsymbol{h}\|)\) neighborhoods
within an active face have surface measure \(O(\|\boldsymbol{h}\|)\), while
the displaced-strip excess risk per unit surface measure is
\(O(\|\boldsymbol{h}\|^2)\). Their total contribution is therefore
\(O(\|\boldsymbol{h}\|^3)=o(\|\boldsymbol{h}\|^2)\). Hence the local
argument of Supplementary Section~\ref{supp:risk-proof} gives
\begin{equation}
\int_{\mathcal{K}_L}
\Delta_{\boldsymbol{h}}(\boldsymbol{y})
\,d\boldsymbol{y}
=
\frac{1}{2}
\boldsymbol{h}^{\top}
\boldsymbol{H}_{R}^{(L)}
\boldsymbol{h}
+
o_L
\left(
\|\boldsymbol{h}\|^2
\right),
\label{supp:eq:truncated-risk-expansion}
\end{equation}
where
\[
\boldsymbol{H}_{R}^{(L)}
=
\sum_{k<l}
\boldsymbol{H}_{kl}^{(L)}.
\]
For each fixed \(L\), the remainder in
\eqref{supp:eq:truncated-risk-expansion} is
\(o(\|\boldsymbol{h}\|^2)\) as
\(\boldsymbol{h}\to\boldsymbol{0}\).

By \eqref{supp:eq:boundary-integrability},
\[
\boldsymbol{H}_{R}^{(L)}
\longrightarrow
\boldsymbol{H}_{R}
=
\sum_{k<l}
\boldsymbol{H}_{kl}
\qquad
\text{as }L\to\infty.
\]
Moreover, \eqref{supp:eq:tail-condition} implies that, for every
\(\varepsilon>0\), \(L\) can be chosen sufficiently large so that
\[
\limsup_{\boldsymbol{h}\to\boldsymbol{0}}
\frac{
1
}{
\|\boldsymbol{h}\|^2
}
\int_{\mathcal{K}_L^{\,c}}
\Delta_{\boldsymbol{h}}(\boldsymbol{y})
\,d\boldsymbol{y}
<
\varepsilon.
\]
To make the limiting argument explicit, write
\[
\mathcal{R}(\boldsymbol{h})
=
R(\boldsymbol{\theta}_0+\boldsymbol{h})-R^\ast .
\]
Then
\begin{align*}
&
\frac{
\left|
\mathcal{R}(\boldsymbol{h})
-
\frac{1}{2}
\boldsymbol{h}^{\top}
\boldsymbol{H}_{R}
\boldsymbol{h}
\right|
}{
\|\boldsymbol{h}\|^2
}
\\
&\qquad\leq
\frac{
\left|
\displaystyle
\int_{\mathcal{K}_L}
\Delta_{\boldsymbol{h}}(\boldsymbol{y})
\,d\boldsymbol{y}
-
\frac{1}{2}
\boldsymbol{h}^{\top}
\boldsymbol{H}_{R}^{(L)}
\boldsymbol{h}
\right|
}{
\|\boldsymbol{h}\|^2
}
\\
&\qquad\quad+
\frac{
1
}{
\|\boldsymbol{h}\|^2
}
\int_{\mathcal{K}_L^{\,c}}
\Delta_{\boldsymbol{h}}(\boldsymbol{y})
\,d\boldsymbol{y}
+
\frac{1}{2}
\left\|
\boldsymbol{H}_{R}^{(L)}
-
\boldsymbol{H}_{R}
\right\|_{\mathrm{op}}.
\end{align*}
For each fixed \(L\), the first term tends to zero as
\(\boldsymbol{h}\to\boldsymbol{0}\) by
\eqref{supp:eq:truncated-risk-expansion}. The second term can be made
arbitrarily small by choosing \(L\) sufficiently large, by
\eqref{supp:eq:tail-condition}, while the third term tends to zero as
\(L\to\infty\) by \eqref{supp:eq:boundary-integrability}. Hence
\begin{equation}
R(\boldsymbol{\theta}_0+\boldsymbol{h})
-
R^\ast
=
\frac{1}{2}
\boldsymbol{h}^{\top}
\boldsymbol{H}_{R}
\boldsymbol{h}
+
o
\left(
\|\boldsymbol{h}\|^2
\right).
\label{supp:eq:noncompact-risk-expansion}
\end{equation}
Hence the quadratic excess-risk expansion established in Supplementary
Section~\ref{supp:risk-proof} remains valid for noncompact active Bayes boundaries under
\eqref{supp:eq:boundary-integrability} and
\eqref{supp:eq:tail-condition}.

\medskip
\noindent\textit{Equivalent log-contrast representation.}
When the prior-weighted class densities are strictly positive in a
neighborhood of an active face, define
\[
d_{kl}(\boldsymbol y;\boldsymbol\theta)
=
\log
\frac{
r_k(\boldsymbol y;\boldsymbol\theta)
}{
r_l(\boldsymbol y;\boldsymbol\theta)
}.
\]
On \(\mathcal F_{kl}\), let
\[
c_{kl}(\boldsymbol s)
=
r_k^0(\boldsymbol s)
=
r_l^0(\boldsymbol s),
\qquad
\boldsymbol a_{kl}(\boldsymbol s)
=
\nabla_{\boldsymbol\theta}
d_{kl}(\boldsymbol s;\boldsymbol\theta_0).
\]
Because
\[
g_{kl}
=
r_k-r_l,
\]
differentiation on the active face gives
\[
\boldsymbol b_{kl}(\boldsymbol s)
=
c_{kl}(\boldsymbol s)
\boldsymbol a_{kl}(\boldsymbol s)
\]
and
\[
\nabla_{\boldsymbol y}
g_{kl}(\boldsymbol s;\boldsymbol\theta_0)
=
c_{kl}(\boldsymbol s)
\nabla_{\boldsymbol y}
d_{kl}(\boldsymbol s;\boldsymbol\theta_0).
\]
Since \(c_{kl}(\boldsymbol s)>0\), substitution into the direct-contrast
representation yields
\[
\boldsymbol H_{kl}
=
\int_{\mathcal F_{kl}}
\frac{
c_{kl}(\boldsymbol s)
}{
\left\|
\nabla_{\boldsymbol y}
d_{kl}(\boldsymbol s;\boldsymbol\theta_0)
\right\|
}
\boldsymbol a_{kl}(\boldsymbol s)
\boldsymbol a_{kl}(\boldsymbol s)^\top
\,dS(\boldsymbol s).
\]
Thus the direct-contrast representation requires no positivity assumption,
whereas the equivalent log-contrast form is often more convenient for the
QDA calculations.

\medskip
\noindent\textit{Proof of Corollary~\ref{cor:asymptotic-risk}.}
Let
\[
\boldsymbol{h}_n
=
\widehat{\boldsymbol{\theta}}_n
-
\boldsymbol{\theta}_0.
\]
By Theorem~\ref{thm:multiclass-risk},
\[
R(\widehat{\boldsymbol{\theta}}_n)
-
R^\ast
=
\frac{1}{2}
\boldsymbol{h}_n^{\top}
\boldsymbol{H}_{R}
\boldsymbol{h}_n
+
o_p(n^{-1}),
\]
because
\(\boldsymbol{h}_n=O_p(n^{-1/2})\). Hence
\begin{equation}
n
\left\{
R(\widehat{\boldsymbol{\theta}}_n)
-
R^\ast
\right\}
=
\frac{1}{2}
\left(
\sqrt{n}\boldsymbol{h}_n
\right)^{\top}
\boldsymbol{H}_{R}
\left(
\sqrt{n}\boldsymbol{h}_n
\right)
+
o_p(1).
\label{supp:eq:scaled-risk-corollary}
\end{equation}

By assumption,
\[
\sqrt{n}\boldsymbol{h}_n
\overset{d}{\longrightarrow}
\boldsymbol{Z},
\qquad
\boldsymbol{Z}
\sim
N(\boldsymbol{0},\boldsymbol{V}).
\]
Since the map
\(\boldsymbol{x}\mapsto
\boldsymbol{x}^{\top}\boldsymbol{H}_{R}\boldsymbol{x}\)
is continuous, the continuous mapping theorem and
Slutsky's theorem applied to
\eqref{supp:eq:scaled-risk-corollary} give
\begin{equation}
n
\left\{
R(\widehat{\boldsymbol{\theta}}_n)
-
R^\ast
\right\}
\overset{d}{\longrightarrow}
\frac{1}{2}
\boldsymbol{Z}^{\top}
\boldsymbol{H}_{R}
\boldsymbol{Z}.
\label{supp:eq:risk-limit-corollary}
\end{equation}

If the nonzero eigenvalues of
\(\boldsymbol V^{1/2}\boldsymbol H_R\boldsymbol V^{1/2}\)
are \(\rho_1,\ldots,\rho_m\), then
\[
\frac12
\boldsymbol Z^\top
\boldsymbol H_R
\boldsymbol Z
\overset{d}{=}
\frac12
\sum_{j=1}^{m}
\rho_j\chi^2_{1,j},
\]
where the \(\chi^2_{1,j}\) variables are independent.

To justify convergence of expectations, let
\[
\boldsymbol{X}_n
=
\sqrt{n}\,
(\widehat{\boldsymbol{\theta}}_n-\boldsymbol{\theta}_0).
\]
By Theorem~\ref{thm:multiclass-risk}, there exist $\varepsilon>0$ and $C<\infty$ such that
\[
0
\leq
R(\boldsymbol{\theta}_0+\boldsymbol{h})-R^\ast
\leq
C\|\boldsymbol{h}\|^2
\]
whenever $\|\boldsymbol{h}\|\leq\varepsilon$. Hence, on the event
$\{\|\widehat{\boldsymbol{\theta}}_n-\boldsymbol{\theta}_0\|
\leq\varepsilon\}$,
\[
n
\left\{
R(\widehat{\boldsymbol{\theta}}_n)-R^\ast
\right\}
\leq
C\|\boldsymbol{X}_n\|^2.
\]
Let
\[
Y_n
=
n
\left\{
R(\widehat{\boldsymbol{\theta}}_n)-R^\ast
\right\},
\qquad
A_n
=
\left\{
\|\widehat{\boldsymbol{\theta}}_n-\boldsymbol{\theta}_0\|
\leq\varepsilon
\right\}.
\]
On \(A_n\),
\[
0\leq
Y_n\mathbf{1}_{A_n}
\leq
C\|\boldsymbol{X}_n\|^2.
\]
Since
\[
\sup_n E\|\boldsymbol{X}_n\|^{2+\delta}<\infty,
\]
the family
\(\{\|\boldsymbol{X}_n\|^2\}\) is uniformly integrable, and hence so is
\(\{Y_n\mathbf{1}_{A_n}\}\).

On \(A_n^c\), since
\[
0
\leq
R(\widehat{\boldsymbol{\theta}}_n)-R^\ast
\leq
1,
\]
we have
\[
E\left(
Y_n\mathbf{1}_{A_n^c}
\right)
\leq
n\,\Pr(A_n^c).
\]
Moreover, Markov's inequality gives
\[
n\,\Pr(A_n^c)
=
n\,
\Pr
\left\{
\|\boldsymbol{X}_n\|
>
\varepsilon\sqrt{n}
\right\}
\leq
\frac{
\sup_m E\|\boldsymbol{X}_m\|^{2+\delta}
}{
\varepsilon^{2+\delta}n^{\delta/2}
}
\longrightarrow
0.
\]
Thus
\[
Y_n\mathbf{1}_{A_n^c}
\longrightarrow
0
\qquad
\text{in }L^1.
\]
Since an \(L^1\)-convergent family is uniformly integrable,
\(\{Y_n\mathbf{1}_{A_n^c}\}\) is uniformly integrable. Consequently,
since
\[
Y_n
=
Y_n\mathbf{1}_{A_n}
+
Y_n\mathbf{1}_{A_n^c},
\]
the family \(\{Y_n\}\) is uniformly integrable.
Combining this with the convergence in distribution in
\eqref{supp:eq:risk-limit-corollary} yields
\begin{align}
n
\left[
E
\left\{
R(\widehat{\boldsymbol{\theta}}_n)
\right\}
-
R^\ast
\right]
&\longrightarrow
\frac{1}{2}
E
\left(
\boldsymbol{Z}^{\top}
\boldsymbol{H}_{R}
\boldsymbol{Z}
\right)
\nonumber\\
&=
\frac{1}{2}
\operatorname{tr}
\left(
\boldsymbol{H}_{R}\boldsymbol{V}
\right),
\label{supp:eq:expected-risk-corollary}
\end{align}
where the final equality follows from
\(E(\boldsymbol{Z}\boldsymbol{Z}^{\top})=\boldsymbol{V}\).
Equivalently,
\[
E
\left\{
R(\widehat{\boldsymbol{\theta}}_n)
\right\}
-
R^\ast
=
\frac{1}{2n}
\operatorname{tr}
\left(
\boldsymbol{H}_{R}\boldsymbol{V}
\right)
+
o(n^{-1}).
\]
This proves Corollary~\ref{cor:asymptotic-risk}.
\hfill\(\square\)

\section{Proof of the classification-weighted spectral criterion}
\label{supp:spectral-proof}

This section proves Theorem~\ref{thm:spectral}.

\medskip
\noindent\textit\noindent\textit{Proof of Theorem~\ref{thm:spectral}.}
Since
\(\boldsymbol{A}\succ\boldsymbol{0}\) and
\(\boldsymbol{J}\succ\boldsymbol{0}\), the relative information matrix
\[
\boldsymbol{C}
=
\boldsymbol{A}^{-1/2}
\boldsymbol{J}
\boldsymbol{A}^{-1/2}
\]
is symmetric positive definite. Hence
\[
\boldsymbol{J}
=
\boldsymbol{A}^{1/2}
\boldsymbol{C}
\boldsymbol{A}^{1/2},
\qquad
\boldsymbol{J}^{-1}
=
\boldsymbol{A}^{-1/2}
\boldsymbol{C}^{-1}
\boldsymbol{A}^{-1/2},
\]
and therefore
\begin{equation}
\boldsymbol{A}^{-1}
-
\boldsymbol{J}^{-1}
=
\boldsymbol{A}^{-1/2}
\left(
\boldsymbol{I}
-
\boldsymbol{C}^{-1}
\right)
\boldsymbol{A}^{-1/2}.
\label{supp:eq:AminusJ}
\end{equation}

Substituting \eqref{supp:eq:AminusJ} into
\[
\Delta_R
=
\operatorname{tr}
\left[
\boldsymbol{H}_{R}
\left(
\boldsymbol{A}^{-1}
-
\boldsymbol{J}^{-1}
\right)
\right]
\]
and using cyclic invariance of the trace gives
\begin{equation}
\Delta_R
=
\operatorname{tr}
\left[
\boldsymbol{W}
\left(
\boldsymbol{I}
-
\boldsymbol{C}^{-1}
\right)
\right],
\qquad
\boldsymbol{W}
=
\boldsymbol{A}^{-1/2}
\boldsymbol{H}_{R}
\boldsymbol{A}^{-1/2}.
\label{supp:eq:Delta-WC}
\end{equation}
Because
\(\boldsymbol{H}_{R}\succeq\boldsymbol{0}\),
the matrix \(\boldsymbol{W}\) is also positive semidefinite.

Let
\[
\boldsymbol{C}
=
\boldsymbol{Q}
\boldsymbol{\Lambda}
\boldsymbol{Q}^{\top},
\qquad
\boldsymbol{\Lambda}
=
\operatorname{diag}
(\lambda_1,\ldots,\lambda_r),
\]
where
\(\boldsymbol{Q}^{\top}\boldsymbol{Q}=\boldsymbol{I}\) and
\(\lambda_j>0\) for all \(j\). Then
\[
\boldsymbol{C}^{-1}
=
\boldsymbol{Q}
\boldsymbol{\Lambda}^{-1}
\boldsymbol{Q}^{\top}.
\]
Using \eqref{supp:eq:Delta-WC} and cyclic invariance once more,
\begin{align}
\Delta_R
&=
\operatorname{tr}
\left[
\boldsymbol{W}
\boldsymbol{Q}
\left(
\boldsymbol{I}
-
\boldsymbol{\Lambda}^{-1}
\right)
\boldsymbol{Q}^{\top}
\right]
\nonumber\\
&=
\operatorname{tr}
\left[
\boldsymbol{Q}^{\top}
\boldsymbol{W}
\boldsymbol{Q}
\left(
\boldsymbol{I}
-
\boldsymbol{\Lambda}^{-1}
\right)
\right].
\label{supp:eq:Delta-diagonal}
\end{align}
Since
\(\boldsymbol{I}-\boldsymbol{\Lambda}^{-1}\) is diagonal, only the
diagonal elements of
\(\boldsymbol{Q}^{\top}\boldsymbol{W}\boldsymbol{Q}\)
contribute to the trace. Writing
\[
w_j
=
\boldsymbol{q}_j^{\top}
\boldsymbol{W}
\boldsymbol{q}_j
\geq0,
\]
where \(\boldsymbol{q}_j\) is the \(j\)th column of
\(\boldsymbol{Q}\), we obtain
\begin{equation}
\Delta_R
=
\sum_{j=1}^{r}
w_j
\left(
1-\frac{1}{\lambda_j}
\right)
=
\sum_{j=1}^{r}
w_j
\frac{\lambda_j-1}{\lambda_j}.
\label{supp:eq:spectral-Delta-proof}
\end{equation}
This proves the spectral representation in Theorem~\ref{thm:spectral}.

Separating the terms corresponding to information gains and losses gives
\begin{equation}
\Delta_R
=
\sum_{\lambda_j>1}
w_j
\frac{\lambda_j-1}{\lambda_j}
-
\sum_{\lambda_j<1}
w_j
\frac{1-\lambda_j}{\lambda_j}.
\label{supp:eq:spectral-gain-loss}
\end{equation}
Terms for which \(\lambda_j=1\) vanish. Since every
\(w_j\geq0\), equation
\eqref{supp:eq:spectral-gain-loss} shows that
\(\Delta_R>0\) if and only if
\[
\sum_{\lambda_j>1}
w_j
\frac{\lambda_j-1}{\lambda_j}
>
\sum_{\lambda_j<1}
w_j
\frac{1-\lambda_j}{\lambda_j},
\]
which is the classification-weighted gain--loss criterion stated in
Theorem~\ref{thm:spectral}.

If an eigenvalue \(\lambda\) has multiplicity greater than one, the individual
eigenvectors within its eigenspace are not uniquely determined. The total
contribution of that eigenspace to \(\Delta_R\), however, is invariant. If
\(\boldsymbol{P}_{\lambda}\) denotes the orthogonal projector onto the
eigenspace associated with \(\lambda\), then
\[
\sum_{j:\lambda_j=\lambda}
w_j
=
\sum_{j:\lambda_j=\lambda}
\boldsymbol{q}_j^{\top}
\boldsymbol{W}
\boldsymbol{q}_j
=
\operatorname{tr}
\left(
\boldsymbol{W}\boldsymbol{P}_{\lambda}
\right).
\]
Hence the total contribution of this eigenspace is
\begin{equation}
\left(
1-\frac{1}{\lambda}
\right)
\operatorname{tr}
\left(
\boldsymbol{W}
\boldsymbol{P}_{\lambda}
\right),
\label{supp:eq:repeated-eigenvalue}
\end{equation}
which is independent of the particular orthonormal basis chosen within the
eigenspace.

This completes the proof.
\hfill\(\square\)

\section{Proof of the local departure from MCAR result}
\label{supp:local-MCAR-proof}

This section proves Proposition~\ref{prop:local-MCAR}. Throughout this section,
expectations are taken under the true data-generating parameter
\(\boldsymbol{\theta}_0\).

\medskip
\noindent\textit{Proof of Proposition~\ref{prop:local-MCAR}.}
Write
\[
U
=
U_{\boldsymbol{\theta}_0}(\boldsymbol{Y}),
\qquad
\boldsymbol{G}
=
\nabla_{\boldsymbol{\theta}}
U_{\boldsymbol{\theta}}(\boldsymbol{Y})
\big|_{\boldsymbol{\theta}=\boldsymbol{\theta}_0},
\]
and let
\[
q_t
=
\operatorname{expit}
\left\{
\alpha(t)+tU
\right\}.
\]
The calibration condition is
\begin{equation}
E(q_t)
=
\gamma,
\qquad
0<\gamma<1.
\label{supp:eq:local-calibration}
\end{equation}
To justify the local dependence of the calibrating intercept on \(t\), define
\[
F(a,t)
=
E
\left[
\operatorname{expit}\{a+tU\}
\right]
-
\gamma.
\]
At \(t=0\), the equation \(F(a,0)=0\) gives
\[
\operatorname{expit}(a)=\gamma,
\]
and hence
\begin{equation}
\alpha(0)
=
\operatorname{logit}(\gamma).
\label{supp:eq:alpha-zero}
\end{equation}
Moreover,
\[
\left.
\frac{\partial F(a,t)}{\partial a}
\right|_{(a,t)=(\alpha(0),0)}
=
\gamma(1-\gamma)
>
0.
\]
Under the stated regularity conditions, the implicit-function theorem
therefore yields a locally differentiable solution \(a=\alpha(t)\) of
\(F\{\alpha(t),t\}=0\) in a neighborhood of \(t=0\).

Differentiating the calibration identity with respect to \(t\), with
differentiation under the expectation justified by the stated regularity
conditions, gives
\[
0
=
E
\left[
q_t(1-q_t)
\left\{
\alpha'(t)+U
\right\}
\right].
\]
Evaluating at \(t=0\), where \(q_0=\gamma\), yields
\[
0
=
\gamma(1-\gamma)
\left[
\alpha'(0)+E(U)
\right],
\]
and therefore
\begin{equation}
\alpha'(0)
=
-
E(U).
\label{supp:eq:alpha-prime}
\end{equation}
A first-order Taylor expansion of \(q_t\) around \(t=0\) now gives
\begin{align}
q_t
&=
q_0
+
t
\left.
\frac{d q_t}{dt}
\right|_{t=0}
+
O(t^2)
\nonumber\\
&=
\gamma
+
t\gamma(1-\gamma)
\left\{
\alpha'(0)+U
\right\}
+
O(t^2)
\nonumber\\
&=
\gamma
+
t\gamma(1-\gamma)
\left\{
U-E(U)
\right\}
+
O(t^2).
\label{supp:eq:q-local-proof}
\end{align}
This establishes the first expansion in Proposition~\ref{prop:local-MCAR}.

For the label-information loss, write
\[
\boldsymbol{I}_{Z\mid\boldsymbol{Y}}
=
\boldsymbol{I}_{Z\mid\boldsymbol{Y}}
(\boldsymbol{\theta}_0;\boldsymbol{Y}).
\]
Along the sequence of population experiments only the missing-label mechanism
varies with \(t\), whereas the data-generating classification model remains
fixed at \(\boldsymbol{\theta}_0\). Thus
\[
\boldsymbol{D}(t)
=
E
\left[
q_t
\boldsymbol{I}_{Z\mid\boldsymbol{Y}}
\right].
\]
Under the stated regularity conditions, differentiation may be passed under
the expectation. Since
\[
\left.
\frac{d q_t}{dt}
\right|_{t=0}
=
\gamma(1-\gamma)
\left\{
U-E(U)
\right\},
\]
it follows that
\begin{equation}
\boldsymbol{D}'(0)
=
\gamma(1-\gamma)
E
\left[
\left\{
U-E(U)
\right\}
\boldsymbol{I}_{Z\mid\boldsymbol{Y}}
\right].
\label{supp:eq:Dprime-proof}
\end{equation}

We next consider the efficient information supplied by the missing-label
indicators. For likelihood inference, the intercept and uncertainty slope are
ordinary nuisance parameters. Write
\[
\boldsymbol{\xi}
=
(\xi_0,\xi_1)^{\top},
\qquad
q(\boldsymbol{y};
\boldsymbol{\theta},\boldsymbol{\xi})
=
\operatorname{expit}
\left\{
\xi_0+\xi_1
U_{\boldsymbol{\theta}}(\boldsymbol{y})
\right\}.
\]
Along the population path considered here, the true nuisance value is
\[
\boldsymbol{\xi}(t)
=
\{\alpha(t),t\}^{\top}.
\]
At
\((\boldsymbol{\theta}_0,\boldsymbol{\xi}(t))\),
define
\[
\boldsymbol{X}
=
\begin{pmatrix}
1\\
U
\end{pmatrix}.
\]
The derivatives of the missingness probability are
\begin{equation}
q_{\boldsymbol{\theta}}
=
t\,q_t(1-q_t)\boldsymbol{G},
\qquad
q_{\boldsymbol{\xi}}
=
q_t(1-q_t)\boldsymbol{X}.
\label{supp:eq:q-derivatives-local}
\end{equation}
Therefore the Bernoulli information blocks of Theorem~\ref{thm:information-decomposition} are
\begin{align}
\boldsymbol{B}_{\theta\theta}(t)
&=
t^2
E
\left[
q_t(1-q_t)
\boldsymbol{G}\boldsymbol{G}^{\top}
\right],
\label{supp:eq:Btt-local}\\
\boldsymbol{B}_{\theta\xi}(t)
&=
t
E
\left[
q_t(1-q_t)
\boldsymbol{G}\boldsymbol{X}^{\top}
\right],
\label{supp:eq:Btx-local}\\
\boldsymbol{B}_{\xi\xi}(t)
&=
E
\left[
q_t(1-q_t)
\boldsymbol{X}\boldsymbol{X}^{\top}
\right].
\label{supp:eq:Bxx-local}
\end{align}

As \(t\to0\),
\[
q_t(1-q_t)
\longrightarrow
\gamma(1-\gamma)
\qquad\text{almost surely},
\]
and
\[
0
\leq
q_t(1-q_t)
\leq
\frac{1}{4}.
\]
The assumed second-moment conditions imply
\[
E\|\boldsymbol{G}\|^2<\infty,
\qquad
E\|\boldsymbol{X}\|^2<\infty,
\]
and, by the Cauchy--Schwarz inequality,
\[
E\left(
\|\boldsymbol{G}\|\,\|\boldsymbol{X}\|
\right)
<
\infty.
\]
Dominated convergence therefore gives
\begin{align*}
E
\left[
q_t(1-q_t)
\boldsymbol{G}\boldsymbol{G}^{\top}
\right]
&=
\gamma(1-\gamma)
E
\left[
\boldsymbol{G}\boldsymbol{G}^{\top}
\right]
+
o(1),
\\
E
\left[
q_t(1-q_t)
\boldsymbol{G}\boldsymbol{X}^{\top}
\right]
&=
\gamma(1-\gamma)
E
\left[
\boldsymbol{G}\boldsymbol{X}^{\top}
\right]
+
o(1),
\\
E
\left[
q_t(1-q_t)
\boldsymbol{X}\boldsymbol{X}^{\top}
\right]
&=
\gamma(1-\gamma)
E
\left[
\boldsymbol{X}\boldsymbol{X}^{\top}
\right]
+
o(1).
\end{align*}
Consequently,
\begin{align}
\boldsymbol{B}_{\theta\theta}(t)
&=
t^2\gamma(1-\gamma)
E
\left[
\boldsymbol{G}\boldsymbol{G}^{\top}
\right]
+
o(t^2),
\label{supp:eq:Btt-expand}\\
\boldsymbol{B}_{\theta\xi}(t)
&=
t\gamma(1-\gamma)
E
\left[
\boldsymbol{G}\boldsymbol{X}^{\top}
\right]
+
o(t),
\label{supp:eq:Btx-expand}\\
\boldsymbol{B}_{\xi\xi}(t)
&=
\gamma(1-\gamma)
E
\left[
\boldsymbol{X}\boldsymbol{X}^{\top}
\right]
+
o(1).
\label{supp:eq:Bxx-expand}
\end{align}
The assumption
\(\operatorname{Var}(U)>0\) implies that
\(E(\boldsymbol{X}\boldsymbol{X}^{\top})\) is positive definite. Indeed, for
\(\boldsymbol{a}=(a_0,a_1)^{\top}\neq\boldsymbol{0}\),
\[
\boldsymbol{a}^{\top}
E
\left[
\boldsymbol{X}\boldsymbol{X}^{\top}
\right]
\boldsymbol{a}
=
E
\left[
(a_0+a_1U)^2
\right]
>0,
\]
because a nonzero affine function of a nondegenerate random variable cannot
vanish almost surely. Hence
\(\boldsymbol{B}_{\xi\xi}(t)\) is nonsingular for all sufficiently small
\(t\), and
\begin{equation}
\boldsymbol{B}_{\xi\xi}(t)^{-1}
=
\frac{1}{\gamma(1-\gamma)}
E
\left[
\boldsymbol{X}\boldsymbol{X}^{\top}
\right]^{-1}
+
o(1).
\label{supp:eq:Bxx-inverse}
\end{equation}

Using the efficient-information formula
\[
\boldsymbol{I}_{M}^{\mathrm{eff}}(t)
=
\boldsymbol{B}_{\theta\theta}(t)
-
\boldsymbol{B}_{\theta\xi}(t)
\boldsymbol{B}_{\xi\xi}(t)^{-1}
\boldsymbol{B}_{\xi\theta}(t),
\]
and substituting
\eqref{supp:eq:Btt-expand}--\eqref{supp:eq:Bxx-inverse}, we obtain
\begin{equation}
\boldsymbol{I}_{M}^{\mathrm{eff}}(t)
=
t^2\gamma(1-\gamma)
\boldsymbol{\mathcal V}_{U}
+
o(t^2),
\label{supp:eq:IM-local-proof}
\end{equation}
where
\begin{align}
\boldsymbol{\mathcal V}_{U}
={}&
E
\left[
\boldsymbol{G}\boldsymbol{G}^{\top}
\right]
-
E
\left[
\boldsymbol{G}\boldsymbol{X}^{\top}
\right]
E
\left[
\boldsymbol{X}\boldsymbol{X}^{\top}
\right]^{-1}
E
\left[
\boldsymbol{X}\boldsymbol{G}^{\top}
\right].
\label{supp:eq:VU-proof}
\end{align}

To verify that
\(\boldsymbol{\mathcal V}_{U}\succeq\boldsymbol{0}\), define
\[
\boldsymbol{C}
=
E
\left[
\boldsymbol{G}\boldsymbol{X}^{\top}
\right]
E
\left[
\boldsymbol{X}\boldsymbol{X}^{\top}
\right]^{-1}.
\]
Then
\begin{align}
E
\left[
(\boldsymbol{G}-\boldsymbol{C}\boldsymbol{X})
(\boldsymbol{G}-\boldsymbol{C}\boldsymbol{X})^{\top}
\right]
=
\boldsymbol{\mathcal V}_{U}.
\label{supp:eq:VU-residual}
\end{align}
The left-hand side is a second-moment matrix and is therefore positive
semidefinite. Hence
\[
\boldsymbol{\mathcal V}_{U}
\succeq
\boldsymbol{0}.
\]

It remains to obtain the derivative of the classification advantage. Let
\[
\boldsymbol{J}(t)
=
\boldsymbol{A}
-
\boldsymbol{D}(t)
+
\boldsymbol{I}_{M}^{\mathrm{eff}}(t)
\]
denote the efficient information under the partially classified experiment.
At \(t=0\),
\[
\boldsymbol{J}(0)
=
\boldsymbol{A}
-
\gamma
E
\left[
\boldsymbol{I}_{Z\mid\boldsymbol{Y}}
\right]
=
\boldsymbol{J}_0.
\]
Since
\(\boldsymbol{I}_{M}^{\mathrm{eff}}(t)=O(t^2)\),
\[
\left.
\frac{d}{dt}
\boldsymbol{I}_{M}^{\mathrm{eff}}(t)
\right|_{t=0}
=
\boldsymbol{0},
\]
and therefore
\begin{equation}
\boldsymbol{J}'(0)
=
-
\boldsymbol{D}'(0).
\label{supp:eq:Jprime-local}
\end{equation}

Now
\[
\Delta_R(t)
=
\operatorname{tr}
\left[
\boldsymbol{H}_{R}
\left\{
\boldsymbol{A}^{-1}
-
\boldsymbol{J}(t)^{-1}
\right\}
\right].
\]
Using the matrix derivative identity
\[
\frac{d}{dt}
\boldsymbol{J}(t)^{-1}
=
-
\boldsymbol{J}(t)^{-1}
\boldsymbol{J}'(t)
\boldsymbol{J}(t)^{-1},
\]
we obtain
\begin{align}
\Delta_R'(0)
&=
\operatorname{tr}
\left[
\boldsymbol{H}_{R}
\boldsymbol{J}_0^{-1}
\boldsymbol{J}'(0)
\boldsymbol{J}_0^{-1}
\right]
\nonumber\\
&=
-
\operatorname{tr}
\left[
\boldsymbol{H}_{R}
\boldsymbol{J}_0^{-1}
\boldsymbol{D}'(0)
\boldsymbol{J}_0^{-1}
\right],
\label{supp:eq:Delta-prime-proof}
\end{align}
where the second equality follows from
\eqref{supp:eq:Jprime-local}. This is the final assertion of the proposition.

Thus the label-information loss may change at first order in \(t\), whereas
the efficient information contributed by the missing-label mechanism begins
only at order \(t^2\). This completes the proof of Proposition~\ref{prop:local-MCAR}.
\hfill\(\square\)

\section{Quadratic discriminant calculations}
\label{supp:qda-calculations}

This section gives the score, Fisher-information, and uncertainty derivatives
used for the three-class QDA illustration in Section~\ref{sec:qda}. The parameterization is
\[
\boldsymbol{\theta}
=
\left(
\boldsymbol{\alpha}^{\top},
\boldsymbol{\mu}_1^{\top},
\boldsymbol{\mu}_2^{\top},
\boldsymbol{\mu}_3^{\top},
\boldsymbol{\sigma}_1^{\top},
\boldsymbol{\sigma}_2^{\top},
\boldsymbol{\sigma}_3^{\top}
\right)^{\top},
\]
where
\(\boldsymbol{\alpha}=(\alpha_1,\alpha_2)^{\top}\) parameterizes the class
probabilities relative to class 3 and
\(\boldsymbol{\sigma}_k=\operatorname{vech}(\boldsymbol{\Sigma}_k)\).

\medskip
\noindent\textit{Class probabilities and QDA discriminants.}
Using class 3 as the baseline, let
\[
\alpha_1=\log\frac{\pi_1}{\pi_3},
\qquad
\alpha_2=\log\frac{\pi_2}{\pi_3}.
\]
Then, with
\(D_\alpha=1+\exp(\alpha_1)+\exp(\alpha_2)\),
\[
\pi_1=\frac{\exp(\alpha_1)}{D_\alpha},
\qquad
\pi_2=\frac{\exp(\alpha_2)}{D_\alpha},
\qquad
\pi_3=\frac{1}{D_\alpha}.
\]
If
\(\boldsymbol{e}_1=(1,0)^{\top}\),
\(\boldsymbol{e}_2=(0,1)^{\top}\), and
\(\boldsymbol{e}_3=(0,0)^{\top}\), then
\begin{equation}
\nabla_{\boldsymbol{\alpha}}
\log\frac{\pi_k}{\pi_l}
=
\boldsymbol{e}_k-\boldsymbol{e}_l.
\label{supp:eq:prior-ratio-gradient}
\end{equation}
The complete-classification score and information for
\(\boldsymbol{\alpha}\) are
\[
\boldsymbol{S}_{\alpha}
=
\begin{pmatrix}
\mathbb{I}(Z=1)-\pi_1\\
\mathbb{I}(Z=2)-\pi_2
\end{pmatrix},
\qquad
\boldsymbol{I}_{\alpha}
=
\begin{pmatrix}
\pi_1(1-\pi_1) & -\pi_1\pi_2\\
-\pi_1\pi_2 & \pi_2(1-\pi_2)
\end{pmatrix}.
\]

For
\(\boldsymbol{x}_k=\boldsymbol{y}-\boldsymbol{\mu}_k\), the pairwise QDA
discriminant is
\begin{align}
d_{kl}(\boldsymbol{y})
={}&
\log\frac{\pi_k}{\pi_l}
-
\frac{1}{2}
\log\frac{|\boldsymbol{\Sigma}_k|}
{|\boldsymbol{\Sigma}_l|}
-
\frac{1}{2}
\boldsymbol{x}_k^{\top}
\boldsymbol{\Sigma}_k^{-1}
\boldsymbol{x}_k
\nonumber\\
&+
\frac{1}{2}
\boldsymbol{x}_l^{\top}
\boldsymbol{\Sigma}_l^{-1}
\boldsymbol{x}_l.
\label{supp:eq:qda-discriminant}
\end{align}
Its spatial derivative is
\begin{equation}
\nabla_{\boldsymbol{y}}d_{kl}(\boldsymbol{y})
=
-\boldsymbol{\Sigma}_k^{-1}\boldsymbol{x}_k
+
\boldsymbol{\Sigma}_l^{-1}\boldsymbol{x}_l,
\label{supp:eq:qda-spatial-gradient}
\end{equation}
while the derivatives with respect to the class means are
\[
\nabla_{\boldsymbol{\mu}_k}d_{kl}
=
\boldsymbol{\Sigma}_k^{-1}\boldsymbol{x}_k,
\qquad
\nabla_{\boldsymbol{\mu}_l}d_{kl}
=
-\boldsymbol{\Sigma}_l^{-1}\boldsymbol{x}_l.
\]
All mean blocks corresponding to classes other than \(k\) and \(l\) are zero.

For the covariance terms, the standard identities
\[
d\log|\boldsymbol{\Sigma}|
=
\operatorname{tr}
\left(
\boldsymbol{\Sigma}^{-1}d\boldsymbol{\Sigma}
\right),
\qquad
d\boldsymbol{\Sigma}^{-1}
=
-\boldsymbol{\Sigma}^{-1}
(d\boldsymbol{\Sigma})
\boldsymbol{\Sigma}^{-1}
\]
give, under symmetric covariance perturbations,
\begin{align}
\frac{\partial d_{kl}}{\partial\boldsymbol{\Sigma}_k}
&=
\frac{1}{2}
\left[
\boldsymbol{\Sigma}_k^{-1}
\boldsymbol{x}_k\boldsymbol{x}_k^{\top}
\boldsymbol{\Sigma}_k^{-1}
-
\boldsymbol{\Sigma}_k^{-1}
\right],
\label{supp:eq:qda-Sigma-k}\\
\frac{\partial d_{kl}}{\partial\boldsymbol{\Sigma}_l}
&=
-\frac{1}{2}
\left[
\boldsymbol{\Sigma}_l^{-1}
\boldsymbol{x}_l\boldsymbol{x}_l^{\top}
\boldsymbol{\Sigma}_l^{-1}
-
\boldsymbol{\Sigma}_l^{-1}
\right].
\label{supp:eq:qda-Sigma-l}
\end{align}
If
\(\boldsymbol{\sigma}_k=\operatorname{vech}(\boldsymbol{\Sigma}_k)\)
and \(\boldsymbol{D}_p\) denotes the duplication matrix, then
\begin{equation}
\nabla_{\boldsymbol{\sigma}_k}d_{kl}
=
\boldsymbol{D}_p^{\top}
\operatorname{vec}
\left(
\frac{\partial d_{kl}}
{\partial\boldsymbol{\Sigma}_k}
\right),
\label{supp:eq:qda-vech-gradient}
\end{equation}
with the analogous expression for class \(l\). Together with
\eqref{supp:eq:prior-ratio-gradient}, these derivatives form
\[
\boldsymbol{a}_{kl}(\boldsymbol{y})
=
\nabla_{\boldsymbol{\theta}}
d_{kl}(\boldsymbol{y}),
\]
the log-contrast boundary-sensitivity vector. On an active face
\(\mathcal{F}_{kl}\), where
\[
c_{kl}(\boldsymbol{s})
=
r_k(\boldsymbol{s})
=
r_l(\boldsymbol{s}),
\]
the relation between the density contrast
\(g_{kl}=r_k-r_l\) and the log contrast \(d_{kl}\) gives
\[
\nabla_{\boldsymbol{\theta}}g_{kl}(\boldsymbol{s})
=
c_{kl}(\boldsymbol{s})
\boldsymbol{a}_{kl}(\boldsymbol{s}),
\qquad
\nabla_{\boldsymbol{y}}g_{kl}(\boldsymbol{s})
=
c_{kl}(\boldsymbol{s})
\nabla_{\boldsymbol{y}}d_{kl}(\boldsymbol{s}).
\]
Hence the corresponding surface-integral contribution is
\[
\boldsymbol{H}_{kl}
=
\int_{\mathcal{F}_{kl}}
\frac{
c_{kl}(\boldsymbol{s})
}{
\left\|
\nabla_{\boldsymbol{y}}
d_{kl}(\boldsymbol{s})
\right\|
}
\boldsymbol{a}_{kl}(\boldsymbol{s})
\boldsymbol{a}_{kl}(\boldsymbol{s})^{\top}
\,dS(\boldsymbol{s}).
\]

\medskip
\noindent\textit{Complete-classification Fisher information.}
For class \(k\), the contribution to the mean score is
\[
\boldsymbol{S}_{\mu_k}
=
\mathbb{I}(Z=k)
\boldsymbol{\Sigma}_k^{-1}
(\boldsymbol{Y}-\boldsymbol{\mu}_k),
\]
and therefore
\begin{equation}
\boldsymbol{I}_{\mu_k}
=
E
\left[
\boldsymbol{S}_{\mu_k}
\boldsymbol{S}_{\mu_k}^{\top}
\right]
=
\pi_k\boldsymbol{\Sigma}_k^{-1}.
\label{supp:eq:qda-mean-information}
\end{equation}
For
\(\boldsymbol{\sigma}_k=\operatorname{vech}(\boldsymbol{\Sigma}_k)\),
the covariance-information block is
\begin{equation}
\boldsymbol{I}_{\sigma_k}
=
\frac{\pi_k}{2}
\boldsymbol{D}_p^{\top}
\left(
\boldsymbol{\Sigma}_k^{-1}
\otimes
\boldsymbol{\Sigma}_k^{-1}
\right)
\boldsymbol{D}_p.
\label{supp:eq:qda-cov-information}
\end{equation}
The mean and covariance scores within a Gaussian class are orthogonal because
the relevant centered third moments vanish. Scores for distributional
parameters belonging to different classes are orthogonal because their class
indicators are mutually exclusive. The prior score is also orthogonal to the
within-class mean and covariance scores, since each latter score has
conditional mean zero given \(Z\). Consequently, in the parameterization above,
\begin{equation}
\boldsymbol{I}_{\mathrm{CC}}
=
\operatorname{blockdiag}
\left(
\boldsymbol{I}_{\alpha},
\boldsymbol{I}_{\mu_1},
\boldsymbol{I}_{\mu_2},
\boldsymbol{I}_{\mu_3},
\boldsymbol{I}_{\sigma_1},
\boldsymbol{I}_{\sigma_2},
\boldsymbol{I}_{\sigma_3}
\right).
\label{supp:eq:qda-ICC}
\end{equation}

\medskip
\noindent\textit{Posterior probabilities and uncertainty derivatives.}
Let
\[
r_k(\boldsymbol{y};\boldsymbol{\theta})
=
\pi_k f_k(\boldsymbol{y}),
\qquad
\tau_k(\boldsymbol{y})
=
\frac{r_k(\boldsymbol{y})}
{\sum_{j=1}^{3}r_j(\boldsymbol{y})},
\]
and define
\[
\boldsymbol{s}_k(\boldsymbol{y})
=
\nabla_{\boldsymbol{\theta}}
\log r_k(\boldsymbol{y}),
\qquad
\overline{\boldsymbol{s}}(\boldsymbol{y})
=
\sum_{j=1}^{3}
\tau_j(\boldsymbol{y})
\boldsymbol{s}_j(\boldsymbol{y}).
\]
Differentiating the normalized posterior probability gives
\begin{equation}
\nabla_{\boldsymbol{\theta}}
\tau_k(\boldsymbol{y})
=
\tau_k(\boldsymbol{y})
\left\{
\boldsymbol{s}_k(\boldsymbol{y})
-
\overline{\boldsymbol{s}}(\boldsymbol{y})
\right\}.
\label{supp:eq:posterior-gradient}
\end{equation}

For Shannon entropy
\[
H(\boldsymbol{y})
=
-\sum_{k=1}^{3}
\tau_k(\boldsymbol{y})
\log\tau_k(\boldsymbol{y}),
\]
the identity
\(\sum_k\nabla_{\boldsymbol{\theta}}\tau_k(\boldsymbol{y})
=\boldsymbol{0}\)
and \eqref{supp:eq:posterior-gradient} yield
\begin{equation}
\nabla_{\boldsymbol{\theta}}H(\boldsymbol{y})
=
-
\sum_{k=1}^{3}
\tau_k(\boldsymbol{y})
\log\tau_k(\boldsymbol{y})
\left\{
\boldsymbol{s}_k(\boldsymbol{y})
-
\overline{\boldsymbol{s}}(\boldsymbol{y})
\right\}.
\label{supp:eq:entropy-gradient}
\end{equation}
Thus, for the normalized entropy uncertainty
\(U_H=H/\log 3\),
\begin{equation}
\nabla_{\boldsymbol{\theta}}U_H(\boldsymbol{y})
=
\frac{1}{\log 3}
\nabla_{\boldsymbol{\theta}}H(\boldsymbol{y}).
\label{supp:eq:normalized-entropy-gradient}
\end{equation}

For Gini uncertainty
\[
G(\boldsymbol{y})
=
1-
\sum_{k=1}^{3}
\tau_k(\boldsymbol{y})^2,
\]
we similarly obtain
\begin{align}
\nabla_{\boldsymbol{\theta}}G(\boldsymbol{y})
&=
-2
\sum_{k=1}^{3}
\tau_k(\boldsymbol{y})
\nabla_{\boldsymbol{\theta}}\tau_k(\boldsymbol{y})
\nonumber\\
&=
-2
\sum_{k=1}^{3}
\tau_k(\boldsymbol{y})^2
\left\{
\boldsymbol{s}_k(\boldsymbol{y})
-
\overline{\boldsymbol{s}}(\boldsymbol{y})
\right\}.
\label{supp:eq:gini-gradient}
\end{align}
Since \(U_G=(3/2)G\) for three classes,
\begin{equation}
\nabla_{\boldsymbol{\theta}}U_G(\boldsymbol{y})
=
-3
\sum_{k=1}^{3}
\tau_k(\boldsymbol{y})^2
\left\{
\boldsymbol{s}_k(\boldsymbol{y})
-
\overline{\boldsymbol{s}}(\boldsymbol{y})
\right\}.
\label{supp:eq:normalized-gini-gradient}
\end{equation}

Finally, for the logistic missing-label mechanism
\[
q(\boldsymbol{y};\boldsymbol{\theta},\boldsymbol{\xi})
=
\operatorname{expit}
\left\{
\xi_0+\xi_1
U_{\boldsymbol{\theta}}(\boldsymbol{y})
\right\},
\qquad
\boldsymbol{\xi}
=
(\xi_0,\xi_1)^{\top},
\]
the required derivatives are
\begin{align}
q_{\boldsymbol{\theta}}
&=
q(1-q)\,
\xi_1
\nabla_{\boldsymbol{\theta}}
U_{\boldsymbol{\theta}}(\boldsymbol{y}),
\label{supp:eq:q-theta-QDA}\\
q_{\boldsymbol{\xi}}
&=
q(1-q)
\begin{pmatrix}
1\\
U_{\boldsymbol{\theta}}(\boldsymbol{y})
\end{pmatrix}.
\label{supp:eq:q-xi-QDA}
\end{align}
Substitution of
\eqref{supp:eq:q-theta-QDA}--\eqref{supp:eq:q-xi-QDA}
into the Bernoulli information formula of
Theorem~\ref{thm:information-decomposition} gives
\begin{align}
\boldsymbol{B}_{\theta\theta}
&=
E
\left[
q(1-q)\xi_1^2
\,
\nabla_{\boldsymbol{\theta}}U_{\boldsymbol{\theta}}(\boldsymbol{Y})
\nabla_{\boldsymbol{\theta}}U_{\boldsymbol{\theta}}(\boldsymbol{Y})^{\top}
\right],
\label{supp:eq:Btt-QDA}\\
\boldsymbol{B}_{\theta\xi}
&=
E
\left[
q(1-q)\xi_1
\,
\nabla_{\boldsymbol{\theta}}U_{\boldsymbol{\theta}}(\boldsymbol{Y})
\begin{pmatrix}
1 &
U_{\boldsymbol{\theta}}(\boldsymbol{Y})
\end{pmatrix}
\right],
\label{supp:eq:Btx-QDA}\\
\boldsymbol{B}_{\xi\xi}
&=
E
\left[
q(1-q)
\begin{pmatrix}
1\\
U_{\boldsymbol{\theta}}(\boldsymbol{Y})
\end{pmatrix}
\begin{pmatrix}
1 &
U_{\boldsymbol{\theta}}(\boldsymbol{Y})
\end{pmatrix}
\right].
\label{supp:eq:Bxx-QDA}
\end{align}
These expressions, together with the conditional label-information loss,
provide the matrices used to evaluate
\(\boldsymbol{I}_{M}^{\mathrm{eff}}\) and
\(\boldsymbol{I}_{\mathrm{PC}}^{\mathrm{eff}}\)
in the QDA experiments.

\section{Additional population robustness results}
\label{supp:population-robustness}

This section reports additional population calculations complementing the
numerical investigation in Section~\ref{sec:numerical}. The results examine
the stability of the classification comparison under covariance
heterogeneity, its decomposition across active Bayes faces under prior
imbalance, the dependence of the phase boundary on the classification
geometry, and the robustness of that phase structure to the choice of
posterior-uncertainty functional.

\medskip
\noindent\textit{Covariance heterogeneity.}
Let
\[
\overline{\boldsymbol{\Sigma}}
=
\sum_{k=1}^{3}
\pi_k\boldsymbol{\Sigma}_k^{(0)}
\]
and consider the path
\[
\boldsymbol{\Sigma}_k(\rho)
=
(1-\rho)\overline{\boldsymbol{\Sigma}}
+
\rho\boldsymbol{\Sigma}_k^{(0)},
\qquad
0\leq \rho\leq 1.
\]
The model therefore varies continuously from common-covariance linear
discrimination at \(\rho=0\) to the reference QDA configuration at
\(\rho=1\). Table~\ref{supp:tab:covariance-heterogeneity} reports the corresponding population risk coefficients.

\begin{table}[t]
\centering
\caption{Classification efficiency along the covariance-heterogeneity path
\(\boldsymbol{\Sigma}_k(\rho)\).}
\label{supp:tab:covariance-heterogeneity}
\begin{tabular}{ccccc}
\toprule
\(\rho\) &
\(\mathcal E_{\mathrm{CC}}\) &
\(\mathcal E_{\mathrm{IG}}\) &
\(\mathcal E_{\mathrm{PC}}\) &
\(\operatorname{ARE}_R\)\\
\midrule
0.00 & 1.652 & 2.723 & 1.425 & 1.160\\
0.20 & 1.651 & 2.720 & 1.425 & 1.159\\
0.40 & 1.644 & 2.712 & 1.422 & 1.156\\
0.60 & 1.627 & 2.688 & 1.412 & 1.152\\
0.80 & 1.599 & 2.648 & 1.395 & 1.147\\
1.00 & 1.558 & 2.585 & 1.365 & 1.142\\
\bottomrule
\end{tabular}
\end{table}

Informative partial classification remains favourable throughout this path.
The relative efficiency decreases only modestly, from \(1.160\) under common
covariance matrices to \(1.142\) in the reference QDA model. Thus, within
this controlled family, increasing covariance heterogeneity weakens the
magnitude of the advantage slightly but does not change its sign. The
substantially larger values of \(\mathcal E_{\mathrm{IG}}\) also show that
ignoring the informative missingness mechanism performs markedly worse
throughout the path.

\medskip
\noindent\textit{Face-specific efficiency under prior imbalance.}
The global classification comparison may conceal different behavior across
the active pairwise boundaries. For an active pair \(k<l\), define the
face-specific asymptotic relative efficiency by
\[
\operatorname{ARE}_{kl}
=
\frac{
\operatorname{tr}
\left(
\boldsymbol{H}_{kl}\boldsymbol{A}^{-1}
\right)
}{
\operatorname{tr}
\left(
\boldsymbol{H}_{kl}\boldsymbol{J}^{-1}
\right)
},
\]
provided the denominator is positive. Thus
\(\operatorname{ARE}_{kl}>1\) means that informative partial classification
has the smaller leading excess-risk contribution associated with the
\(k\)-versus-\(l\) active Bayes face.

In the rare-class configuration \(\pi_3=0.05\), the approximate
face-specific relative efficiencies are
\[
\operatorname{ARE}_{12}=1.040,
\qquad
\operatorname{ARE}_{13}=0.968,
\qquad
\operatorname{ARE}_{23}=0.856.
\]
Thus the informative mechanism remains favourable for the boundary separating
the two common classes but loses efficiency along both faces involving the
rare third class. The combined effect yields
\(\operatorname{ARE}_R<1\).

For the reference prior \(\pi_3=0.30\), the corresponding values are
\[
\operatorname{ARE}_{12}=1.134,
\qquad
\operatorname{ARE}_{13}=1.139,
\qquad
\operatorname{ARE}_{23}=1.151,
\]
so all three active faces contribute favourably to the global result. At the
opposite extreme, with \(\pi_3=0.90\),
\[
\operatorname{ARE}_{12}=0.931,
\qquad
\operatorname{ARE}_{13}=0.938,
\qquad
\operatorname{ARE}_{23}=1.051.
\]
The global deterioration under strong prior imbalance therefore again reflects
heterogeneous face-specific contributions: only the \(2\)-versus-\(3\)
boundary remains favourable.

\medskip
\noindent\textit{Geometry-dependent phase boundaries.}
The critical normalized entropy slope also varies with the underlying Bayes
geometry. Table~\ref{supp:tab:separation-critical} reports the thresholds for selected class-separation
values, while Table~\ref{supp:tab:prior-critical} gives the corresponding thresholds for selected
class-3 prior probabilities.

\begin{table}[t]
\centering
\caption{Critical normalized entropy slopes under selected
class-separation levels.}
\label{supp:tab:separation-critical}
\begin{tabular}{cccc}
\toprule
\(s\) &
\(\gamma=0.10\) &
\(\gamma=0.30\) &
\(\gamma=0.50\)\\
\midrule
0.75 & 3.04 & 3.57 & 4.62\\
1.00 & 2.90 & 3.52 & 4.87\\
1.30 & 2.90 & 3.67 & 5.38\\
\bottomrule
\end{tabular}
\end{table}

\begin{table}[t]
\centering
\caption{Critical normalized entropy slopes under selected class-3 prior
probabilities.}
\label{supp:tab:prior-critical}
\begin{tabular}{cccc}
\toprule
\(\pi_3\) &
\(\gamma=0.10\) &
\(\gamma=0.30\) &
\(\gamma=0.50\)\\
\midrule
0.10 & 3.07 & 4.07 & no crossing observed\\
0.30 & 2.90 & 3.52 & 4.87\\
0.80 & 3.05 & 3.95 & 6.54\\
\bottomrule
\end{tabular}
\end{table}

These thresholds confirm that the existence and location of a favourable
regime depend jointly on the missing-label proportion and the classification
geometry. In particular, for \(\pi_3=0.10\) and \(\gamma=0.50\), no crossing
was observed over the mechanism range examined. The classification relative
efficiency increased toward one at moderate uncertainty slopes but remained
below one and subsequently declined as the uncertainty dependence became
stronger. Thus stronger informativeness alone does not guarantee a favourable
classification regime.

\medskip
\noindent\textit{Entropy versus Gini uncertainty.}
To assess whether the phase-transition behavior is specific to normalized
Shannon entropy, we repeated the reference population calculation using the
normalized Gini uncertainty defined in Section~\ref{sec:qda}. Both
uncertainty functionals take values in \([0,1]\), so their logistic slopes
operate on the same normalized range. For each value of \(\gamma\), the
mechanism intercept was recalibrated to preserve the specified marginal
missing-label proportion.

\begin{table}[t]
\centering
\caption{Critical normalized uncertainty slopes for entropy- and
Gini-dependent missingness in the reference QDA configuration.}
\label{supp:tab:entropy-gini-critical}
\begin{tabular}{ccc}
\toprule
\(\gamma\) &
Entropy \(t_H^\star\) &
Gini \(t_G^\star\)\\
\midrule
0.10 & 2.90 & 2.35\\
0.30 & 3.52 & 2.84\\
0.50 & 4.87 & 3.88\\
\bottomrule
\end{tabular}
\end{table}

The two uncertainty measures produce the same qualitative phase structure:
the critical slope increases with the marginal missing-label proportion under
both mechanisms. For the reference QDA configuration, the Gini-based
mechanism reaches the favourable region at a smaller normalized slope than
the entropy-based mechanism for each value of \(\gamma\) examined. Because
the two uncertainty functionals have different shapes even after
normalization, these numerical thresholds should not be interpreted as a
general efficiency ordering between entropy and Gini uncertainty. Their role
here is instead to show that the transition between unfavourable and
favourable informative missingness is not specific to the entropy
specification.
\section{Additional computational details for the finite-sample study}
\label{supp:finite-computation}

This section gives implementation details and numerical diagnostics for the
finite-sample experiment reported in Section~\ref{sec:finite-sample}.

\medskip
\noindent\textit{Parameterization, optimization, and convergence.}
For numerical optimization, each covariance matrix was represented through a
lower log-Cholesky factor,
\[
\boldsymbol{\Sigma}_k
=
\boldsymbol{L}_k\boldsymbol{L}_k^{\top},
\qquad
\boldsymbol{L}_k
=
\begin{pmatrix}
\exp(\ell_{k1}) & 0\\
\ell_{k2} & \exp(\ell_{k3})
\end{pmatrix},
\]
which guarantees positive definiteness throughout the optimization.

The complete-classification estimator was obtained in closed form. The MCAR
and IPC estimators were obtained by numerical maximization of their respective
observed-data likelihoods. Three starting values were used for each numerical
fit in the final Monte Carlo experiment. For IPC, the optimization vector
contains the 17 QDA parameters together with the two missingness parameters
\((\xi_0,t_H)\).

A numerical fit was retained only when the optimizer reported successful
termination, the estimated parameter vector was not numerically located at an
optimization bound, and the relative finite-difference gradient was below the
prespecified convergence tolerance. The resulting numbers of usable fits were
\[
\begin{array}{c|ccc}
\toprule
n & \mathrm{CC} & \mathrm{MCAR} & \mathrm{IPC}\\
\midrule
250  & 500 & 500 & 463\\
500  & 500 & 500 & 485\\
1000 & 500 & 499 & 491\\
\bottomrule
\end{array}
\]
corresponding to IPC convergence rates of \(92.6\%\), \(97.0\%\), and
\(98.2\%\), respectively.

\medskip
\noindent\textit{Population quantities and parameterization checks.}
The information matrices entering the asymptotic comparison were evaluated
under the true generating model. The active-face curvature matrix
\(\boldsymbol{H}_R\) was obtained by numerical contour integration over the
three active pairwise Bayes boundaries.

In the log-Cholesky coordinates used for numerical optimization,
\[
\operatorname{tr}(\boldsymbol{H}_R)
=
0.7227935085,
\]
whereas in the
\((\boldsymbol{\alpha},\boldsymbol{\mu},
\operatorname{vech}(\boldsymbol{\Sigma}))\)
coordinates used in the theoretical presentation,
\[
\operatorname{tr}(\boldsymbol{H}_R)
=
0.6035616902.
\]
The difference is expected because the trace of
\(\boldsymbol{H}_R\) itself is parameterization dependent. By contrast, the
classification-risk coefficients
\(\operatorname{tr}(\boldsymbol{H}_R\boldsymbol{V}_m)\)
are invariant under the corresponding smooth coordinate transformation.

For the reference configuration, the resulting asymptotic coefficients are
\[
\mathcal{K}_{\mathrm{CC}}
=
1.5581470885,
\qquad
\mathcal{K}_{\mathrm{MCAR}}
=
2.0540758514,
\qquad
\mathcal{K}_{\mathrm{IPC}}
=
1.3645461995.
\]
Hence
\[
\operatorname{ARE}_{R,\mathrm{IPC:CC}}
=
1.1418793215,
\qquad
\operatorname{ARE}_{R,\mathrm{MCAR:CC}}
=
0.7585635591.
\]

\medskip
\noindent\textit{Numerical evaluation of population excess risk.}
For a fitted parameter vector
\(\widehat{\boldsymbol{\theta}}\), excess risk was evaluated as
\[
\mathcal{X}(\widehat{\boldsymbol{\theta}})
=
E_{\boldsymbol{\theta}_0}
\left[
\max_k\tau_{0k}(\boldsymbol{Y})
-
\tau_{0,\widehat{C}(\boldsymbol{Y})}(\boldsymbol{Y})
\right].
\]
The expectation was approximated with a common stratified population sample of
\(600{,}000\) observations, consisting of \(200{,}000\) draws conditionally
from each class. Class-specific averages were then weighted by the true class
probabilities. The same population sample was used for every converged fitted
classifier, thereby reducing numerical integration noise in comparisons across
methods and Monte Carlo replications.

The integration variability was small relative to the sampling variability of
the fitted classifiers. At \(n=1000\), the scaled integration Monte Carlo
standard errors were approximately
\[
0.0082,\qquad
0.0099,\qquad
0.0073
\]
for CC, MCAR, and IPC, respectively, compared with replication Monte Carlo
standard errors
\[
0.0424,\qquad
0.0559,\qquad
0.0398.
\]
Thus the uncertainty in the reported finite-sample comparisons is dominated by
variation across fitted training samples rather than by the numerical
population-risk calculation.

\medskip
\noindent\textit{Quadratic risk and covariance validation.}
Since the quadratic expansion gives
\[
R(\widehat{\boldsymbol{\theta}})-R^\ast
=
\frac{1}{2}
(\widehat{\boldsymbol{\theta}}-\boldsymbol{\theta}_0)^{\top}
\boldsymbol{H}_R
(\widehat{\boldsymbol{\theta}}-\boldsymbol{\theta}_0)
+
o_p(n^{-1}),
\]
the directly evaluated excess risk was compared on the scale
\[
2n\,\mathcal{X}(\widehat{\boldsymbol{\theta}}),
\]
for which the asymptotic coefficient is
\(\operatorname{tr}(\boldsymbol{H}_R\boldsymbol{V}_m)\).
Correspondingly, the local quadratic approximation was checked by evaluating
\[
n
(\widehat{\boldsymbol{\theta}}-\boldsymbol{\theta}_0)^{\top}
\boldsymbol{H}_R
(\widehat{\boldsymbol{\theta}}-\boldsymbol{\theta}_0)
\]
for each usable fit. At \(n=1000\), the Monte Carlo averages were
\[
1.5341,\qquad
2.0407,\qquad
1.3735
\]
for CC, MCAR, and IPC, respectively. The corresponding directly evaluated values of
\(2n\,\mathcal{X}(\widehat{\boldsymbol{\theta}})\) were
\[
1.5247,\qquad
2.0294,\qquad
1.3682,
\]
while the theoretical coefficients were
\[
1.5581,\qquad
2.0541,\qquad
1.3645.
\]
The agreement among these three calculations provides a direct numerical check
of the quadratic risk approximation and of its information-based asymptotic
limit.

For the covariance calculation, define
\[
\widehat{\boldsymbol{V}}_{m,n}
=
n\,
\widehat{\operatorname{Cov}}
(\widehat{\boldsymbol{\theta}}_m).
\]
Let \(\boldsymbol{V}_m\) denote the asymptotic covariance matrix for method
\(m\). The empirical covariance matrices were compared with their theoretical
limits using the relative Frobenius discrepancy
\[
D_{V,m}(n)
=
\frac{
\left\|
\widehat{\boldsymbol{V}}_{m,n}
-
\boldsymbol{V}_m
\right\|_{F}
}{
\left\|
\boldsymbol{V}_m
\right\|_{F}
},
\]
and the classification-weighted discrepancy
\[
D_{R,m}(n)
=
\frac{
\left|
\operatorname{tr}
\left(
\boldsymbol{H}_R
\widehat{\boldsymbol{V}}_{m,n}
\right)
-
\operatorname{tr}
\left(
\boldsymbol{H}_R
\boldsymbol{V}_m
\right)
\right|
}{
\operatorname{tr}
\left(
\boldsymbol{H}_R
\boldsymbol{V}_m
\right)
}.
\]
Accordingly, the ``Empirical trace'' and ``Theoretical trace'' columns below
refer respectively to
\[
\operatorname{tr}
\left(
\boldsymbol{H}_R
\widehat{\boldsymbol{V}}_{m,n}
\right)
\qquad\text{and}\qquad
\operatorname{tr}
\left(
\boldsymbol{H}_R
\boldsymbol{V}_m
\right).
\]

\begin{table}[t]
\centering
\caption{Covariance diagnostics for the finite-sample Monte Carlo experiment.}
\label{supp:tab:covariance}
\begin{tabular}{llcccc}
\toprule
\(n\) & Method &
\(D_{V,m}(n)\) &
\(D_{R,m}(n)\) &
Empirical trace &
Theoretical trace\\
\midrule
250
& CC   & 0.1573 & 0.0134 & 1.5373 & 1.5581\\
& MCAR & 0.1608 & 0.0530 & 2.1629 & 2.0541\\
& IPC  & 0.1965 & 0.0480 & 1.4300 & 1.3645\\[2pt]

500
& CC   & 0.1550 & 0.0476 & 1.4840 & 1.5581\\
& MCAR & 0.1777 & 0.0261 & 2.0005 & 2.0541\\
& IPC  & 0.1773 & 0.0442 & 1.4248 & 1.3645\\[2pt]

1000
& CC   & 0.1249 & 0.0142 & 1.5361 & 1.5581\\
& MCAR & 0.1205 & 0.0051 & 2.0435 & 2.0541\\
& IPC  & 0.1309 & 0.0068 & 1.3739 & 1.3645\\
\bottomrule
\end{tabular}
\end{table}

The full covariance discrepancy is reduced by \(n=1000\), although the
improvement is not monotone across all intermediate sample sizes because it
weights all parameter directions equally. The
classification-weighted discrepancy is substantially smaller at \(n=1000\):
approximately \(1.4\%\), \(0.5\%\), and \(0.7\%\) for CC, MCAR, and IPC,
respectively. This latter comparison is the more directly relevant diagnostic
for the excess-risk theory, since only covariance error in directions weighted
by \(\boldsymbol{H}_R\) contributes to the leading classification risk.

\medskip
\noindent\textit{Paired comparisons and reproducibility.}
Because CC and IPC were constructed from the same complete sample within each
Monte Carlo replication, their excess risks can also be compared pairwise.
Among replications in which both estimators were usable, the mean scaled paired
difference
\[
2n
\left\{
\mathcal{X}_{\mathrm{CC}}
-
\mathcal{X}_{\mathrm{IPC}}
\right\}
\]
was
\[
0.1018,\qquad
0.0462,\qquad
0.1497
\]
for \(n=250,500,\) and \(1000\), respectively. The proportions of paired
replications in which IPC had the smaller excess risk were
\[
0.538,\qquad
0.509,\qquad
0.556.
\]
These proportions are not expected to approach one: the theoretical result
concerns the difference in expected excess risk, not samplewise stochastic
dominance.

The final simulation used \(B=500\) independent complete samples at each
sample size and was implemented in R. The same generated complete sample was
used to construct CC, MCAR, and IPC within each replication.

%

\end{document}